\documentclass{article}
\pdfoutput=1

\usepackage[final]{neurips_2022}

\usepackage[utf8]{inputenc} 
\usepackage[T1]{fontenc}    
\usepackage{hyperref}       
\usepackage{url}            
\usepackage{booktabs}       
\usepackage{amsmath,amsthm,amssymb}
\usepackage{amsfonts}       
\usepackage{nicefrac}       
\usepackage{microtype}      
\usepackage{xcolor}         
\usepackage[ruled]{algorithm2e}
\usepackage{wrapfig}
\usepackage{graphicx} 

\usepackage[numbers]{natbib}

\SetKwInput{KwInput}{Input}                
\SetKwInput{KwOutput}{Output}              
\SetKwInput{KwInit}{Initalization}              
\usepackage{thmtools, thm-restate}

\newtheorem{theorem}{Theorem}[section]
\newtheorem{corollary}{Corollary}[theorem]

\newtheorem{lemma}[theorem]{Lemma}
\newtheorem{claim}[theorem]{Claim}
\newtheorem{remark}{Remark}

\usepackage[colorinlistoftodos,bordercolor=orange,backgroundcolor=orange!20,linecolor=orange,textsize=tiny]{todonotes}

\title{Effective Dimension in Bandit Problems under Censorship}

\author{%
  Gauthier Guinet \thanks{Work done prior to joining Amazon.}\\
  AWS AI Labs\\
  \texttt{guinetgg@amazon.com} \\
  \And
  Saurabh Amin \thanks{Dept. of Civil and Environmental Engineering \& Laboratory for Information and Decision Systems.} \\
  MIT \\
  \texttt{amins@mit.edu} \\
  \And
  Patrick Jaillet \thanks{Dept. of Electrical Engineering and Computer Science \& Laboratory for Information and Decision Systems.}\\
  MIT\\
  \texttt{jaillet@mit.edu} \\
}

\begin{document}

\maketitle

\begin{abstract}
  In this paper, we study both multi-armed and contextual bandit problems in censored environments. Our goal is to estimate the performance loss due to censorship in the context of classical algorithms designed for uncensored environments. Our main contributions include the introduction of a broad class of censorship models and their analysis in terms of the \emph{effective dimension} of the problem -- a natural measure of its underlying statistical complexity and main driver of the regret bound. In particular, the effective dimension allows us to maintain the structure of the original problem at first order, while embedding it in a bigger space, and thus naturally leads to results analogous to uncensored settings. Our analysis involves a continuous generalization of the Elliptical Potential Inequality, which we believe is of independent interest. We also discover an interesting property of decision-making under censorship: a transient phase during which initial misspecification of censorship is self-corrected at an extra cost; followed by a stationary phase that reflects the inherent slowdown of learning governed by the effective dimension. Our results are useful for applications of sequential decision-making models where the feedback received depends on strategic uncertainty (e.g., agents’ willingness to follow a recommendation) and/or random uncertainty (e.g., loss or delay in arrival of information).
\end{abstract}

\section{Introduction}


Bandit problems are prototypical models of sequential decision-making under uncertainty. They are widely studied due to their applications in recommender systems, online advertising, medical treatment assignment, revenue management, network routing and control \cite{lattimore2020bandit,MAL-068}. Our work is motivated by settings in which the feedback received by the decision-maker in each round of decision is censored by a stochastic process that depends on the current action as well as past history of feedbacks and actions. For instance, in typical missing data problems, the decision-maker needs to deal with frequent losses of information (or delays in arrival of information) due to exogeneous failures such as faulty and/or unreliable communication. Missing observations in dynamical interactions with the environment are a common concern in diverse fields ranging from operations management to health sciences to physical sciences~\cite{YANG20091896,HonKin10,little2002statistical}. In other settings, such as AI-driven platforms for health alerts, route guidance, and product recommendations~\cite{NEURIPS2019_e49eb652,yu2021reinforcement}, the reception of feedback depends on whether or not the decision (or recommendation) is adopted by strategic agents (e.g. patients, customers or drivers) with private valuations. Thus, from the platform’s viewpoint, the adoption behavior of heterogeneous agents can be regarded as a \emph{stochastic censorship process}. 

In static environments, the bias induced by the presence of randomly missing information has been thoroughly studied~\cite{little2002statistical,review_missing}. However, in online settings, the dynamics of learning and acting are inherently coupled: since censorship mediates current information of the environment, it impacts the outcome of data-driven decision process; this in turn conditions the future decisions and future censored feedback, creating a complex and endogenous joint temporal dependency. Our work contributes to the analysis of such phenomena for a broad classes of decision and censorship models. Importantly, it is the first \emph{normative inquiry} of how censorship impacts the statistical complexity of bandit problems. We develop an analysis approach that is useful for both estimating the performance loss due to censorship and refining the classical algorithms designed for uncensored environments.

\subsection{Related Work}
Within the extensive bandits literature, well-surveyed in \citep{lattimore2020bandit,MAL-068}, our work is most closely related to stochastic delayed bandits. Initially, this line of work focused on the joint evolution of actions and information in settings where the reception of the latter is delayed~\cite{dudik2011efficient}. Of particular interest is the packet loss model recently introduced in \cite{stoch_unrest_delay}, which provides the regret bound $\mathcal{O}(\frac{1}{p}R_{T})$ where $R_{T}$ is the uncensored regret and $p$ the censorship probability. Analogous results have been shown in the context of Combinatorial Multi-Armed Bandits with probabilistically triggered arms; see for example, \cite{JMLR:v17:14-298} and \cite{NIPS2017_a8e864d0}. Our work provides a systematic approach to study more general censorship models, and sheds light on how the impact of coupled feedback and censorship realizations on the expected regret can be evaluated in terms of the \emph{effective dimension} of the problem. 

Importantly, we also tackle the contextual bandit problems, where relatively few results are available on the regret under missing or censored feedback. A notable exception is the work of \cite{pmlr-v119-vernade20a}, who focus on a different information structure and obtain a scaling of $1/p$ (see Remark \ref{unif_models}). A related contribution by \cite{theshold_pot} provides both a potential-based analysis of the Upper Confidence Bound algorithm (UCB) for multi-armed bandits and an algorithmic variant leveraging the Kaplan-Meier estimator, although their censorship setting is different than ours. In particular, our results are applicable to settings when delay is significantly large (possibly infinite). This is in contrast to prior results on bandits with delayed information structure which assume either that the delay is \textit{constant}, \textit{upper bounded}, has a \textit{finite mean}, or simply provide regret guarantees that are \textit{linear in the cumulative delays} up to time $T$ \citep{dudik2011efficient,joulani2013online,queue_delay,ZhouGLM,pike2018bandits}. Under such assumptions on delay, one usually gets a second order additive dependency of the regret in terms of delay parameters, which practically says that delay is benign for bandits. On the other hand, we show that censorship leads to a first order multiplicative dependency on regret and we provide a complete characterization of this dependency for a wide range of bandits and censorship models. 

Moreover, the abovementioned works primarily focus on modifying well-known bandit algorithms to account for delays, or propose new delay-robust algorithms which may be difficult to implement in practice; a notable exception includes~\cite{wu2022thompson} but it focuses on Thompson Sampling. In our work, we instead focus on estimating the performance loss due to censorship and derive insights on the behavior of well-known UCB class of algorithms~\citep{li2010contextual,pmlr-v15-chu11a,NIPS2011_e1d5be1c}. These algorithms are widely used in practice; moreover, their theoretical study has been shown to be useful for analysis of broader class of algorithms (notably Thompson Sampling~\cite{agrawal2012analysis,tsVanRoy} and Information-Directed Sampling~\citep{NIPS2014_301ad0e3,pmlr-v75-kirschner18a}).

On a somewhat related note, the literature on non-stochastic multi-armed bandit problems with delays~\citep{NIPS2010_7bb06076,pmlr-v49-cesa-bianchi16,NEURIPS2020_33c5f5bf} also tackles multiplicative dependency, although in a different setting than ours. Another related line of work is Partial Monitoring \citep{JMLR:v11:audibert10a,partialmonitor} which deals with generic categorization of learnability, rather than a fine-grained analysis of dimensionality in relation to censorship, which is our current focus. 

Our work contributes to the Generalized Linear Contextual Bandits literature \cite{NIPS2010_c2626d85,li2017provably} in two ways: firstly, through the use of these models in a sequential decision-making framework on which the impact of censorship is assessed in Sec. \ref{CB}. Secondly, by showing that our multi-threshold censorship model \ref{MT_model} induces, at first order, a non-linear structure that closely mirrors such models. Our results provide new tools to study this structure. It is useful to note that the notion of \textit{effective dimension} has been well-studied in the statistical learning and kernels literature~\citep{GPBandit,6138914} (where it is defined for a Gram matrix $K_n$ and regularization $\lambda$ as $d^n_{\text{eff}}(\lambda) = \text{tr}(K_n(K_n + \lambda\mathbb{I}_{d})^{-1})$). Our work shows that an analogous quantity governs the regret bound of bandit problems in censored settings. 

Finally, there is a rich literature on classical missing and censored data problems~\cite{little2002statistical,review_missing}. Although conditional on the choice of a given action the missing data/censorship process we study is an instance of missing-completely-at-random (MCAR), the online action generating process adds a significant difficulty to the problem: whereas MCAR is typically studied under a well-defined distributional assumption (e.g. i.i.d. generation of action), our problem needs to deal with adaptive (hence non i.i.d.) data generation process with respect to the filtration of past information. In particular, the structure of missing data set results from strong endogenous dependencies with past realization of the censorship (see Sec. \ref{set up}).

\subsection{Summary of Results}

In Sec. \ref{MAB}, we consider Multi-Armed Bandit (MAB) models and prove that the regret scales as $\Tilde{\mathcal{O}}(d_{\mathit{eff}}\sqrt{T})$ (Thm. \ref{THM Finite arms}), where $d_{\mathit{eff}}$ is the effective dimension with value $\sum_{a\in[d]}\frac{1}{p_{a}}$. In doing so, we recover and generalize related results from~\citep{stoch_unrest_delay,JMLR:v17:14-298} to more complex regularized settings and noise models. In particular, we prove that the effective dimension results from characterizing the so-called censored cumulative potential $\mathbb{V}_{\alpha}$. Interestingly, we also show that the adaptive nature of censorship on $\mathbb{V}_{\alpha}$ plays only a second order role (Prop. \ref{Monitoring AG}), that is, impact of censorship can be treated in an \textit{offline} manner at first order. 

Importantly, our study of MAB under censorship instantiates an analysis framework which extends to Linear Contextual Bandits (LCB) (Sec. \ref{CB}). Our main result provides that regret is still governed by the effective dimension, but now with a dependency of $\Tilde{\mathcal{O}}(\sigma\sqrt{d\cdot d_{\mathit{eff}}}\sqrt{T})$ (Thm. \ref{THM Linear arms}). To the best of our knowledge, these regret bounds provide the first theoretical characterization in LCB with censorship, and contribute to the literature by evaluating the impact of censorship on the performance of UCB-type algorithms. Our second main contribution is identifying the effective dimension for a broad class of multi-threshold models \ref{MT_model} as well as a precise understanding of the dynamic behavior induced by these models (Thm. \ref{THM Linear Optim MTM}). In particular, we find that censorship introduces a two-phase behavior: a transient phase during which the initial censoring misspecification is self-corrected at an additional cost; followed by a stationary phase that reflects the inherent slowdown of learning governed by the effective dimension. In extending our analysis from MAB to LCB, we also develop a continuous generalization of the widely used Elliptical Potential Inequality (Prop. \ref{Potential Control Linear}), which we believe is also of independent interest. Finally, our results (Thm. \ref{THM Finite arms} and Prop. \ref{Instance Dep Regret Finite} for MAB and Thm. \ref{THM Linear arms} for LCB) suggest that the UCB class of algorithms is indeed a reliable method for stochastic bandits problems under censorship.

\section{Problem Setup and Background}\label{set up}

\paragraph{Bandit Model:}

We successively consider stochastic multi-armed bandits (Sec. \ref{MAB}) and Linear Contextual Bandits (LCB) (Sec. \ref{CB}) in censored environments. In both settings, at each round $t\leq T$, the agent observes an action set $\mathcal{A}_{t}\subset \mathcal{A}$. She then selects an action $a_{t}\in \mathcal{A}_{t}$ (i.e. an \textit{arm}) to which a noisy feedback $r(a_{t})+\epsilon_{t}$ is associated, where $r(a_{t})$ is a bounded reward and $\epsilon_{t}$ is an i.i.d. sub-Gaussian noise of pseudo-variance $\sigma^{2}$. For action $a$, the sub-optimality gap at time $t$ is denoted $\Delta_{t}(a)\triangleq \max_{\Tilde{a}\in \mathcal{A}_{t}} r(\Tilde{a}) - r(a)$, and the maximal gap $\Delta_{max}\triangleq \max_{a,t} \Delta_{t}(a)$. We now recall the specifics of each model:
\begin{itemize}
    \item \textbf{MAB}: There is a finite number of actions $d$, enumerated as $\mathcal{A}\triangleq [d]$, each having a scalar reward $\theta^{\star}_{a}$. Arms are \textit{independent}: playing one arm gives no information about the others.
    \item \textbf{LCB}: The action set $\mathcal{A}_{t}$ is a subset of the unit ball $\mathbb{B}_{d}$, possibly infinite. Unless explicitly mentioned, the reward is assumed to be linear with respect to a latent unknown vector $\theta^{\star}\in \mathbb{R}^{d}$, i.e. $r(a) = \langle a,\theta^{\star}\rangle$. Non-stochastic contexts are modeled by the fact that $\mathcal{A}_{t}$ is drawn by an oblivious adversary. Here one does not need to rely on the typical i.i.d assumption on their generating process \cite{ZhouGLM,NIPS2010_c2626d85}.
\end{itemize}

\paragraph{Information Structure:}
In the classical uncensored setting, the noisy feedback is immediately observed post-decision and utilized to make decisions in the next round. We introduce the following \textbf{censorship} model: an independent Bernoulli random variable of parameter $p(a_{t})$ denoted as $x_{a_{t}}$ is drawn after each decision $a_{t}$ and the feedback is observed, i.e., \textit{realized}, if and only if $x_{a_{t}}=1$; else the feedback is said to be \textit{censored}. We recover the uncensored setting when $p(a)\equiv 1$. Henceforth, in both finite and linear settings, the Bernoulli parameter corresponding to the censorship probability depends on the action chosen i.e. our model allows the censorship to be heterogeneous across actions. Given that the action chosen at time $t$ is a random variable, $p(a_{t})$ refers to a random variable as well.




\noindent\begin{minipage}{0.54\textwidth}
\begin{algorithm}[H]
    \SetAlgoLined
    \KwInput{Total time $T$, Regularization $\lambda$, Precision $\delta$}
    \For{$t=1,\dots,T$}{
        Provide reward estimator $\Tilde{r}^{\lambda}_{t}$ verifying w.p. $1-\delta$: \\
        $\quad \forall a\in\mathcal{A}_{t}, r(a)\leq \Tilde{r}^{\lambda}_{t}(a)$\;
        Play action $a_{t}=\operatorname{argmax}_{a\in\mathcal{A}_{t}} \Tilde{r}^{\lambda}_{t}(a)$ \;
        \If{$(a_{t},r(a_{t})+\epsilon_{t})$ is realized i.e. $x_{a_{t}}=1$}{
            Update $\Tilde{r}^{\lambda}_{t}$\;}
    }
\caption{Generic UCB}
\label{Gen_UCB}
\end{algorithm}
\end{minipage}%
\hfill%
\begin{minipage}{0.45\textwidth}
\paragraph{Algorithms:} To study the impact of censorship on bandit problems, we consider the class of high-probability index algorithms based on the \textit{optimism under uncertainty} principle, commonly referred as \textbf{UCB}-algorithms. Following \cite{pmlr-v75-kirschner18a}, Algorithm \ref{Gen_UCB} summarizes the generic UCB design framework. We detail in App.\ref{prel} the specific instances of UCB for MAB (resp. LCB) used in Sec.\ref{MAB} (resp. Sec.\ref{CB}). Moreover, this family of algorithms strongly relies on regularized reward estimators $\Tilde{r}^{\lambda}_{t}$, where the regularizer is mostly used to prevent an artificial cold-start exploratory phase.
\end{minipage}


\paragraph{Performance Criterion:} The frequentist performance of the agent is measured by the notion of \textit{pseudo regret}, i.e., the difference between the algorithm's cumulative reward and the best total reward. More formally, we introduce for any policy $\pi \in \Pi$:
    \begin{align*}
        R(T,\pi) = \sum_{t=1}^{T} \max_{a\in \mathcal{A}_{t}} r(a)-\sum_{t=1}^{T} r(a_{t}) = \sum_{t=1}^{T}\Delta_{t}(a_{t}).
    \end{align*}
We aim to provide guarantees on $\mathbb{E}[R(T,\pi)]$ with respect to the number of rounds $T$ and quantities that govern the \textit{complexity} of the problem (for example number of arms, ambient dimension $d$, parameters of censorship model or smoothness properties of the reward $r$). Here, the expectation is with respect to the noise induced by the feedback, the censorship and a possibly randomized policy.

\subsection{Notations}

Transpose of a vector $u$ is denoted by $u^{\top}$, classical Euclidean inner product by $\langle .,.\rangle$ and trace operator by $\operatorname{Tr}$. For positive semi-definite matrix $\boldsymbol{\Sigma} \in \mathbb{R}^{d \times d}$ and for any vector $u \in \mathbb{R}^{d}$, notation $\|u\|_{\Sigma}$ refers to $\sqrt{u^{\top} \boldsymbol{\Sigma} u}$. We use notation $\mathbb{I}_{d}$ to denote the $d\times d$ identity matrix. $\mathbb{B}_{d}$ is the unit ball in $\mathbb{R}^{d}$. $[n]$ is the set of integers $\{1,2, \cdots, n\}$. For a given function $f$, we note $f^{(i)}$ the $i^{th}$ derivative of $f$. To avoid confusion with the dimension $d$, we use $\partial x$ instead of $dx$ to denote an infinitesimal increase of $x$. We use the asymptotic notations $\sim$, $\mathcal{O}$, $\Theta$ and $\Tilde{\mathcal{O}}$ ($\mathcal{O}$ when $\log$ factors are removed). Finally, for an event $\mathcal{H}$, we use $\neg \mathcal{H}$ to denote its complement.

\section{Multi-Armed Bandits}\label{MAB}

\subsection{Effective Dimension and Regret Bounds}

The main result of this section is that censorship effectively enlarges the dimension of the problem. We define the effective dimension as $d_{\mathit{eff}}\triangleq \sum_{a\in [d]}\frac{1}{p_{a}}$ and our result (Thm. \ref{THM Finite arms}) shows that, at first order, the regret is guaranteed to be the same as the uncensored problem with $d_{\mathit{eff}}$ arms instead of $d$. 


\begin{restatable}{theorem}{THMFinitearms}
\label{THM Finite arms}
Under censorship, the UCB algorithm with regularization $\lambda$ has an instance-independent expected regret of:
    \begin{align*}
    \mathbb{E}[R(T,\pi_{\text{UCB}})] 
    &\leq \Tilde{\mathcal{O}}(\sigma \sqrt{d_{\mathit{eff}}T}).
    \end{align*}
\end{restatable}
Furthermore, we obtain analogous regret guarantees for instance-dependent cases where, at first order,
the uncensored dimension $\sum_{a\neq a^{\star}}\frac{\sigma^{2}}{\Delta_{a}}$ enlarges to $\sum_{a\neq a^{\star}}\frac{\sigma^{2}}{p_{a}\Delta_{a}}$:
\begin{restatable}{proposition}{InstanceDepRegretFinite}
\label{Instance Dep Regret Finite}
For a fixed action set $\mathcal{A}_{t} \equiv[d]$ and for a-priori known action gap $\Delta_{a}\triangleq \max_{\Tilde{a}}\theta^{\star}_{\Tilde{a}}-\theta^{\star}_{a}$, the UCB algorithm with regularization $\lambda$ has the instance-dependent expected regret:
\begin{align*}
    \mathbb{E}[R(T,\pi_{\text{UCB}})] 
    &\leq \mathcal{O}\Big(\log(T)\sum_{a\neq a^{\star}}\frac{1}{p_{a}}\max(\frac{\sigma^{2}}{\Delta_{a}},\Delta_{a})\Big).
\end{align*}
\end{restatable}
On one hand, a preliminary understanding of censorship posits an increase of the average "\textit{regret per information gain}" \cite{pmlr-v75-kirschner18a} (as it takes longer on average to get the same amount of information) but does not change the underlying complexity of the problem. One the other hand, our results (Thm. \ref{THM Finite arms} and Prop. \ref{Instance Dep Regret Finite}) postulate that the censored problem is equivalent at first order to a higher dimensional problem but explored with the same \textit{regret per information gain}. 

The abovementioned results extends to a-priori known heteroskedasticity (see Rem. \ref{Hetero Finite 1} and \ref{Hetero Finite 2} in App. \ref{Proof MAB}). For this general setting, the effective dimension for instance-independent (resp. dependent) case is given by $ \sum_{a}\frac{\sigma_{a}^{2}}{p_{a}}$ (resp. $\sum_{a\neq a^{\star}}\frac{\sigma_{a}^{2}}{p_{a}\Delta_{a}}$), where $\sigma_{a}^{2}$ is the variance proxy of arm $a$. Although the scaling in $\sum_{a}\frac{1}{\Delta_{a}p_{a}}$ was already mentioned in \cite{stoch_unrest_delay} for unregularized setting with homogeneous variance $\sigma^{2}$ and proven to be optimal, our results generalize these findings.

\subsection{Cumulative Censored Potential}

We now provide a proof sketch of Thm. \ref{THM Finite arms}, and in doing so, we instantiate an analysis framework that will be extended in Sec. \ref{CB}. This proof consists in the successive elimination of the noise induced by the feedback and censorship. This leads to regret guarantees on a resulting deterministic quantity by characterizing worst-case learning conditions. The first step of the proof is a variant of the classical reduction of the UCB regret to another quantity we refer to as the \emph{expected cumulative censored potential}. Before stating it, we define at the end of a round $t\in [T]$, the random number of times an arm $a$ has been \textit{pulled} as $\tau_{a}(t) \triangleq \sum_{l=1}^{t} \mathbf{1}\{a_{l}=a\}$. Similarly, the number of times an action $a$ has been \textit{realized} at the end of round $t$ is denoted $N_{a}(t) \triangleq \sum_{l=1}^{t} \mathbf{1}{\{a_{l}=a, x_{a_{l}}=1\}}$. We then have:
\begin{restatable}{lemma}{PotentialReductionFinite}
\label{Potential Reduction Finite} 
Given an uniform regularization of $\lambda>0$, the UCB algorithm verifies:
\begin{align*} 
    \mathbb{E}[R(T,\pi_{\text{UCB}})] \leq 2 \sqrt{6\sigma^{2}\log(T)}\mathbb{E}[\mathbb{V}_{\frac{1}{2}}(T,\pi_{\text{UCB}})] + 2\lambda\|\theta^{\star}\|_{\infty}\mathbb{E}[\mathbb{V}_{1}(T,\pi_{\text{UCB}})] + \frac{2d\Delta_{max}}{T}
\end{align*}
where, for any $\alpha>0$ and $\pi \in \Pi$, the cumulative potential under censorship is given by:
\begin{align*}
    \mathbb{V}_{\alpha}(T,\pi) = \sum_{t=1}^{T}(N_{a_{t}}(t-1)+\lambda)^{-\alpha}.
\end{align*}
\end{restatable}

Without censorship, the cumulative potential translates the average rate of decay of uncertainty on the reward of different arms and is closely linked to the divergence between the true reward distribution and the empirical distribution of observed rewards \cite{pmlr-v119-shekhar20a}. Introducing censorship transforms the classical deterministic decay rate into a stochastic one. For a typical reward distribution, the rate of decay is proportional to a term in $n^{-\alpha}$ or can be upper bounded by such a term (see for e.g. \cite{pmlr-v119-shekhar20a}), where $n$ is the number of \textit{observed} rewards. Therefore, a higher $\alpha$ corresponds to faster learning.

In contrast to the classical non-regularized analysis or to the LCB case of Sec. \ref{CB}, we observe two different orders of $\alpha$ ($\nicefrac{1}{2}$ and $1$) coming from the use of the $L_{\infty}$-norm instead of the $L_{2}$-norm. Taken independently, they lead to respective contributions of $\mathcal{O}(d_{\mathit{eff}}\log(T))$ and $\mathcal{O}(\sqrt{d_{\mathit{eff}}T})$. Note that by working with a general $\alpha$, our analysis naturally extends beyond sub-Gaussian noise to more general assumptions about the Laplace transform of noise (e.g., lighter or heavier tails), as discussed in Rem.\ref{Tails}.
To further study $\mathbb{V}_{\alpha}$, we introduce the following property:
\begin{restatable}{proposition}{PotentialControlFinite}
\label{Potential Control Finite} 
For all $\alpha >0$, $\delta \in ]0,1]$ and given $\psi_{\alpha}$ a primitive of $x\mapsto x^{-\alpha}$, we have:
\begin{align*}
    \max_{\pi \in \Pi} \mathbb{E}[\mathbb{V}_{\alpha}(T,\pi)] &\leq 
    \frac{d_{\mathit{eff}}}{(1-\delta)^{\alpha}}\left[\psi_{\alpha}(\frac{T}{d_{\mathit{eff}}}+\frac{\lambda}{1-\delta}) - \psi_{\alpha}(\frac{\lambda}{1-\delta})\right] + \frac{24d_{\mathit{eff}}\log(T)+d}{\lambda^{\alpha}} + \frac{4d_{\mathit{eff}}}{\lambda^{\alpha}\delta^{2}T^{12\delta^{2}}}.
\end{align*}
\end{restatable}

The proof of this proposition involves two steps: firstly, we remove the stochastic dependence induced by the censorship through concentration properties (See App. \ref{Proof MAB}), and we then solve the resulting policy maximization problem (Lemma \ref{Optimization Lemma Finite}). In the first step, we consider for a given $\delta \in ]0,1]$ the event:
\begin{align*}
    \mathcal{H}_{CEN}(\delta) &= \left\{\exists a \in [d], t\in[T],N_{a}(t) < (1-\delta)p_{a}\tau_{a}(t) \quad \text{and} \quad \tau_{a}(t) \geq T_{0}(a) \right\},
\end{align*}
where $T_{0}(a)\triangleq 24\log(T)/p_{a}+1$ and claim that $\mathbb{P}(\mathcal{H}_{CEN}(\delta)) \leq \frac{4d_{\mathit{eff}}}{\delta^{2}}T^{-12\delta^{2}}$, improving a result of \cite{stoch_unrest_delay}. Here $\mathcal{H}_{CEN}$ denotes the event where there is a significant gap between the realized and expected number of observed rewards. We consider its complement in our analysis of the principal order of regret. This allows us to lower bound for each action, the realized number of reward observations by a multiple of the number of times that action was selected, thus eliminating the randomness induced by censoring.

Our second step makes use of the following lemma (also known as a \textit{water-filling process} in information theory \cite{ITbook}):
\begin{restatable}{lemma}{OptimizationLemmaFinite} \label{Optimization Lemma Finite}
For $\psi_{\alpha}$ a primitive of $x\mapsto x^{-\alpha}$ where $\alpha \in ]0,1]$, regularization $(\lambda_{a})_{a\in[d]}\in (\mathbb{R}_{>0})^{d}$ and censorship vector $(p_{a})_{a\in [d]}$, the solution of the optimization problem:
\begin{align*}
    \max_{\tau_{1}\dots,\tau_{d}\geq 0} \quad & \sum_{a\in[d]} \frac{1}{p_{a}}\Big(\psi_{\alpha}(p_{a}\tau_{a}+\lambda_{a})-\psi_{\alpha}(\lambda_{a})\Big) \quad \textrm{s.t.} \quad  \sum_{a\in[d]}\tau_{a}=T
\end{align*}
is given by $\tau^{\star}_{a}=\frac{1}{p_{a}}[C-\lambda_{a}]^{+}$, where $C$ ensures the total budget constraint $ \sum_{a\in[d]}\tau^{\star}_{a}=T$. In particular, with $\lambda_{\text{eff}}\triangleq \frac{1}{d_{\mathit{eff}}}\sum_{a\in[d]}\frac{\lambda_{a}}{p_{a}}$ and $\lambda_{a}^{0}\triangleq d_{\mathit{eff}}(\lambda_{a}-\lambda_{\text{eff}})$, the optimal solution is given by $\tau^{\star}_{a}\triangleq \frac{1}{p_{a}d_{\mathit{eff}}}(T-\lambda_{a}^{0})$ for $T\geq \displaystyle\max_{a} \lambda_{a}^{0}$ and the optimal value is $d_{\mathit{eff}}\psi_{\alpha}(\frac{T}{d_{\mathit{eff}}}+\lambda_{\text{eff}}) - \sum_{a\in [d]}\frac{1}{p_{a}}\psi_{\alpha}(\lambda_{a})$.
\end{restatable}

For unregularized algorithms, this framework can be easily applied to provide instances-dependent guarantees by adding constraints of type $\tau_{a} \leq f(\Delta_{a})$ within Lemma \ref{Optimization Lemma Finite}. Optimal guarantees under regularization such as the ones given in Prop. \ref{Instance Dep Regret Finite} require however to consider both orders of $\mathbb{V}_{\alpha}$ ($\nicefrac{1}{2}$ and $1$) simultaneously and not independently, leading to slight variations as shown in the proof of Prop. \ref{Instance Dep Regret Finite}. Next, we further discuss the properties of $\mathbb{V}_{\alpha}$ given its importance in our analysis.

\subsection{Evaluating Adaptivity Gain}


It is well known that adaptivity is a key feature of sequential decision problems: optimal policies use feedback from previous decisions to decide the next action to take based on the data, and in comparison non-adaptive policies can be quite suboptimal. Somewhat interestingly, the main result of this section is that adaptivity in the context of censoring does not provide a significant advantage to the decision maker. More precisely, being able to observe which decisions have been censored and adapting to this information does not bring more than a second order gain. In proving this result, we quantify and gain insight into the expected performance of policies that are adaptive to the realization of the censorship process, in comparison to a class of non-adaptive (i.e., offline) policies. 

In fact, through the introduction of $\mathcal{H}_{CEN}(\delta)$ and for any $\alpha \in [0,1]$, $\delta \in ]0,1]$, we showed in Prop. \ref{Potential Control Finite} the upper bound $\frac{d_{\mathit{eff}}}{(1-\delta)^{\alpha}}\psi_{\alpha}(\frac{T}{d_{\mathit{eff}}} + \frac{\lambda}{1-\delta})$ for the learning complexity $\max \mathbb{E}[\mathbb{V}_{\alpha}(T,\pi)]$ where the maximum is taken over the class of adaptive policies $\Pi_{\mathit{adapt}}$, i.e., measurable with respect to the censorship. Note that the exact value of such maximum is notoriously difficult to study due to the adaptive nature of censorship induced by the decision-making process. Next, we introduce $\Pi_{\mathit{off}}$, the class of policies that are not adaptive with respect to the censorship and we prove that :
\begin{restatable}{lemma}{asymptoff}\label{asympt_off} For $\alpha \in ]0,1]$ and $\lambda>0$, we have $\displaystyle \max_{\pi \in \Pi_{\text{off}}} \mathbb{E}[\mathbb{V}_{\alpha}(T,\pi)] \sim d_{\mathit{eff}}\psi_{\alpha}(\frac{T}{d_{\mathit{eff}}}+\lambda)$.
\end{restatable}
In other words, restricting attention to offline policies is sufficient to obtain the correct scaling. The next step to complete our claim is the asymptotic expansion: 
\begin{restatable}{proposition}{MonitoringAG}\label{Monitoring AG} For $\alpha \in ]0,1]$, by denoting $\displaystyle \gamma_{\alpha}(\mathbf{p}) \triangleq \frac{\alpha}{2d_{\mathit{eff}}^{1-\alpha}}\sum_{a\in [d]}\frac{1}{p_{a}}\Big(\sum_{\Tilde{a}\neq a}\frac{1-p_{\Tilde{a}}}{p_{\Tilde{a}}}\Big)$, we have:
\begin{align*}
    \max_{\pi \in \Pi_{\text{adapt}}} \mathbb{E}[\mathbb{V}_{\alpha}(T,\pi)] - \max_{\pi \in \Pi_{\text{off}}} \mathbb{E}[\mathbb{V}_{\alpha}(T,\pi)] = \gamma_{\alpha}(\mathbf{p})\frac{1}{T^{\alpha}} + o(\frac{1}{T^{\alpha}}).  \tag{$\star $}\label{Constant}
\end{align*}
Moreover, if for a given $\beta \in ]0,1[$, we introduce $\Pi_{\text{single}}(\beta T)$ the policy class whose censorship information set has a single updating at time $\lfloor\beta T\rfloor$, we have:
\begin{align*}
    \max_{\pi \in \Pi_{\text{single}}(\beta T)} \mathbb{E}[\mathbb{V}_{\alpha}(T,\pi)] - \max_{\pi \in \Pi_{\text{off}}} \mathbb{E}[\mathbb{V}_{\alpha}(T,\pi)] =  \gamma_{\alpha}(\mathbf{p})\frac{\beta}{T^{\alpha}} + o(\frac{1}{T^{\alpha}}). \tag{$\star \star$}\label{One-Shot}
\end{align*}
\end{restatable}
Thus, $\gamma_{\alpha}(\mathbf{p})$ can be viewed as an adaptivity gain resulting from the continuous correction of the cumulative variance induced by the action selection process. Essentially, it is closely related to the Jensen Gap of an appropriate random variable and the proof involves the study of the Taylor expansion of the potential function $\psi_{\alpha}$. (\ref{One-Shot}) tells us that a single observation of the censorship realization is sufficient to obtain a  near-optimal \textit{gain in adaptivity}.  We present a  proof sketch of Prop. \ref{Monitoring AG} in App. \ref{Proof MAB}. This shows that censorship in MAB can be treated in an \textit{offline} manner at first order.

\section{Contextual Bandit} \label{CB}


In this section, we study Linear Contextual Bandits (LCBs) under censorship. The regret analysis for the generic censorship model in Sec.~\ref{set up} is significantly more complex for LCB than for MAB. This is due to the fact that different actions contribute differently to the information acquisition, leading to a non-linear phenomenon governing the trade-off between reward and information gain~(see Sec.\ref{temp dyn}). 



\subsection{Multi-threshold Models and Regret Bounds}

To address the abovementioned challenge, we now introduce a simple  \emph{multi-threshold} censorship model, which enables a precise regret analysis.  In particular, we consider that feedback is censored according to the following action-dependant probability:
\begin{align*}
    p:a \in \mathbb{B}_{d} \mapsto  \sum_{j=0}^{k} \mathbf{1}\{\sin(\phi_{j}) \leq \langle a,u\rangle <\sin(\phi_{j+1}) \}p_{j}, \tag{$\mathcal{MT}$}\label{MT_model}
\end{align*}
where $(\phi_{j})_{j\leq k+1}$ is an increasing sequence verifying $\phi_{0}=-\frac{\pi}{2}$, $\phi_{k+1}=\frac{\pi}{2}$ and $u\in \mathbb{R}^{d}$ is a unit vector. We assume that $(p_{j})_{j\leq k}$ is decreasing,  i.e. the censorship is increasing with $j$ in direction $u$. Henceforth, we refer to the interval $[\sin(\phi_{j}),\sin(\phi_{j+1})[$ as \emph{region $j$}. Note that simple models such as uniform censorship are subsumed by this family (for $k$ equals $0$). 

\begin{wrapfigure}{r}{0.5\textwidth}
  \begin{center}
    \includegraphics[width=0.4\textwidth]{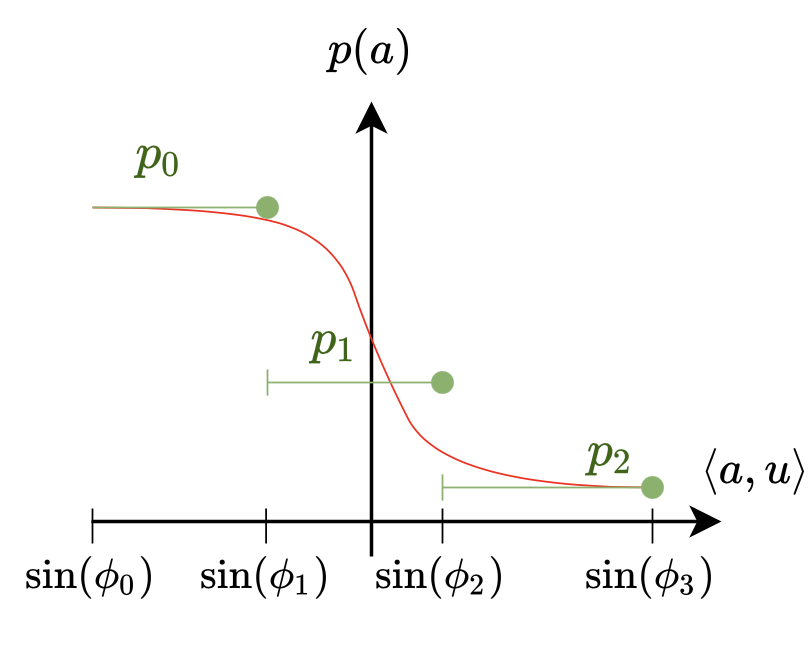}
  \end{center}
  \caption{Example of a multi-threshold model for $k=2$ (Green). Logistic censorship model (Red)}
	\label{fig:MT}
\end{wrapfigure}
The two main features of the multi-threshold model are: the \emph{radial} aspect (the censorship probability depends on the action through a scalar product with a given vector) and the \emph{monotonicity} (the censorship is monotone in the value of this scalar product). Note that \ref{MT_model} can be seen as a piecewise constant approximation of any Generalized Linear Model (GLM) \cite{mccullagh1989generalized}. Thus, the simplicity of this censorship model is not an inherently limiting factor on the generality of our subsequent results. 

Moreover, \ref{MT_model} admits a natural behavioral interpretation: Such a distribution can be seen as induced by a population model of heterogeneous random-utility maximizing agents. A single threshold model (i.e. $k$ equals $1$) corresponds to a given agent type, and the multi-threshold model naturally results from aggregate responses of heterogeneous population~\cite{DynamicDiscreteChoice}. 

We now state the main result of this section:

\begin{restatable}{theorem}{THMLineararms}\label{THM Linear arms}
For a given multi-threshold censorship model \ref{MT_model}, there exits $d_{\mathit{eff}}$ such that the UCB algorithm with regularization $\lambda$ has an instance-independent expected regret of:
\begin{align*}
    \mathbb{E}[R(T,\pi_{\text{UCB}})] 
    &\leq\Tilde{\mathcal{O}}(\sigma\sqrt{d\cdot d_{\mathit{eff}}}\sqrt{T}).
\end{align*}
\end{restatable}
Importantly, note the mapping from the original dimension $d$ to the enlarged $\sqrt{d \cdot d_{\mathit{eff}}}$, in contrast to the previous dilation $d\mapsto d_{\mathit{eff}}$ for the case of MAB problems. 
An extension to Generalized Linear Contextual Bandits is provided in App. \ref{gen linear} where we show that the dimension is governed by $\sqrt{d \cdot d_{\mathit{eff}}}/\kappa$, with $\kappa$ corresponding to a minimum of the derivative of the link function (encompassing the smoothness of the GLM at its maximum)~\citep{li2017provably,NIPS2010_c2626d85}. We conjecture that this result still holds if we relax the monotonicity property of \ref{MT_model} although it will require some modifications in the proofs of section \ref{Proof Multi}. On the other hand, we believe that the radial property is necessary, considering the related literature on GLMs (further discussed in App. \ref{gen linear}) where it appears prominently. 

\subsection{Generalized Cumulative Censored Potential}
Analogous to the MAB case, we now introduce for LCB the random matrices corresponding to the effective realization $\mathbb{W}^{C}_{t}\triangleq \lambda\mathbb{I}_{d} + \sum_{n=1}^{t} x_{a_{t}}a_{t}a_{t}^{\top}$ and the expected realization $\mathbb{W}_{t} \triangleq \lambda\mathbb{I}_{d} + \sum_{n=1}^{t} p(a_{t})a_{t}a_{t}^{\top}$. 
We also introduce the continuous counterpart of $\mathbb{W}_{t}$ defined as $\mathbb{W}(t)\triangleq \lambda\mathbb{I}_{d}+\int_{u=0}^{t}p(a(u))a(u)a(u)^{\top}\partial u$, where $(a(u))_{u\leq T}$ is an integrable deterministic path.\footnote{In this section, the generic notation $X(t)$ is used for continuous time quantities and $X_{t}$ for  discrete time.} We emphasise that the use of continuous counterpart is key in enabling our next results. As in the MAB case, we bound the regret although now using a generalization of $\mathbb{V}_{\alpha}$:
\begin{restatable}{lemma}{PotentialReductionLinear} \label{Potential Reduction Linear} For all $\delta \in ]0,1]$, there exists a constant $\Tilde{\beta_{\delta}}(T)=\Theta(\sqrt{d\log(T)})$ such that
\begin{align*}
    \mathbb{E}[R(T,\pi_{\text{UCB}})] \leq 2\Tilde{\beta_{\delta}}(T) \sqrt{T\mathbb{E}[\mathbb{V}_{1}(T,\pi_{\textit{UCB}})]} + \delta T\Delta_{max},
\end{align*}
where, for $\alpha>0$ and $\pi \in \Pi$, the linear extension of the cumulative censored potential is given by:
\begin{align*}
    \mathbb{V}_{\alpha}(T,\pi) \triangleq \sum_{t=1}^{T}\|a_{t}\|^{2}_{(\mathbb{W}^{C}_{t-1})^{-\alpha}} = \sum_{t=1}^{T}\operatorname{Tr}((\mathbb{W}^{C}_{t-1})^{-\alpha}a_{t}a_{t}^{\top}).
\end{align*}
\end{restatable}

The proof idea is analogous (albeit more complex) than in the finite action case (see App. \ref{Proof LCB}).
In order to get a handle on $\mathbb{V}_{\alpha}$, we again leverage a two-step approach: first we eliminate the randomness due to censorship (here, we utilize matrix martingale inequalities) and then optimize the resulting deterministic quantity seen through a continuous lens. The first step requires the following result:
\begin{restatable}{proposition}{PotentialControlLinear} \label{Potential Control Linear}
For any $\delta \in ]0,1]$, $\lambda>0$, $\alpha>0$ and policy $\pi \in \Pi$, we have: 
\begin{align*}
    \mathbb{E}[V_{\alpha}(T,\pi)] \leq \frac{\delta}{\lambda^{\alpha}} + C(\delta)^{\alpha} \operatorname{Tr}\Big(\int_{0}^{T}\mathbb{W}(t)^{-\alpha}a(t)a(t)^{\top}\partial t\Big),
\end{align*}
where $C(\delta)\triangleq 8(\lambda+1)\max(\log(d/\delta))/\lambda,1)/\lambda$.
\end{restatable}
The key idea of this result is to observe that the telescopic sum on which the classical Elliptical Potential lemma \citep{NIPS2011_e1d5be1c,adpt_cofond_Russo,carpentier2020elliptical} heavily relies on is, in fact, the discrete approximation of an integral over a matrix path. This critical methodological contribution is further discussed in Rem. \ref{rem 1} and \ref{tour de force}.
\begin{remark}\label{rem 1}
One way to fully appreciate the generality of this result is to consider the simpler case of classical uncensored environment for which we obtain for $\alpha>0, \alpha \neq 1$:
\begin{align*}
    \sum_{t=1}^{T}\|a_{t}\|^{2}_{\mathbb{W}_{t-1}^{-\alpha}} 
    &\leq \Big(\frac{\lambda+1}{\lambda}\Big)^{\alpha}\frac{\operatorname{Tr}\Big(\int_{0}^{T} \partial\mathbb{W}(t)^{1-\alpha}\Big)}{1-\alpha} = \Big(\frac{\lambda+1}{\lambda}\Big)^{\alpha}\frac{\operatorname{Tr}(\mathbb{W}^{1-\alpha}_{T}-\mathbb{W}^{1-\alpha}_{0})}{1-\alpha}. 
\end{align*}
For $\alpha =1$, a similar reasoning is applied using the formula $\operatorname{Tr}(\log(A))=\log(\det A)$:
\begin{align*}
     \sum_{t=1}^{T}\|a_{t}\|^{2}_{\mathbb{W}_{t-1}^{-1}} 
     &\leq \frac{\lambda+1}{\lambda}\int_{0}^{T}\frac{\partial\log\det(\mathbb{W}(t))}{\partial t}\partial t = \frac{\lambda+1}{\lambda}\operatorname{Tr}(\log\mathbb{W}_{T}-\log\mathbb{W}_{0}) = \frac{\lambda+1}{\lambda}\log\frac{\det\mathbb{W}_{T}}{\det\mathbb{W}_{0}}.
\end{align*}
A deeper study of the eigenvalues of $\mathbb{W}^{1-\alpha}_{T}$ then yields the worst-case upper bound $d^{\alpha}(d\lambda + T)^{1-\alpha}/(1-\alpha)$ for $\alpha < 1$ and $d\lambda^{1-\alpha}/(\alpha-1)$ for $\alpha>1$, recovering more naturally and extending the results of \citep{carpentier2020elliptical}. 
Thus, analogous to the \textit{water filling process} highlighted in the MAB case in Lemma \ref{Optimization Lemma Finite}, we now consider a \textit{spectral water-filling} process \cite{ITbook} optimizing over the eigenvalues of $\psi_{\alpha}(\mathbb{W}_{T})$ with a slight abuse of notations ( $\mathbb{W}^{1-\alpha}_{T}$ and $\log\mathbb{W}_{T}$ in this discussion).
\end{remark}
Following Rem.\ref{rem 1}, for the general censored case the challenge now becomes to identify a suitable matrix operator on which the aforementioned spectral maximization can be performed. By applying Lemma \ref{Potential Reduction Linear}, we henceforth focus on the case of $\alpha=1$ for which Prop. \ref{Potential Control Linear} implies that for any policy: 
\begin{align*}
    \operatorname{Tr}\Big(\int_{0}^{T}\mathbb{W}(t)^{-1}a(t)a(t)^{\top}\partial t\Big) = \int_{0}^{T}\frac{1}{p(a(t))}\frac{\partial\log\det(\mathbb{W}(t))}{\partial t}\partial t.
\end{align*}
Next, we focus on maximizing this integral over the policy class $\Pi$ and again recover the notion of effective dimension.

\subsection{Effective Dimension in Linear Settings}

We now highlight immediate properties of the effective dimension, and then present its general study for the multi-threshold model \ref{MT_model}.

\begin{lemma}\label{unif_models}
Let us consider an uniform censorship model $p:a\mapsto \Bar{p}$. By leveraging the case of equality in the Arithmetic-Geometric inequality applied to the eigenvalues of $\mathbb{W}_{T}$, we then simply deduce the associated effective dimension $d_{\mathit{eff}}\triangleq d/\Bar{p}$:
\begin{align*}
    \max_{\pi \in \Pi}\int_{0}^{T}\frac{1}{\Bar{p}}\frac{\partial\log\det(\mathbb{W}(t))}{\partial t}\partial t
    &= d_{\mathit{eff}}\log(1+\frac{T}{\lambda d_{\mathit{eff}}}).
\end{align*}
\end{lemma}
In fact, the logarithmic scaling of this quantity persists while moving beyond the uniform censorship assumption. This also highlights the importance of the leading dimension factor, crudely upper bounded by $d/p_{\mathit{min}}$ in the next lemma:
\begin{restatable}{lemma}{logscaling}\label{log_scaling}
For \textit{any} censorship function $p$, by introducing lower and upper bounds $(p_{\mathit{min}},p_{\mathit{max}})$ of $p$, we have:
\begin{align*}
    \frac{d}{p_{\mathit{max}}}\log(1+\frac{p_{\mathit{min}}T}{d\lambda}) \leq \max_{\pi \in \Pi} \int_{0}^{T}\frac{1}{p(a(t))}\frac{\partial\log\det(\mathbb{W}(t))}{\partial t}\partial t
    \leq \frac{d}{p_{\mathit{min}}}\log(1+\frac{p_{\mathit{max}}T}{d\lambda}).
\end{align*}
\end{restatable}

 
Related problems in the Generalized Linear Models literature \cite{ZhouGLM,li2017provably,NIPS2010_c2626d85} are implicitly solved in the spirit of Lemma \ref{log_scaling}, where a minimum of the derivative of the link function plays the role of $p_{\mathit{min}}$ above. However, when the function $p$ varies with action $a$, a more careful analysis is required to derive useful dimensional bounds. Our next major result addresses this gap in the literature by improving the bounds provided in Lemma  \ref{log_scaling}:
\begin{restatable}{theorem}{THMLinearOptimMTM} \label{THM Linear Optim MTM}
For a multi-threshold censorship model \ref{MT_model}, we have: 
\begin{align*}
    \max_{\pi \in \Pi}  \int_{0}^{T}\frac{1}{p(a(t))}\frac{\partial\log\det(\mathbb{W}(t))}{\partial t}\partial t =  d_{\mathit{eff}}\log(T) + o(\log(T)), \tag{$\mathcal{P}$}\label{optim_prob}
\end{align*}
where $d_{\mathit{eff}}$ is the effective dimension. 
Furthermore, $d_{\mathit{eff}}$ is characterized by two cases:
\begin{itemize}
    \item \textbf{Case 1:} Single region $j$ effective dimension $d_{\mathit{eff}} = \frac{d}{p_{j}}$.
    \item \textbf{Case 2:} Bi-region $(i,j)$ effective dimension, with $i<j$:
    \begin{align*}
        d_{\mathit{eff}}&=\frac{1}{p_{j}}\left[(d-1) \frac{1-l(i,j)}{\frac{p_{i}}{p_{j}}-l(i,j)}+\frac{u(i,j)-1}{u(i,j)-\frac{p_{i}}{p_{j}}}\right] < \frac{d}{p_{j}}. \tag{$\mathcal{D}$}\label{bi_reg}
    \end{align*}
    where $l(i,j) \triangleq \frac{\sin^{2}(\phi_{i})}{\sin^{2}(\phi_{j})}$ and $u(i,j)  \triangleq \frac{\cos^{2}(\phi_{i})}{\cos^{2}(\phi_{j})}$.
\end{itemize}
\end{restatable}
The implications of these cases are further discussed in Fig.\ref{deff} in App. \ref{Proof Multi}. Notice that a necessary condition for the bi-region $(i,j)$ effective dimension to arise is the constraint on $\frac{p_{i}}{p_{j}}$: 
\begin{align*}
\max(1,\underbrace{\frac{d l(i,j)u(i,j)}{u(i,j)+(d-1) l(i,j)}}_{\triangleq s^{\star}(i,j)}) < \frac{p_{i}}{p_{j}} < \underbrace{\frac{(d-1) u(i,j)+l(i,j)}{d}}_{\triangleq r^{\star}(i,j)}
\end{align*}
In the limit $\frac{p_{i}}{p_{j}}\rightarrow r^{\star}(i,j)$, $d_{\mathit{eff}}$ goes again to $d/p_{j}$. We interpret this limiting case as \textit{locally hard} in the sense that censorship in region $j$ is sufficiently important in comparison to all other regions to impose a maximal effective dimension to the problem, irrespective of the values of $p_{i}$, matching Lemma \ref{log_scaling}. On the other hand, for the other limiting case (under additional mild assumptions), we find that $d_{\mathit{eff}}$ also goes to $d/p_{j}$, but now for a \textit{uniformly hard} reason: that is, censorship is approximately constant and equal to $p_{j}$, recovering the Lemma \ref{unif_models}. Finally, in between these two extremes lies the \textit{minimum effective dimension} for a given value of $\frac{p_{i}}{p_{j}}$. 

\subsection{Temporal dynamics of $\mathbb{W}(t)$}\label{temp dyn}

The proof of Thm. \ref{THM Linear Optim MTM} requires the characterization of the dynamics of the optimal policy of (\ref{optim_prob}). Importantly, we discover that the evolution of $\mathbb{W}(t)$ is described by two qualitatively different regimes as outlined next. It turns out that our continuous approach to analyzing cumulative censored potential is an important tool to obtaining this result.
\paragraph{Transient Regime:}
There exists a decreasing sequence of censorship regions $\{i_{1}=k,\dots,i_{l}\}$ of length $l \in [k+1]$ and associated time sequence $\{t_{0}\triangleq 0,t_{1},\dots,t_{l}\}$ such that whenever $t_{j}\leq t \leq t_{j+1}$ for a given index $j\leq l-1$, the evolution of $\mathbb{W}(t)$ is given by:
    \begin{align*}
        \mathbb{W}(t) &=  p_{i_{j+1}}(t-t_{j})\mathbb{W}_{i_{j+1}} + \mathbb{W}(t_{j}) = p_{i_{j+1}}(t-t_{j})\mathbb{W}_{i_{j+1}} + \sum_{n=1}^{j} p_{i_{n}}(t_{n}-t_{n-1})\mathbb{W}_{i_{n}} + \lambda \mathbb{I}_{d},
    \end{align*}
    where $\mathbb{W}_{i}$ denotes the $d\times d$ diagonal matrix $ \text{diag}(\frac{\cos^{2}(\phi_{i})}{d-1},\dots,\frac{\cos^{2}(\phi_{i})}{d-1},\sin^{2}(\phi_{i}))$.
 Interestingly, the initial misspecification of censorship is self-corrected during this transient step but at an extra cost. This characterization of transient regime highlights an important consequence of using classical algorithms in censored environments.
\paragraph{Steady State Regime:} Post-transient regime, the dynamics of $\mathbb{W}(t)$ enter a steady state regime, where one of the two cases necessarily arise:\footnote{These cases are fully characterized in terms of parameters of censorship model in Lemmas \ref{One-step Transient Analysis}, \ref{Dual Reachability Analysis}, \ref{Bi-Region Effective Dimension} and Cor. \ref{Path Formula}.}.  
\begin{itemize}
    \item \textbf{Case 1: Single region $i_{l}$.} This case arises when the last element of the time sequence $t_{l}$ is equal to $+\infty$ and we have the single region evolution for all $t\geq t_{l-1}$:
    \begin{align*}
        \mathbb{W}(t) &=  p_{i_{l}}(t-t_{l-1})\mathbb{W}_{i_{l}} + \mathbb{W}(t_{l-1}) =  p_{i_{l}}(t-t_{l-1})\mathbb{W}_{i_{l}} + \sum_{n=1}^{l-1} p_{i_{n}}(t_{n}-t_{n-1})\mathbb{W}_{i_{n}}  + \lambda \mathbb{I}_{d}.
    \end{align*}
    The effective dimension corresponding to this dynamics is $d/p_{i_{l}}$, with the following equality for $T\geq t_{l-1}$:
    \begin{align*}
        \int_{0}^{T}\frac{1}{p(a(t))}\frac{\partial\log\det(\mathbb{W}(t))}{\partial t}\partial t = \frac{1}{p_{i_{l}}}\log\det(\mathbb{W}(T))+ \sum_{n=1}^{l-1} (\frac{1}{p_{i_{n}}}-\frac{1}{p_{i_{n+1}}})\log\det\mathbb{W}(t_{n}).
    \end{align*}
    \item \textbf{Case 2: Bi-region $(i_{l+1},i_{l})$.} This case arises when the steady-state dynamics of $\mathbb{W}(t)$ span the two regions $(i_{l+1},i_{l})$ with $i_{l+1}<i_{l}$. For all $t\geq t_{l}$, we have the evolution:
    \begin{align*}
         \mathbb{W}(t) &\propto p_{i_{l+1}}(t+\lambda^{\star})\begin{pmatrix}
\cos^{2}(\phi_{i_{l}})(u(i_{l+1},i_{l})-\frac{p_{i_{l+1}}}{p_{i_{l}}})\mathbb{I}_{d-1} & (0) \\
(0)  &\sin^{2}(\phi_{i_{l}})(\frac{p_{i_{l+1}}}{p_{j}}-l(i_{l+1},i_{l})) 
\end{pmatrix}.
    \end{align*}
    where $\lambda^{\star}$ and the proportionality factor are specified in SI. The corresponding effective dimension is given by (\ref{bi_reg}) and the following equality holds for all $T\geq t_{l}$:
    \begin{align*}
        \int_{0}^{T}\frac{1}{p(a(t))}\frac{\partial\log\det(\mathbb{W}(t))}{\partial t}\partial t = d_{\mathit{eff}}\log(1+\frac{T-t_{l}}{t_{l}+\lambda^{\star}}) + \sum_{n=1}^{l} (\frac{1}{p_{i_{n}}}-\frac{1}{p_{i_{n+1}}})\log\det\mathbb{W}(t_{n}). 
    \end{align*}
\end{itemize}

For further discussions on transient and steady state regimes, we refer to Fig.\ref{reach_stat}, \ref{reach} and \ref{switch}. in App. \ref{Proof Multi}.


\section{Concluding Remarks}

In this work, we demonstrate that the complexity of bandit learning under censorship is governed by the notion of effective dimension. To do so, we developed a novel analysis framework which enables us to precisely estimate this quantity for a broad class of multi-threshold censorship models. An important future work would be to extend our model and approach to Bayesian settings, which will likely provide us with useful insights on the cumulative censored potential $\mathbb{V}_{\alpha}$, as initiated by \cite{hamidi2021randomized}. Future work also includes relaxing the Missing Completely at Random (MCAR) property in favor of time-dependent censorship models such as Markov Decision Processes (MDPs). We believe that tools similar to those developed in our potential-based analysis can be applied in this case. 
Finally, the contributions of our work may be of interest to the recent value alignment literature, where the question of learnability under humain-AI interactions is central. \cite{mab_val_align, hadfieldmenell2016cooperative,christiano2017deep}.

We do not envision any negative societal impacts of our work other than that of bandits algorithms deployed in AI-driven platforms. 


\section*{Acknowledgments and Disclosure of Funding}

This research project is supported by the AFOSR FA9550-19-1-0263 “Building attack resilience into complex networks” Grant. The authors would like to thank Prem Talwai and the anonymous reviewers for providing insightful comments and suggestions.


\bibliographystyle{plain}
\bibliography{main.bib}

\section*{Checklist}


\begin{enumerate}

\item For all authors...
\begin{enumerate}
  \item Do the main claims made in the abstract and introduction accurately reflect the paper's contributions and scope?
    \answerYes{}
  \item Did you describe the limitations of your work?
    \answerYes{}
  \item Did you discuss any potential negative societal impacts of your work?
    \answerYes{} 
  \item Have you read the ethics review guidelines and ensured that your paper conforms to them?
    \answerYes{}
\end{enumerate}

\item If you are including theoretical results...
\begin{enumerate}
  \item Did you state the full set of assumptions of all theoretical results?
    \answerYes{} See Sec. $\ref{MAB}$ and $\ref{CB}$ as well as SI.
        \item Did you include complete proofs of all theoretical results?
    \answerYes{} See SI.
\end{enumerate}

\item If you ran experiments...
\begin{enumerate}
  \item Did you include the code, data, and instructions needed to reproduce the main experimental results (either in the supplemental material or as a URL)?
    \answerNA{}
  \item Did you specify all the training details (e.g., data splits, hyperparameters, how they were chosen)?
    \answerNA{}
        \item Did you report error bars (e.g., with respect to the random seed after running experiments multiple times)?
    \answerNA{}
        \item Did you include the total amount of compute and the type of resources used (e.g., type of GPUs, internal cluster, or cloud provider)?
    \answerNA{}
\end{enumerate}

\item If you are using existing assets (e.g., code, data, models) or curating/releasing new assets...
\begin{enumerate}
  \item If your work uses existing assets, did you cite the creators?
    \answerNA{}
  \item Did you mention the license of the assets?
    \answerNA{}
  \item Did you include any new assets either in the supplemental material or as a URL?
    \answerNA{}
  \item Did you discuss whether and how consent was obtained from people whose data you're using/curating?
    \answerNA{}
  \item Did you discuss whether the data you are using/curating contains personally identifiable information or offensive content?
    \answerNA{}
\end{enumerate}

\item If you used crowdsourcing or conducted research with human subjects...
\begin{enumerate}
  \item Did you include the full text of instructions given to participants and screenshots, if applicable?
    \answerNA{}
  \item Did you describe any potential participant risks, with links to Institutional Review Board (IRB) approvals, if applicable?
    \answerNA{}
  \item Did you include the estimated hourly wage paid to participants and the total amount spent on participant compensation?
    \answerNA{}
\end{enumerate}

\end{enumerate}


\newpage

\appendix

\section{Preliminaries}\label{prel}

In this section, we first provide the instances of the UCB algorithm used in Sec.\ref{MAB} and Sec.\ref{CB}. We also indicate in Tab.\ref{tab:TableOfNotationForMyResearch} the notations used throughout the paper to help the reader. 




\begin{table}[htbp]\caption{Summary of Notations}
\centering 
\begin{tabular}{r c p{10cm} }
\toprule
\multicolumn{3}{c}{}\\
\multicolumn{3}{c}{\underline{Bandit Problem Variables}}\\
\multicolumn{3}{c}{}\\
$T$ & $\triangleq$ & Total number of rounds of the sequential decision-making problem.  \\
$d$ & $\triangleq$ & Number of arms in Sec.\ref{MAB}, Dimension of action feature vector in Sec.\ref{CB}.\\
$(\mathcal{A}_{t},\mathcal{A})$ & $\triangleq$ & Action set at time $t$; Union of all action sets $\mathcal{A}_{t}$.\\
$a_{t}$ & $\triangleq$ & Action picked at time $t$; selected by policy $\pi$, seen as a function of previous history.  \\
$(\epsilon_{t},\sigma^{2})$ & $\triangleq$ & Stochastic feedback noise a time $t$. Sub-Gaussian with pseudo-variance parameter $\sigma^{2}$. If $\sigma^{2}$ depends on the action selected (heteroskedasticity), we use $\sigma^{2}_{a}$ instead. \\
$(r,\theta^{\star})$ & $\triangleq$ & Unknown reward function, maps action to scalar reward. Parameterized by unknown latent state $\theta^{\star}$.\\
$\Delta_{t}(a)$ & $\triangleq$ & Sub-optimality gap of action $a$ at time $t$, reward difference with optimal decision of clairvoyant policy \\
$(\Delta_{a},\Delta_{\max})$ & $\triangleq$ & If $\Delta_{t}(a)$ is independent of $t$, we use  $\Delta_{a}\equiv \Delta_{t}(a)$. $\Delta_{\max}$ is an upper bound of $\Delta_{t}(a)$ for all actions $a$ and time $t$.\\
$R(T,\pi)$ & $\triangleq$ & Pseudo regret of policy $\pi$ over $T$ rounds.  \\
\multicolumn{3}{c}{}\\
\multicolumn{3}{c}{\underline{Censorship Variables}}\\
\multicolumn{3}{c}{}\\
$p_{a}$ & $\triangleq$ & Probability that action $a$ is censored if selected, used in Sec. \ref{MAB}. Notation $p(a)$ is used in Sec.\ref{CB} to emphasize the dependency of $p$ on action $a$. \\
$(\phi_{j},u,p_{j})$ & $\triangleq$ & Parameters of the multi-threshold censorship model. Vector $u$ defines the direction of censorship, $(\phi_{j})_{j\leq k+1}$ define the censorship regions with fixed censorship probability and $(p_{j})_{j\leq k}$ define the probability of being censored for each region $j$.\\
$x_{a_{t}}$ & $\triangleq$ & Random variable indicating if feedback is censored as round $t$. Follows i.i.d Bernoulli distribution of parameter $p(a_{t})$.  \\
\multicolumn{3}{c}{}\\
\multicolumn{3}{c}{\underline{Algorithmic and Analysis Variables}}\\
\multicolumn{3}{c}{}\\
$\lambda$ & $\triangleq$ & Regularization tuning parameter. $\lambda_{a}$ is used if heterogeneous action-based regularization. \\
$\Tilde{\Delta}^{\lambda}_{t}(a)$ & $\triangleq$ & High-probability upper bound on the sub-optimality gap, used in UCB algorithms.  \\
$\mathbb{V}_{\alpha}(T,\pi)$ & $\triangleq$ & Random cumulative censored potential, seen as a function of policy $\pi$ and number of rounds $T$. First introduced in Sec.\ref{MAB} and extended in Sec.\ref{CB}.  \\
$\psi_{\alpha}$ & $\triangleq$ & Primitive of the function $x\mapsto x^{-\alpha}$, for a given $\alpha>0$.  \\
$N_{a}(t)$ & $\triangleq$ & Total number of time action $a$ is \textit{realized} at the end of round $t$  by policy $\pi$. Used in Sec.\ref{MAB}.\\
$\tau_{a}(t)$ & $\triangleq$ & Total number of time action $a$ is \textit{played} at the end of round $t$ by policy $\pi$. Used in Sec.\ref{MAB}.\\
$\mathbb{W}_{t}^{C}$ & $\triangleq$ & Censored Design Matrix. Linear generalization of $(N_{a}(t))_{a\in[d]}$. Used in Sec.\ref{CB}.\\
$\mathbb{W}_{t}$ & $\triangleq$ & Expected Design Matrix. Linear generalization of $(p_{a}\tau_{a}(t))_{a\in[d]}$. Used in Sec.\ref{CB}.\\
$\mathbb{W}(t)$ & $\triangleq$ & Continuous generalization of the expected design matrix$\mathbb{W}_{t}$.\\
\bottomrule
\end{tabular}
\label{tab:TableOfNotationForMyResearch}
\end{table}


\subsection{UCB algorithms}\label{UCB-algo}
\begin{itemize}
    \item \textbf{UCB-MAB:} Following \cite{lattimore2020bandit}, the UCB algorithms for the MAB case with homogeneous regularization $\lambda>0$ uses the following optimistic reward estimator at time $t$:
    \begin{align*}
        \Tilde{r}^{\lambda}_{t}(a) \triangleq \hat{\theta}^{\lambda}_{t}(a) + \sqrt{\frac{6\sigma^{2}\log(T)}{\lambda+ N_{a}(t-1)}} + \frac{\lambda\|\theta^{\star}\|_{\infty}}{\lambda+ N_{a}(t-1)}.
    \end{align*}
    It is based on the use of the regularized empirical mean to estimate the reward of action $a$ at the end of round $t$:
    \begin{align*}
        \hat{\theta}^{\lambda}_{t}(a) &\triangleq \frac{1}{N_{a}(t) + \lambda}\sum_{\tau=1}^{t} (r(a_{\tau})+\tau) \mathbf{1}\{a_{\tau}=a, x_{a_{\tau}}=1\} \\
        &=  \frac{N_{a}(t)}{N_{a}(t) + \lambda}\theta_{a}^{\star}+\frac{1}{N_{a}(t) + \lambda}\sum_{\tau=1}^{t} \epsilon_{a_{\tau}} \mathbf{1}\{a_{\tau}=a, x_{a_{\tau}}=1\}.
    \end{align*}
    The high-confidence property of this algorithm is proven in Lemma \ref{Fail Optim Finite}.\footnote{Typically, an upper bound on $\|\theta^{\star}\|_{\infty}$ for MAB (resp. $\|\theta^{\star}\|_{2}$ for LCB) is used instead of this unknown quantity. We keep $\|\theta^{\star}\|_{\infty}$ (resp. $\|\theta^{\star}\|_{2}$) not to overload notations but our results immediately extends to the use of the latter.}
    Under a-priori known heteroskedasticity, the reward estimator can be expressed as:
    \begin{align*}
        \Tilde{r}^{\lambda}_{t}(a) \triangleq \hat{\theta}^{\lambda}_{t}(a) + \sqrt{\frac{6\sigma_{a}^{2}\log(T)}{\lambda+ N_{a}(t-1)}} + \frac{\lambda\|\theta^{\star}\|_{\infty}}{\lambda+ N_{a}(t-1)}.
    \end{align*}
    \item \textbf{UCB for LCB} Following \cite{NIPS2011_e1d5be1c,lattimore2020bandit}, the UCB algorithms for the LCB case with homogeneous regularization $\lambda>0$ uses the following optimistic reward estimator at time $t$:
    \begin{align*}
        \Tilde{r}^{\lambda}_{t}(a) &\triangleq
        \langle a, \hat{\theta}^{\lambda}_{t-1}\rangle + \beta_{t-1}(\delta)\|a\|_{\mathbb{W}^{C}_{t-1}},
    \end{align*}
    where we introduced the random quantity:
    \begin{align*}
        \beta_{t-1}(\delta) \triangleq \sqrt{\sigma^{2} \log \left(\frac{\det(\mathbb{W}^{C}_{t-1})}{\det(\lambda \mathbb{I}_{d})}\right)+2\sigma^{2}\log(\frac{1}{\delta})}+\sqrt{\lambda} \|\theta^{\star}\|_{2}
    \end{align*}
    It is based on the use of the regularized least square estimator to estimate the vector $\theta^{\star}$ at the end of round $t$:
    \begin{align*}
        \hat{\theta}^{\lambda}_{t} = (\mathbb{W}^{C}_{t})^{-1}\sum_{\tau=1}^{t}(\epsilon_{\tau}+\langle a_{\tau},\theta^{\star}\rangle)x_{a_{\tau}}a_{\tau}
    \end{align*}
    The high-confidence property of this estimator is proven in Lemma \ref{Optimistic Lemma Linear}.
\end{itemize}




\section{Proof of Sec. \ref{MAB} - Multi-Armed Bandits} \label{Proof MAB}

In this section, we prove the results in Sec.\ref{MAB} on the MAB case. We start by proving Lemmas \ref{Potential Reduction Finite}, \ref{Fail Optim Finite}, \ref{Derando Censo Linear}, \ref{Optimization Lemma Finite} and Prop. \ref{Potential Control Finite}. Thanks to those results, we then tackle Thm. \ref{THM Finite arms} and Prop. \ref{Instance Dep Regret Finite}. To conclude the section, we further study the properties of the adaptivity gain, by proving Lemma \ref{asympt_off} and Prop. \ref{Monitoring AG}. Recall that effective dimension $d_{\mathit{eff}}$ is referring to $\sum_{a\in[d]}\frac{1}{p_{a}}$ in this section.

\subsection{Proof of Lemma \ref{Potential Reduction Finite}}

\PotentialReductionFinite*

\begin{proof}
At a given round $t\in [T]$, we have under the event $\neg \mathcal{H}_{\text{UCB}}^{\lambda}$ introduced in Lemma \ref{Fail Optim Finite}: 
\begin{align*}
     \Delta_{t}(a_{t}) = \max_{a\in\mathcal{A}_{t}}\theta^{\star}_{a} - \theta^{\star}_{a_{t}} \leq 2\sqrt{6\sigma^{2} \frac{\log(T)}{N_{a_{t}}(t-1)+\lambda}} + 2\frac{\lambda \|\theta^{\star}\|_{\infty}}{\lambda  + N_{a_{t}}(t-1)},
\end{align*}
where the inequality comes from the definition of the UCB algorithm and the conditioning on $\neg \mathcal{H}_{\text{UCB}}^{\lambda}$. We find there the origin of the two different orders of $N_{a}$ ($\nicefrac{1}{2}$ and $1$). Taken independently, those lead to a contribution of respectively $\mathcal{O}(d_{\mathit{eff}}\log(T))$ and $\mathcal{O}(\sqrt{d_{\mathit{eff}}T})$ . More precisely, we have:
\begin{align*}
    R(T,\pi_{\text{UCB}}|\neg \mathcal{H}_{\text{UCB}}^{\lambda}) &\leq 2 \sqrt{6\sigma^{2}\log(T)} \sum_{t=1}^{T}\sqrt{\frac{1}{N_{a_{t}}(t-1)+\lambda}} + 2\lambda\|\theta^{\star}\|_{\infty} \sum_{t=1}^{T}\frac{1}{N_{a_{t}}(t-1)+\lambda} \\
    &= 2 \sqrt{6\sigma^{2}\log(T)}\mathbb{V}_{\frac{1}{2}}(T,\pi_{\text{UCB}}) + 2\lambda\|\theta^{\star}\|_{\infty}\mathbb{V}_{1}(T,\pi_{\text{UCB}}).
\end{align*}
Therefore, thanks to Lemma \ref{Fail Optim Finite}, we deduce that:
\begin{align*}
    R(T,\pi_{\text{UCB}}) &\leq (1-\mathbb{P}(\mathcal{H}_{\text{UCB}}^{\lambda}))R(T,\pi_{\text{UCB}}|\neg \mathcal{H}_{\text{UCB}}^{\lambda}) + \mathbb{P}(\mathcal{H}_{\text{UCB}}^{\lambda})\Delta_{max}T\\
    &\leq 2 \sqrt{6\sigma^{2}\log(T)}\mathbb{V}_{\frac{1}{2}}(T,\pi_{\text{UCB}}) + 2\lambda\|\theta^{\star}\|_{\infty}\mathbb{V}_{1}(T,\pi_{\text{UCB}}) + \frac{2d\Delta_{max}}{T}.
\end{align*}
Finally, we conclude that:
\begin{align*}
    \mathbb{E}[R(T,\pi_{\text{UCB}})] \leq 2 \sqrt{6\sigma^{2}\log(T)}\mathbb{E}[\mathbb{V}_{\frac{1}{2}}(T,\pi_{\text{UCB}})] + 2\lambda\|\theta^{\star}\|_{\infty}\mathbb{E}[\mathbb{V}_{1}(T,\pi_{\text{UCB}})] + \frac{2d\Delta_{max}}{T}.
\end{align*}
\end{proof}
\subsection{Statement and Proof of Lemma \ref{Fail Optim Finite}}

The main step in this reduction from regret to cumulative censored potential is the study of the \textit{failure of optimism} event thanks to the following result: 
\begin{lemma}\label{Fail Optim Finite} For a regularization $\lambda>0$ and $\delta \in ]0,1]$, we introduce the event:
\begin{align*}
    \mathcal{H}_{\text{UCB}}^{\lambda} = \Big\{\exists a \in [d], t\in[T],  |\hat{\theta}^{\lambda}_{t}(a) -\theta^{\star}_{a}| > \sqrt{\frac{6\sigma^{2}\log(T)}{\lambda+ N_{a}(t)}} + \frac{\lambda\|\theta^{\star}\|_{\infty}}{\lambda+ N_{a}(t)}\Big\}.
\end{align*}
We then have $\mathbb{P}(\mathcal{H}_{\text{UCB}}^{\lambda}) \leq \frac{2d}{T^{2}}$.
\end{lemma}

\begin{proof}
Although this event is similar to the one introduced in the classical UCB proof idea, the subtlety comes from the randomness induced by the censorship as well as the impact of regularization. The main idea is adopt a worst-case agnostic approach. First, let's note that for a given $t \in [T],a \in [d]$, we have:
\begin{align*}
    |\hat{\theta}^{\lambda}_{t}(a) -\theta^{\star}_{a}| &= |\frac{1}{N_{a}(t) +\lambda}\sum_{\tau=1}^{t} \epsilon_{\tau} \mathbf{1}\{a_{\tau}=a, x_{a_{\tau}}=1\} -\frac{\lambda}{N_{a}(t)+\lambda}\theta^{\star}_{a}| \\
    &\leq |\frac{1}{N_{a}(t) +\lambda}\sum_{\tau=1}^{t} \epsilon_{\tau} \mathbf{1}\{a_{\tau}=a, x_{a_{\tau}}=1\}| + \frac{\lambda}{N_{a}(t) +\lambda}\|\theta^{\star}\|_{\infty} .
\end{align*}
Therefore, for a given $a\in[d], t\in [T]$, by introducing the event $\mathcal{B}_{(t,a)}\triangleq\Big\{|\hat{\theta}^{\lambda}_{t}(a) -\theta^{\star}_{a}| > \sqrt{\frac{6\sigma^{2}\log(T)}{\lambda+ N_{a}(t)}} + \frac{\lambda\|\theta^{\star}\|_{\infty}}{\lambda+ N_{a}(t)}\Big\}$, we deduce:
\begin{align*}
    \mathcal{B}_{(t,a)} &\subset \Big\{|\frac{1}{N_{a}(t) +\lambda}\sum_{\tau=1}^{t} \epsilon_{\tau} \mathbf{1}\{a_{\tau}=a, x_{a_{\tau}}=1\}| + \frac{\lambda}{N_{a}(t) +\lambda}\|\theta^{\star}\|_{\infty}  > \sqrt{\frac{6\sigma^{2}\log(T)}{\lambda+ N_{a}(t)}} + \frac{\lambda\|\theta^{\star}\|_{\infty}}{\lambda+ N_{a}(t)}\Big\} \\
     &\subset \Big\{|\frac{1}{N_{a}(t) +\lambda}\sum_{\tau=1}^{t} \epsilon_{\tau} \mathbf{1}\{a_{\tau}=a, x_{a_{\tau}}=1\}|> \sqrt{\frac{6\sigma^{2}\log(T)}{\lambda+ N_{a}(t)}} \Big\}.
\end{align*}
Then, we have:
\begin{align*}
    \mathbb{P}(\mathcal{H}_{\text{UCB}}^{\lambda}) &= \displaystyle \mathbb{P}\Big(\bigcup_{a\in[d]}\bigcup_{t\in[T]} \mathcal{B}_{(t,a)}\Big)\\
    &\leq \mathbb{P}\Big(\bigcup_{a\in[d]}\bigcup_{t\in[T]}\Big\{|\frac{1}{N_{a}(t) +\lambda} \sum_{\tau=1}^{t} \epsilon_{\tau} \mathbf{1}\{a_{\tau}=a, x_{a_{\tau}}=1\}|> \sqrt{\frac{6\sigma^{2}\log(T)}{\lambda+ N_{a}(t)}} \Big\}\Big)\\
    &\leq \sum_{a\in[d]} \mathbb{P}\Big(\bigcup_{t\in[T]} \Big\{|\frac{1}{N_{a}(t) +\lambda} \sum_{\tau=1}^{t} \epsilon_{\tau} \mathbf{1}\{a_{\tau}=a, x_{a_{\tau}}=1\}|> \sqrt{\frac{6\sigma^{2}\log(T)}{\lambda+ N_{a}(t)}} \Big\}\Big) \\
    & = \sum_{a\in[d]}\mathbb{P}\Big(\bigcup_{k\in[T]} \bigcup_{t\in[T]} \Big\{|\frac{1}{N_{a}(t) +\lambda} \sum_{\tau=1}^{t} \epsilon_{\tau} \mathbf{1}\{a_{\tau}=a, x_{a_{\tau}}=1\}|^{2}> \frac{6\sigma^{2}\log(T)}{\lambda+ N_{a}(t)}; N_{a}(t)=k\Big\}\Big)\\
    & = \sum_{a\in[d]}\sum_{k\in[T]}\mathbb{P}(N_{a}(t)=k)\mathbb{P}\Big(\bigcup_{t\in[T]} \Big\{ |\frac{1}{k +\lambda} \sum_{\tau=1}^{t} \epsilon_{\tau} \mathbf{1}\{a_{\tau}=a, x_{a_{\tau}}=1\}|^{2} >\frac{6\sigma^{2}\log(T)}{k}\Big| N_{a}(t)=k\Big\}\Big)\\
    & \leq \sum_{a\in[d]} \sum_{k\in[T]}\mathbb{P}\Big(\bigcup_{t\in[T]}\Big\{|\frac{1}{k +\lambda} \sum_{\tau=1}^{t} \epsilon_{\tau} \mathbf{1}\{a_{\tau}=a, x_{a_{\tau}}=1\}|^{2} >\frac{6\sigma^{2}\log(T)}{\lambda+k}\Big| N_{a}(t)=k\Big\}\Big) \\
    &= \sum_{a\in[d]} \sum_{k\in[T]}\mathbb{P}\Big(|\frac{\sum_{l=1}^{k} \epsilon_{l}}{k +\lambda}|^{2} >\frac{6\sigma^{2}\log(T)}{\lambda+k}\Big),
\end{align*}
where we successively used union bounds over the action set and number of realizations and conditioned over number of realizations $k$. We re-indexed the random sub-Gaussian variables $(\epsilon_{t})$ for last expression thanks to the i.i.d property. Then, for a given $k$, using Hoeffding inequality for sub-Gaussian variables, we have:
\begin{align*}
    \mathbb{P}\Big(|\frac{\sum_{l=1}^{k}\epsilon_{l}}{k+\lambda}|^{2}>\frac{6\sigma^{2}\log(T)}{k+\lambda}\Big) &= \mathbb{P}\Big(|\sum_{l=1}^{k}\epsilon_{l}|>\sqrt{6\sigma^{2}(k+\lambda)\log(T)}\Big)
    &\leq 2\exp\{-\frac{6\sigma^{2}(k+\lambda)\log(T)}{2k\sigma^{2}}\} \\
    &\leq \frac{2}{T^{3}}
\end{align*}
where the used that fact that $\sum_{l=1}^{k}\epsilon_{l}$ is sub-Gaussian of pseudo-variance parameter $k\sigma^{2}$
Therefore, this yields:
\begin{align*}
    \sum_{a\in[d]} \sum_{k\in[T]} \mathbb{P}\Big(|\frac{\sum_{l=1}^{k}\epsilon_{l}}{k+\lambda}|^{2}>\frac{6\sigma^{2}\log(T)}{k+\lambda}\Big) \leq \frac{2d}{T^{2}}.
\end{align*}
Finally, we conclude that $\mathbb{P}(\mathcal{H}_{\text{UCB}}^{\lambda}) \leq \frac{2d}{T^{2}}$. 
\end{proof}
\begin{remark}\label{Tails}
We note that assuming tails distribution for the reward noise $\epsilon$ of the form:
\begin{align*}
    \mathbb{P}\left(\epsilon \geq x\right) \leq\exp\Big\{\frac{-x^{1+q}}{2\sigma^{2}}\Big\}
\end{align*}
for a given $q>0$, as suggested for instance in \cite{ZhouGLM}, would lead the use of the confidence interval:
\begin{align*}
    \mathcal{H}_{\text{UCB}}^{\lambda,q} = \Big\{\exists a \in [d], t\in[T],  |\hat{\theta}^{\lambda}_{t}(a) -\theta^{\star}_{a}| > \Big(6\sigma^{2}\log(T)\Big)^{\frac{1}{1+q}}\Big(\lambda+ N_{a}(t)\Big)^{-\frac{q}{1+q}} + \frac{\lambda\|\theta^{\star}\|_{\infty}}{\lambda+ N_{a}(t)}\Big\}.
\end{align*}
Indeed, the same reasoning as above would then yield:
\begin{align*}
    \mathbb{P}\Big(|\frac{\sum_{l=1}^{k}\epsilon_{l}}{k+\lambda}|>(6\sigma^{2}\log(T))^{\frac{1}{1+q}}(k+\lambda)^{-\frac{a}{1+q}}\Big) &= \mathbb{P}\Big(|\sum_{l=1}^{k}\epsilon_{l}|>(6\sigma^{2}(k+\lambda)\log(T))^{\frac{1}{1+q}}\Big)\\ &\leq 2\exp\{-\frac{6\sigma^{2}(k+\lambda)\log(T)}{2k\sigma^{2}})\} \leq \frac{2}{T^{3}}
\end{align*}
and therefore $\mathbb{P}(\mathcal{H}_{\text{UCB}}^{\lambda,q})\leq \frac{2d}{T^{2}}$. For $q=1$, we recover the sub-Gaussian case, which in turns lead to the study of $\mathbb{V}_{1/2}$, as done in Lemma \ref{Potential Reduction Finite}. For general $q>0$, we would would then consider $\mathbb{V}_{q/(1+q)}$, which lead to the upper bound $\mathcal{O}(d_{\mathit{eff}}^{q/(1+q)}T^{1/(1+q)})$ through the use of Prop. \ref{Potential Control Finite}.
\end{remark}
\subsection{Statement and Proof of Lemma \ref{Derando Censo Linear}}

\begin{lemma}\label{Derando Censo Linear}  For any $\delta \in ]0,1]$, $\lambda>0$ and censorship model, let's introduce the event:
\begin{align*}
    \mathcal{H}^{I}_{\text{CEN}}(\delta) &= \left\{\exists a \in [d], t\in[T],N_{a}(t) < (1-\delta)p_{a}\tau_{a}(t) \quad \text{and} \quad \tau_{a}(t) \geq T_{0}(a) \right\},
\end{align*}
where $T_{0}(a)\triangleq 24\log(T)/p_{a}+1$. We then have $\mathbb{P}(\mathcal{H}^{I}_{\text{CEN}}(\delta)) \leq \frac{4d_{\mathit{eff}}}{\delta^{2}}T^{-12\delta^{2}}$.
\end{lemma}
\begin{proof}
First, we apply successively two unions bounds over the action set and the number of realizations, mirroring the analysis of \cite{stoch_unrest_delay}:
\begin{align*}
    \mathbb{P}(\mathcal{H}^{I}_{\text{CEN}}(\delta)) &\leq \sum_{a\in [d]} \mathbb{P}\Big(\Big\{\exists t\in[T],  \tau_{a}(t) \geq T_{0}(a),N_{a}(t) < (1-\delta)p_{a}\tau_{a}(t) \Big\}\Big) \\
    &= \sum_{a\in [d]} \mathbb{P}\Big(\bigcup_{k_{a}\in[T_{0}(a),T]} \bigcup_{t\in[T]}\Big\{\tau_{a}(t) \geq T_{0}(a),N_{a}(t) < (1-\delta)p_{a}\tau_{a}(t), \tau_{a}(t) = k_{a} \Big\}\Big) \\
    &\leq \sum_{a\in [d]} \sum_{k_{a}\geq T_{0}(a)}\mathbb{P}\Big(\bigcup_{t\in[T]}\Big\{N_{a}(t) < (1-\delta)p_{a}\tau_{a}(t)\Big|\tau_{a}(t) = k_{a}\Big\}\Big).
\end{align*}
We then use a multiplicative Chernoff inequality for Binomial Distribution to deduce:
\begin{align*}
    \sum_{a\in[d]} \sum_{k_{a}\geq T_{0}(a)} \mathbb{P}\Big(N_{a}(t) &< (1-\delta)p_{a}\tau_{a}(t)\Big|\tau_{a}(t)=k_{a}\Big) \leq \sum_{a\in[d]} \sum_{k_{a}\geq T_{0}(a)} \exp\{-\frac{\delta^{2}k_{a}p_{a}}{2}\}. 
\end{align*}
The novelty of our proof is to leverage a integral comparison to deduce the improved control:
\begin{align*}
    \sum_{a\in[d]} \sum_{k_{a}\geq T_{0}(a)} \exp\{-\frac{\delta^{2}k_{a}p_{a}}{2}\} &\leq 2 \sum_{a\in[d]} \left[-\frac{2}{\delta^{2}p_{a}}\exp\{-\frac{\delta^{2}k_{a}p_{a}}{2}\}\right]^{\tau_{a}(t)}_{T_{0}(a)-1} \\
    &\leq \frac{4}{\delta^{2}}d_{\mathit{eff}}\frac{1}{T^{12\delta^{2}}} - \frac{4}{\delta^{2}}\sum_{a\in[d]}\frac{1}{p_{a}}\exp\{-\frac{\delta^{2}\tau_{a}(t)p_{a}}{2}\} \leq \frac{4}{\delta^{2}}d_{\mathit{eff}}\frac{1}{T^{12\delta^{2}}}.
\end{align*}
Picking for instance $\delta = \frac{1}{2}$ yields $\mathbb{P}(\mathcal{H}^{I}_{\text{CEN}}(\frac{1}{2})) \leq \frac{16d_{\mathit{eff}}}{T^{3}}$. 
\end{proof}
\subsection{Proof of Lemma \ref{Optimization Lemma Finite}}

\OptimizationLemmaFinite*

\begin{proof}
We first introduce the Lagrangian of the problem $\mathcal{L}(\tau_{1},\dots,\tau_{d},\mu):= \sum_{a\in[d]} \frac{1}{p_{a}}\Big(\psi_{\alpha}(p_{a}\tau_{a}+\lambda_{a})-\psi_{\alpha}(\lambda_{a})\Big) + \mu(T-\sum_{a\in[d]}\tau_{a})$. Differentiating with respect to $\tau_{a}$ for all $a\in[d]$ yields the equations:
\begin{align*}
    \frac{1}{(p_{a}\tau_{a}+\lambda_{a})^{\alpha}} - \mu = 0.
\end{align*}
We then write it equivalently as:
\begin{align*}
    \tau_{a} = \frac{1}{p_{a}}[\mu^{-1/\alpha} - \lambda_{a}].
\end{align*}
However, since $(\tau_{a})$ must be nonnegative, it may not always be possible to find a solution of this form. We then verify using KKT conditions that the solution:
\begin{align*}
    \tau_{a} = \frac{1}{p_{a}}[C - \lambda_{a}]^{+},
\end{align*}
where $C$ ensures the total budget constraint $ \sum_{a\in[d]}\tau^{\star}_{a}=T$, is optimal. In particular, whenever $T\geq \max_{a}\lambda_{a}^{0}$, we recover the solution provided in the second part the Lemma. 
\end{proof}
\subsection{Proof of Prop. \ref{Potential Control Finite} }

\PotentialControlFinite*

\begin{proof}
For a given $\alpha\in]0,1]$, we condition on the event $\mathcal{H}^{I}_{\text{CEN}}(\delta)$ introduced in Lemma \ref{Derando Censo Linear} and consider the cases $\tau_{a}(t)\geq T_{0}(a)$ and $\tau_{a}(t)<T_{0}(a)$. This yields for any policy $\pi \in \Pi$:
\begin{align*}
    \mathbb{V}_{\alpha}(T,\pi|\mathcal{H}^{I}_{\text{CEN}}(\delta)) &\leq \frac{\sum_{a\in[d]}T_{0}(a)}{\lambda^{\alpha}} + \sum_{t=1}^{T}\left((1-\delta)p_{a_{t}}\tau_{a_{t}}(t-1)+\lambda\right)^{-\alpha}\\
    &\leq \frac{24 d_{\mathit{eff}}\log(T)+d}{\lambda^{\alpha}} + \frac{1}{(1-\delta)^{\alpha}}\sum_{t=1}^{T}\left(p_{a_{t}}\tau_{a_{t}}(t-1)+\frac{\lambda}{1-\delta}\right)^{-\alpha}\\
    &\leq \frac{24 d_{\mathit{eff}}\log(T)+d}{\lambda^{\alpha}} + \frac{1}{(1-\delta)^{\alpha}} \sum_{a\in[d]}\int_{0}^{\tau_{a}(T)}\left(p_{a}u+\frac{\lambda}{1-\delta}\right)^{-\alpha}\partial u\\
    &= \frac{24 d_{\mathit{eff}}\log(T)+d}{\lambda^{\alpha}} + \frac{1}{(1-\delta)^{\alpha}} \sum_{a\in[d]}\frac{1}{p_{a}}[\psi_{\alpha}(p_{a}\tau_{a}(T)+\frac{\lambda}{1-\delta})-\psi_{\alpha}(\frac{\lambda}{1-\delta})].
\end{align*}
We then apply the Lemma \ref{Optimization Lemma Finite} with constant $\Tilde{\lambda} \triangleq \lambda/(1-\delta)$ to deduce:
\begin{align*}
    \max_{\pi \in \Pi} \mathbb{V}_{\alpha}(T,\pi|\neg \mathcal{H}^{I}_{\text{CEN}}(\delta)) \leq \frac{24 d_{\mathit{eff}}\log(T)+d}{\lambda^{\alpha}} + \frac{d_{\mathit{eff}}}{(1-\delta)^{\alpha}}\Big[\psi_{\alpha}(\frac{T}{d_{\mathit{eff}}}+\frac{\lambda}{1-\delta}) - \psi_{\alpha}(\frac{\lambda}{1-\delta})\Big].
\end{align*}
Then, we conclude thanks to Lemma \ref{Derando Censo Linear} that:
\begin{align*}
    \max_{\pi \in \Pi} \mathbb{E}[\mathbb{V}_{\alpha}(T,\pi)] &\leq \mathbb{P}(\neg \mathcal{H}^{I}_{\text{CEN}}(\delta))\max_{\pi \in \Pi} \mathbb{V}_{\alpha}(T,\pi|\neg \mathcal{H}^{I}_{\text{CEN}}(\delta)) + (1-\mathbb{P}(\neg \mathcal{H}^{I}_{\text{CEN}}(\delta)))\frac{1}{\lambda^{\alpha}} \\
    &\leq \frac{1}{(1-\delta)^{\alpha}}d_{\mathit{eff}}\left[\psi_{\alpha}(\frac{T}{d_{\mathit{eff}}}+\frac{\lambda}{1-\delta}) - \psi_{\alpha}(\frac{\lambda}{1-\delta})\right] + \frac{24d_{\mathit{eff}}\log(T)+d}{\lambda^{\alpha}} \\&+ \frac{4}{\delta^{2}}d_{\mathit{eff}}\frac{1}{\lambda^{\alpha}T^{12\delta^{2}}}.
\end{align*}
In particular, for $\alpha =1$ and $\delta = \frac{1}{2}$, this involves:
\begin{align*}
    \max_{\pi \in \Pi} \mathbb{E}[\mathbb{V}_{1}(T,\pi)] \leq 2d_{\mathit{eff}}\log(\frac{T}{2\lambda}+1) + \frac{24d_{\mathit{eff}}\log(T)+d}{\lambda} + 16d_{\mathit{eff}}\frac{1}{\lambda T^{2}},
\end{align*}
and for $\alpha=\frac{1}{2}$ and $\delta = \frac{1}{2}$, this yields:
\begin{align*}
    \max_{\pi \in \Pi} \mathbb{E}[\mathbb{V}_{\frac{1}{2}}(T,\pi)] \leq\sqrt{2}d_{\mathit{eff}}\left[\sqrt{\frac{T}{d_{\mathit{eff}}}+2\lambda} - \sqrt{2\lambda}\right] + \frac{24d_{\mathit{eff}}\log(T)+d}{\sqrt{\lambda}} + 16d_{\mathit{eff}}\frac{1}{\sqrt{\lambda}T^{2}}.
\end{align*}
\end{proof}
\subsection{Proof of Thm. \ref{THM Finite arms}}

\THMFinitearms*

\begin{proof}
We first apply Lemma \ref{Potential Reduction Finite} to deduce:
\begin{align*}
     \mathbb{E}[R(T,\pi_{\text{UCB}})] &\leq 2 \sqrt{6\sigma^{2}\log(T)}\mathbb{E}[\mathbb{V}_{\frac{1}{2}}(T,\pi_{\text{UCB}})] + 2\lambda\|\theta^{\star}\|_{\infty}\mathbb{E}[\mathbb{V}_{1}(T,\pi_{\text{UCB}})] + \frac{2d\Delta_{max}}{T} \\
     &\leq 2 \sqrt{6\sigma^{2}\log(T)}\max_{\pi \in \Pi}\mathbb{E}[\mathbb{V}_{\frac{1}{2}}(T,\pi)] + 2\lambda\|\theta^{\star}\|_{\infty}\max_{\pi \in \Pi}\mathbb{E}[\mathbb{V}_{1}(T,\pi)] + \frac{2d\Delta_{max}}{T}.
\end{align*}
We then apply proposition \ref{Potential Control Finite}, with $\delta =1/2$ in order to deduce:
\begin{align*}
    \mathbb{E}[R(T,\pi_{\text{UCB}})] &\leq 2 \sqrt{6\sigma^{2}\log(T)}\Big(\sqrt{2}d_{\mathit{eff}}\Big[\sqrt{\frac{T}{d_{\mathit{eff}}}+2\lambda} - \sqrt{2\lambda}\Big] + \frac{24d_{\mathit{eff}}\log(T)+d}{\sqrt{\lambda}} + 16d_{\mathit{eff}}\frac{1}{\sqrt{\lambda}T^{2}}\Big) \\
    &+ 2\lambda\|\theta^{\star}\|_{\infty}\Big(2d_{\mathit{eff}}\log\Big(\frac{T}{2\lambda}+1\Big) + \frac{24d_{\mathit{eff}}\log(T)+d}{\lambda^{\alpha}} + 16d_{\mathit{eff}}\frac{1}{\lambda T^{2}}\Big) + \frac{2d\Delta_{max}}{T}.
\end{align*}
By taking $\lambda = o(\log(T))$ and considering only the leading order, we conclude that:
\begin{align*}
    \mathbb{E}[R(T,\pi_{\text{UCB}})] \leq \Tilde{\mathcal{O}}( \sigma\sqrt{d_{\mathit{eff}}T}).
\end{align*}
Note that our proof easily allows to get high-probability bounds on regret instead of bounds on its expected value.
\end{proof}

\begin{remark}\label{Hetero Finite 1}
We now extend Thm. \ref{THM Finite arms} to heteroskedastic MAB. In this model, the pseudo-variance of the sub-Gaussian noisy reward is arm-dependent and denoted $\sigma_{a}$. Moreover, the value of $\sigma_{a}$ is known to the designer of the algorithm, that is, it can be used as a parameter for the UCB algorithm. We first apply a slightly modified version of Lemma \ref{Potential Reduction Finite} to deduce:
\begin{align*}
     \mathbb{E}[R(T,\pi_{\text{UCB}})] &\leq 2 \sqrt{6\log(T)}\mathbb{E}[\Bar{\mathbb{V}}_{\frac{1}{2}}(T,\pi_{\text{UCB}})] + 2\lambda\|\theta^{\star}\|_{\infty}\mathbb{E}[\mathbb{V}_{1}(T,\pi_{\text{UCB}})] + \frac{2d\Delta_{max}}{T}, 
\end{align*}
where for $\alpha>0$ and $\pi \in \Pi$, we introduced the variance-based cumulative potential:
\begin{align*}
    \Bar{\mathbb{V}}_{\alpha}(T,\pi) = \sum_{t=1}^{T}(\frac{N_{a_{t}}(t-1)}{\sigma^{1/\alpha}_{a_{t}}}+\frac{\lambda}{\sigma^{1/\alpha}_{a_{t}}})^{-\alpha}.
\end{align*}
Thus, heteroskedasticity induces the mapping $\Breve{p_{a}}\equiv p_{a}/\sigma^{1/\alpha}_{a}$ and $\Breve{\lambda}_{a}\equiv \lambda/\sigma^{1/\alpha}_{a}$. Following the proof of Prop. \ref{Potential Control Finite}, we deduce for any $\alpha>0$ and time allocation $(\tau_{a}(T))_{a\in[d]}$:
\begin{align*}
     \mathbb{V}_{\alpha}(T,\pi|\mathcal{H}^{I}_{\text{CEN}}(\delta)) &\leq \frac{24 d_{\mathit{eff}}\log(T)+d}{\lambda^{\alpha}} + \frac{1}{(1-\delta)^{\alpha}} \sum_{a\in[d]}\frac{1}{\Breve{p_{a}}}[\psi_{\alpha}(\Breve{p_{a}}\tau_{a}(T)+\frac{\Breve{\lambda}_{a}}{1-\delta})-\psi_{\alpha}(\frac{\Breve{\lambda}_{a}}{1-\delta})].
\end{align*}
In order to apply Lemma \ref{Optimization Lemma Finite}, we introduce the notation:
\begin{align*}
    \Breve{d}_{\mathit{eff}} = \sum_{a\in[d]}\frac{\sigma^{1/\alpha}_{a}}{p_{a}}, \quad
    \Breve{\lambda}_{\mathit{eff}} = \frac{\lambda}{1-\delta} \frac{d_{\mathit{eff}}}{\Breve{d}_{\mathit{eff}}} \quad \text{and}\quad
    \Breve{\lambda}_{a}^{0} = \frac{\lambda\Breve{d}_{\mathit{eff}}}{1-\delta}(\frac{1}{\sigma^{1/\alpha}_{a}}- \frac{d_{\mathit{eff}}}{\Breve{d}_{\mathit{eff}}}).
\end{align*}
and we deduce that whenever $T\geq \max_{a} \Breve{\lambda}_{a}^{0}$, we have:
\begin{align*}
    \mathbb{V}_{\alpha}(T,\pi|\mathcal{H}^{I}_{\text{CEN}}(\delta)) &\leq \frac{24 d_{\mathit{eff}}\log(T)+d}{\lambda^{\alpha}} + \frac{1}{(1-\delta)^{\alpha}}\Big[\Breve{d}_{\mathit{eff}}\psi_{\alpha}(\frac{T}{\Breve{d}_{\mathit{eff}}}+\Breve{\lambda}_{\mathit{eff}})-\sum_{a\in[d]}\frac{\sigma^{1/\alpha}_{a}}{p_{a}}\psi_{\alpha}(\frac{\lambda}{(1-\delta)\sigma^{1/\alpha}_{a}})\Big].
\end{align*}
In particular, by considering the case $\alpha=\nicefrac{1}{2}$ and only the leading order, we deduce that:
\begin{align*}
    \mathbb{E}[R(T,\pi_{\text{UCB}})] \leq \Tilde{\mathcal{O}}\Big(\sqrt{\Breve{d}_{\mathit{eff}}T}\Big),
\end{align*}
where as affirmed $\Breve{d}_{\mathit{eff}}=\sum_{a\in[d]}\frac{\sigma_{a}^{2}}{p_{a}}$. 

\end{remark}

\subsection{Proof of Prop. \ref{Instance Dep Regret Finite}}

\InstanceDepRegretFinite*

\begin{proof} As in the proof of Lemma \ref{Potential Control Finite}, for a given round $t\in [T]$, we have under the event $\neg \mathcal{H}_{\text{UCB}}^{\lambda}$
\begin{align*}
    \Delta_{a} = \max_{\Tilde{a}\in\mathcal{A}}\theta^{\star}_{\Tilde{a}} - \theta^{\star}_{a} \leq 2\sqrt{6\sigma^{2} \frac{\log(T)}{N_{a_{t}}(t-1)+\lambda}} + 2\frac{\lambda \|\theta^{\star}\|_{\infty}}{\lambda  + N_{a_{t}}(t-1)}.
\end{align*}
It is as an inequality of the second degree and thus for any $t\in[T]$, $a\in [d]$:
\begin{align*}
    x_{1}\left(\sqrt{\frac{1}{\lambda+N_{a}(t)}}\right)^{2} + x_{2}\sqrt{\frac{1}{\lambda+N_{a}(t)}} - \Delta_{a}\geq 0,
\end{align*}
where $x_{1}=2\lambda \|\theta^{\star}\|_{\infty}$ and $x_{2}=2\sqrt{6\sigma^{2}\log(T)}$. Solving it yields:
\begin{align*}
    \sqrt{\frac{1}{\lambda+N_{a}(t)}} \geq \frac{1}{2x_{1}}(-x_{2}+\sqrt{x_{2}^{2}+4\Delta_{a}x_{1}}),
\end{align*}
or equivalently:
\begin{align*}
    N_{a}(T) &\leq \Big(\frac{4\lambda \|\theta^{\star}\|_{\infty}}{\sqrt{24\sigma^{2}\log(T)+8\Delta_{a}\lambda \|\theta^{\star}\|_{\infty}}-\sqrt{24\sigma^{2}\log(T)}}\Big)^{2}-\lambda \triangleq \Theta(T),
\end{align*}
where we used the notation $\Theta(T)$ to simplify the presentation.Therefore, under $\neg \mathcal{H}^{I}_{\text{CEN}}(\frac{1}{2})$, we have:
\begin{align*}
    \tau_{a}(t) \leq \max(T_{0}(a),\frac{2}{p_{a}}\Theta(T)).
\end{align*}
This yields a conditional regret of:
\begin{align*}
    R(T|&\neg(\mathcal{H}^{I}_{\text{CEN}}(\frac{1}{2})\cup \mathcal{H}^{\lambda}_{\text{UCB}})) \leq \sum_{a\in[d], a \neq a^{\star}}\Delta_{a}\tau_{a}(T) = \sum_{a\in[d], a \neq a^{\star}}\frac{2\Delta_{a}}{p_{a}}\max(12\log(T)+\frac{p_{a}}{2},\Theta(T)), 
\end{align*}
where $a^{\star}\triangleq \operatorname{argmax}_{\Tilde{a}\in\mathcal{A}}\theta^{\star}_{\Tilde{a}}$ and an expected regret of:
\begin{align*}
    \mathbb{E}[R(T,\pi_{\mathit{UCB}})] &\leq \sum_{a\in[d], a \neq a^{\star}}\frac{2\Delta_{a}}{p_{a}}\max(12\log(T)+\frac{p_{a}}{2},\Theta(T)) + \frac{d\Delta_{max}}{T} + \frac{16 d_{\mathit{eff}}\Delta_{max}}{T^{2}}.
\end{align*}
In particular, for the regularization $\lambda = o(\log(T))$, we have the asymptotic:
\begin{align*}
    \Theta(T) = \Big(\frac{4\lambda \|\theta^{\star}\|_{\infty}}{\sqrt{24\sigma^{2}\log(T)+8\Delta_{a}\lambda \|\theta^{\star}\|_{\infty}}-\sqrt{24\sigma^{2}\log(T)}}\Big)^{2} = \frac{24\sigma^{2}\log(T)}{\Delta_{a}^{2}} + \frac{8\lambda \|\theta^{\star}\|_{\infty}}{2\Delta_{a}} + o(1).
\end{align*}
And thus, we conclude that:
\begin{align*}
    \mathbb{E}[R(T,\pi_{\mathit{UCB}})] &\leq \mathcal{O}\Big(\log(T)\sum_{a\in[d], a \neq a^{\star}}\frac{1}{p_{a}}\max(\frac{\sigma^{2}}{\Delta_{a}},\Delta_{a})\Big).
\end{align*}
Again, note that our proof easily allows to get high-probability bounds on regret instead of bounds on its expected value.
\end{proof}
\begin{remark}\label{Hetero Finite 2}
As in the instance-independent case, previous reasoning immediately extends to a-priori known heteroskedasticity and yields the upper bound:
\begin{align*}
    \mathbb{E}[R(T,\pi_{\mathit{UCB}})] &\leq \mathcal{O}\Big(\log(T)\sum_{a\in[d], a \neq a^{\star}}\frac{1}{p_{a}}\max(\frac{\sigma_{a}^{2}}{\Delta_{a}},\Delta_{a})\Big).
\end{align*}
\end{remark}


Next, we provide additional insights to the main result of this section. In particular, we seek to gain intuition about how the policies that are adaptive to the realization of censorship process would perform in expectation against a class of non-adaptive (i.e. offline ) policies. In order to precisely derive asymptotic behavior of such policies, we introduce and study a continuous counterpart of the discrete original policy maximization problem $\max_{\pi \in \Pi}\mathbb{E}[\mathbb{V}_{\alpha}(T,\pi)]$.

Lemma \ref{asympt_off} provides the basis for continuous approach in the case of offline policies by leveraging concentration inequalities for inverse Binomial distribution. We then extend this approach in the proof of Prop. \ref{Monitoring AG}. This extension enables us to provide an exact expression for the asymptotic gain of a policy class that monitors the censorship at a single point in time, as well as estimate the gain from fully adaptive policies.

\subsection{Proof of Lemma \ref{asympt_off}}

\asymptoff*
\begin{proof}
Given the offline nature of the policy class, we have:
\begin{align*}
    \max_{\pi \in \Pi_{\text{off}}} \mathbb{E}[\mathbb{V}_{\alpha}(T,\pi)] = \max_{(\tau_{a})_{a\in[d]}}\sum_{a\in[d]}\mathbb{E}\Big[\sum_{n=1}^{\tau_{a}}\frac{1}{(X^{a}_{n-1}+\lambda)^{\alpha}})\Big] = \max_{(\tau_{a})_{a\in[d]}}\sum_{a\in[d]}\sum_{n=1}^{\tau_{a}}\mathbb{E}\Big[\frac{1}{(X^{a}_{n-1}+\lambda)^{\alpha}})\Big]
\end{align*}
where we have re-indexed $(N_{a}(t))$ by actions, where $(\tau_{a})_{a\in[d]}$ is a time allocation such that $\sum_{a\in[d]}\tau_{a}=T$ and where for a given action $a$, $(X^{a}_{n})_{n\leq \tau_{a}}$ are dependent random variables verifying $X_{n+1} = X^{a}_{n}+\mathcal{B}(p_{a})$ and $X^{a}_{n}\sim \mathcal{B}(n,p_{a})$. 

To lower bound this quantity, we fix a time allocation $(\tau_{a})_{a\in[d]}$ and use the fact that $x\mapsto x^{-\alpha}$ is convex with Jensen's inequality to deduce:
\begin{align*}
    \sum_{n=1}^{\tau_{a}}\mathbb{E}\Big[\frac{1}{(X^{a}_{n-1}+\lambda)^{\alpha}})\Big] &\geq \sum_{n=1}^{\tau_{a}}\frac{1}{(\mathbb{E}[X^{a}_{n-1}]+\lambda)^{\alpha}} = \sum_{n=1}^{\tau_{a}}\frac{1}{(p_{a}(n-1)+\lambda)^{\alpha}} = \frac{1}{\lambda^{\alpha}} + \sum_{n=1}^{\tau_{a}-1}\frac{1}{(p_{a}n+\lambda)^{\alpha}} 
\end{align*}
We then leverage the fact that $(px+\lambda)^{-\alpha}\geq \int_{x-1}^{x}(pu+\lambda)^{-\alpha}\partial u$ to deduce:
\begin{align*}
    \sum_{n=1}^{\tau_{a}}\mathbb{E}\Big[\frac{1}{(X^{a}_{n-1}+\lambda)^{\alpha}})\Big] \geq \frac{1}{\lambda^{\alpha}} + \sum_{a\in[d]}\int_{0}^{\tau_{a}-1}\frac{1}{(p_{a}x+\lambda)^{\alpha}}\partial x = \frac{1}{\lambda^{\alpha}} + \frac{1}{p_{a}}\Big[\psi_{\alpha}(p_{a}(\tau_{a}-1)+\lambda)-\psi_{\alpha}(\lambda)\Big]
\end{align*}
and therefore, for any time allocation $(\tau_{a})_{a\in[d]}$, we have:
\begin{align*}
    \max_{\pi \in \Pi_{\text{off}}} \mathbb{E}[\mathbb{V}_{\alpha}(T,\pi)] \geq \frac{d}{\lambda^{\alpha}} + \sum_{a\in[d]}\frac{1}{p_{a}}\Big[\psi_{\alpha}(p_{a}(\tau_{a}-1)+\lambda)-\psi_{\alpha}(\lambda)\Big]
\end{align*}
Although the maximum over time allocation is given by Lemma \ref{Optimization Lemma Finite}, we simply use the allocation $(\frac{T}{p_{a}d_{\mathit{eff}}})_{a\in[d]}$ to deduce:
\begin{align*}
    \max_{\pi \in \Pi_{\text{off}}} \mathbb{E}[\mathbb{V}_{\alpha}(T,\pi)] \geq  \sum_{a\in[d]}\frac{1}{p_{a}}\psi_{\alpha}(\frac{T}{d_{\mathit{eff}}}+\lambda-\frac{1}{p_{a}})+\frac{d}{\lambda^{\alpha}} -\sum_{a\in[d]}\frac{1}{p_{a}}\psi_{\alpha}(\lambda)
\end{align*}
By making the distinction between $\alpha=1$ and $\alpha<1$ to obtain the explicit expression of $\psi_{\alpha}$, we then show that the LHS is equivalent to $d_{\mathit{eff}}\psi_{\alpha}(\frac{T}{d_{\mathit{eff}}}+\lambda)$. The proof of the upper bound is more involved. In proving it, let's first assume that:
\begin{claim}\label{concentration Bin}
For all $a \in [d]$, $n\geq 1$, there exists a constant $C_{2}^{a}$ such that:
\begin{align*}
    \mathbb{E}\Big[\frac{1}{(X^{a}_{n}+\lambda)^{\alpha}})\Big] \leq (1+\frac{C_{2}^{a}}{(n p_{a})^{1/4}})\frac{1}{(p_{a}n+\lambda)^{\alpha}}
\end{align*}
\end{claim}
Given this result, we deduce:
\begin{align*}
    \sum_{n=1}^{\tau_{a}-1}\mathbb{E}\Big[\frac{1}{(X^{a}_{n}+\lambda)^{\alpha}})\Big] \leq \sum_{n=1}^{\tau_{a}-1} (1+\frac{C_{2}^{a}}{(np_{a})^{1/4}})\frac{1}{(p_{a}n+\lambda)^{\alpha}}.
\end{align*}
Therefore, we have:
\begin{align*}
    \max_{(\tau_{a})_{a\in[d]}}\sum_{a\in[d]}\sum_{n=1}^{\tau_{a}}\mathbb{E}\Big[\frac{1}{(X^{a}_{n-1}+\lambda)^{\alpha}})\Big] &\leq \frac{d}{\lambda^{\alpha}}+ \max_{(\tau_{a})_{a\in[d]}}\sum_{a\in[d]}\sum_{n=1}^{\tau_{a}}\mathbb{E}\Big[\frac{1}{(X^{a}_{n}+\lambda)^{\alpha}})\Big] \\
    &\leq \sum_{a\in[d]}\frac{1}{\lambda^{\alpha}} + \max_{(\tau_{a})_{a\in[d]}}\sum_{a\in[d]}\sum_{n=1}^{\tau_{a}}\frac{1}{(p_{a}n+\lambda)^{\alpha}} \\&\quad \quad \quad+ (\max_{a\in[d]}C_{2}^{a})\max_{(\tau_{a})_{a\in[d]}}\sum_{a\in[d]}\sum_{n=1}^{\tau_{a}}\frac{1}{(p_{a}n)^{1/4}}\frac{1}{(p_{a}n+\lambda)^{\alpha}}.
\end{align*}
We first consider the second maximization problem and note that:
\begin{align*}
    \max_{(\tau_{a})_{a\in[d]}}\sum_{a\in[d]}\sum_{n=1}^{\tau_{a}}\frac{1}{(p_{a}n)^{1/4}}\frac{1}{(p_{a}n+\lambda)^{\alpha}} &\leq \lambda^{1/4} \max_{(\tau_{a})_{a\in[d]}}\sum_{a\in[d]}\sum_{n=1}^{\tau_{a}}\frac{1}{(p_{a}n+\lambda)^{\alpha+1/4}} \\
    &= \mathcal{O}(d_{\mathit{eff}}\psi_{\alpha+\frac{1}{4}}(\frac{T}{d_{\mathit{eff}}}+\lambda)) \\
    &= o(d_{\mathit{eff}}\psi_{\alpha}(\frac{T}{d_{\mathit{eff}}}+\lambda))
\end{align*}
where we used an integral comparison and Lemma \ref{Optimization Lemma Finite} to deduce the $\mathcal{O}$ scaling and the fact that $\alpha \in ]0,1]$ the deduce the $o$ scaling.
Similarly, we know that the first maximization problem scales as $d_{\mathit{eff}}\psi_{\alpha}(\frac{T}{d_{\mathit{eff}}}+\lambda)$ through another integral comparison and use of Lemma \ref{Optimization Lemma Finite}. Given this, we conclude that the upper bound is equivalent to $d_{\mathit{eff}}\psi_{\alpha}(\frac{T}{d_{\mathit{eff}}}+\lambda)$.
Thanks to those two results, we finally affirm that:
\begin{align*}
    \displaystyle \max_{\pi \in \Pi_{\text{off}}} \mathbb{E}[\mathbb{V}_{\alpha}(T,\pi)] \sim d_{\mathit{eff}}\psi_{\alpha}(\frac{T}{d_{\mathit{eff}}}+\lambda).
\end{align*}
The last step needed is to prove Claim. \ref{concentration Bin}. In doing so, we extend Lemma $5.3$ of \cite{thesesarlot} to more general inverse power function (i.e. $\alpha \neq 1$) with regularization $\lambda>0$. We first introduce a tuning parameter $u\geq 0$ and write:
\begin{align*}
    (\mathbb{E}[X^{a}_{n}]+\lambda)^{\alpha}\mathbb{E}\Big[\frac{1}{(X^{a}_{n}+\lambda)^{\alpha}}\Big] &= (n p_{a}+\lambda)^{\alpha} \mathbb{E}\Big[\frac{1}{(X^{a}_{n}+\lambda)^{\alpha}}\mathbf{1}\{X^{a}_{n}\leq u\mathbb{E}[X^{a}_{n}]\}\Big] \\&\quad \quad \quad+ \mathbb{E}\Big[\frac{1}{(X^{a}_{n}+\lambda)^{\alpha}}\mathbf{1}\{X^{a}_{n}> u \mathbb{E}[X^{a}_{n}]\}\Big] \\
    &\leq \frac{(n p_{a}+\lambda)^{\alpha}}{\lambda^{\alpha}}\mathbb{P}(X^{a}_{n}\leq u\cdot n p_{a}) + (\frac{n p_{a}+\lambda}{u\cdot n p_{a}+\lambda})^{\alpha}.
\end{align*}
Using a Berstein inequality for Binomiale variable, we have for all $\theta>0$ and $n\in \mathbb{N}$:
\begin{align*}
    \mathbb{P}\left(X_{n}^{a} \leq\left(1-\sqrt{2 \theta}-\frac{\theta}{3}\right) n p_{a}\right) \leq e^{-\theta n p_{a}}.
\end{align*}
Thus, for all $0<\theta \leq \frac{3(\sqrt{5}-\sqrt{3})^{2}}{2}$, by setting $u \equiv (1-\sqrt{2 \theta}-\frac{\theta}{3})\geq 0$, we obtain that:
\begin{align*}
    (\mathbb{E}[X^{a}_{n}]+\lambda)^{\alpha}\mathbb{E}\Big[\frac{1}{(X^{a}_{n}+\lambda)^{\alpha}}\Big] \leq  \frac{(n p_{a}+\lambda)^{\alpha}}{\lambda^{\alpha}}e^{-\theta n p_{a}} + \Big(\frac{n p_{a}+\lambda}{(1-\sqrt{2 \theta}-\frac{\theta}{3})n p_{a}+\lambda}\Big)^{\alpha}.
\end{align*}
By taking $\theta = A\frac{\log(n p_{a}+\lambda)}{n p_{a}}$ for another tunable parameter $A$ and for $n$ large enough to ensure $A\frac{\log(n p_{a}+\lambda)}{n p_{a}}\leq \frac{3(\sqrt{5}-\sqrt{3})^{2}}{2}$, this yields:
\begin{align*}
    (\mathbb{E}[X^{a}_{n}]+\lambda)^{\alpha}\mathbb{E}\Big[\frac{1}{(X^{a}_{n}+\lambda)^{\alpha}}\Big] &\leq \frac{(n p_{a}+\lambda)^{\alpha}}{\lambda^{\alpha}(np_{a}+\lambda)^{A}} + \Big(\frac{n p_{a}+\lambda}{(1-\sqrt{2 A\frac{\log(np_{a}+\lambda)}{n p_{a}}}-A\frac{\log(n p_{a}+\lambda)}{3 n p_{a}})n p_{a}+\lambda}\Big)^{\alpha} \\
    &= \frac{(n p_{a}+\lambda)^{\alpha}}{\lambda^{\alpha}(np_{a}+\lambda)^{A}} + \Big(\frac{n p_{a}+\lambda}{n p_{a}+\lambda -\sqrt{2 A n p_{a} \log(n p_{a}+\lambda)}-\frac{A\log(n p_{a}+\lambda)}{3}}\Big)^{\alpha} \\
\end{align*}
For $n$ sufficiently large to ensure $\frac{3\sqrt{2 A n p_{a} \log(np_{a})}+A\log(np_{a})}{3(np_{a}+\lambda)}\leq 1/2$ and given that $\alpha \in ]0,1]$, we then have:
\begin{align*}
    \Big(\frac{1}{1 -\frac{2\sqrt{2 A n p_{a} \log(np_{a}+\lambda)}+A\log(np_{a}+\lambda)}{3(np_{a}+\lambda)}}\Big)^{\alpha} 
    &\leq \Big(1 + 2\frac{3\sqrt{2 A n p_{a} \log(np_{a}+\lambda)}+A\log(np_{a}+\lambda)}{3(np_{a}+\lambda)} \Big)^{\alpha}\\
    &\leq 1 + 2\alpha \frac{3\sqrt{2 A n p_{a} \log(np_{a}+\lambda)}+A\log(np_{a}+\lambda)}{3(np_{a}+\lambda)}
\end{align*}
To conclude, we take take $A\equiv(\alpha+1)$ and this ensure that for $n$ sufficiently large, we obtain:
\begin{align*}
    (\mathbb{E}[X^{a}_{n}]+\lambda)^{\alpha}\mathbb{E}\Big[\frac{1}{(X^{a}_{n}+\lambda)^{\alpha}}\Big] &\leq 1 + 2\alpha \frac{3\sqrt{2 (\alpha+1) n p_{a} \log(np_{a}+\lambda)}+(\alpha+1)\log(np_{a}+\lambda)}{3(np_{a}+\lambda)} \\
    &\quad \quad \quad + \frac{1}{\lambda^{\alpha}(np_{a}+\lambda)}. 
\end{align*}
The leading order of this quantity is $\sqrt{\frac{log(n)}{n}}=o(n^{-1/4})$ and therefore, we conclude that there exists a constant $C_{2}^{a}$, depending on $\lambda, \alpha$ and $p_{a}$ such that for all $n\geq 1$:
\begin{align*}
    (\mathbb{E}[X^{a}_{n}]+\lambda)^{\alpha}\mathbb{E}\Big[\frac{1}{(X^{a}_{n}+\lambda)^{\alpha}}\Big] \leq 1 + \frac{C_{2}^{a}}{(n p_{a})^{1/4}},
\end{align*}
where $C_{2}^{a}$ is artificially increased to remove the two lower bounds conditions on $n$.
\end{proof}


\subsection{Proof of Prop. \ref{Monitoring AG}}

\MonitoringAG*

Thus, we find that the power of a single monitoring is sufficient to ensure almost the same gain as adaptivity i.e. constant monitoring. The linear dependency in $T_{0}$ (due to the linear increase of variance in Binomial models) is also surprising. In non-asymptotic regime, it is still true but for $\beta$ verifying $0<\beta_{-}\leq \beta\leq \beta_{+}<1$ for given $(\beta_{-},\beta_{+})$. We also observe a more general concave property of the single monitoring gain seen as a function of $T_{0}$, with limits equals to $0$ on the borders on the interval. We conjecture that this concavity is likely to turn in a submodular dependency for several monitoring shots.


\begin{proof}

\textbf{Single Monitoring:} We first prove a slightly extended version of (\ref{One-Shot}) by considering a monitoring at time $T_{0}$ and we recover the results of Prop. \ref{Monitoring AG} by setting $T_{0}\equiv \beta T$, for a given $\beta \in ]0,1[$. 
For the first step of the proof, we consider the continuous approximation of $\displaystyle \max_{\pi \in \Pi_{\text{single}}(T_{0})}\mathbb{E}[\mathbb{V}_{\alpha}(T,\pi)]$ given by the optimization problem over continuous variables:
\begin{align*}
 \max_{\tau_{a}(T_{0}),\tau_{a}(T)}& \mathbb{E}\Big[\sum_{a\in[d]}\frac{1}{p_{a}}[\psi_{\alpha}(N_{a}(T))-\psi_{\alpha}(N_{a}(T_{0}))] + \sum_{a\in[d]}\frac{1}{p_{a}}[\psi_{\alpha}(N_{a}(T_{0}))-\psi_{\alpha}(\lambda)]\Big]\\ \tag{$\mathcal{S}\mathcal{M}$} \label{SM}
\textrm{s.t.} & \sum_{a\in[d]}\tau_{a}(T_{0}) = T_{0},\\
 & \sum_{a\in[d]}\tau_{a}(T) = T,    \\
 & \forall a \in [d],\quad \tau_{a}(T)\geq \tau_{a}(T_{0}).   
\end{align*}
In \ref{SM}, the single monitoring $\max$ player initially commits to an allocation of the $T_{0}$ first rounds through the policy $(\tau_{a}(T_{0}))_{a\in[d]}$, with resulting gain expressed as the second term of the maximization problem. The player then observes the realization $N_{a}(T_{0})\sim \mathcal{B}(\tau_{a}(T_{0}),p_{a})$ and allocates the rest of the $T-T_{0}$ budget through the allocation $(\tau_{a}(T))_{a\in[d]}$, with resulting gain expressed as the first term of the maximization problem. . Therefore, the single monitoring gain assesses the value of observing the deviation of $N_{a}(T_{0})$ from its expectation $\tau_{a}(T_{0})p_{a}$. In an analogous way, we then introduce the continuous approximation of $\displaystyle \max_{\pi \in \Pi_{\text{off}}}\mathbb{E}[\mathbb{V}_{\alpha}(T,\pi)]$ given by:
\begin{align*}
 \max_{\tau_{a}(T)}& \quad \mathbb{E}\Big[\sum_{a\in[d]}\frac{1}{p_{a}}[\psi_{\alpha}(N_{a}(T))-\psi_{\alpha}(\lambda)]\Big]\\ \tag{$\mathcal{O}\mathcal{F}\mathcal{F}$} \label{OFF}
\textrm{s.t.} & \sum_{a\in[d]}\tau_{a}(T) = T.  
\end{align*}
In \ref{OFF}, $(N_{a}(T_{0}))_{a\in[d]}$ is not observed and thus can not be leveraged by the offline player to adapt the second part of the allocation. On what follows, we use $\mathbf{E}[\mathbf{N}]$ and $\mathbf{V}[\mathbf{N}]$ to denote respectively the mean and variance of $\mathbf{N}$, the empirical discrete distribution over $(N_{a}(T_{0}))_{a\in[d]}$ with associated weights $(1/p_{a}d_{\mathit{eff}})_{a\in[d]}$. Given a realization of $(N_{a}(T_{0}))_{a\in[d]}$, we use Lemma \ref{Optimization Lemma Finite} to deduce the optimal choice of $(\tau_{a}(T))_{a\in[d]}$ in \ref{SM} and resulting expected conditional gain:
\begin{align*}
    \sum_{a\in[d]}\underbrace{\frac{1}{p_{a}}\Big[ \psi_{\alpha}(\frac{T-T_{0}}{d_{\mathit{eff}}}+\mathbf{E}[\mathbf{N}]) -\psi_{\alpha}(N_{a}(T_{0}))\Big]}_{\textit{Gain between $T_{0}$ and $T$ for arm $a$}} + \sum_{a\in[d]}\underbrace{\frac{1}{p_{a}}\Big[ \psi_{\alpha}(N_{a}(T_{0})) -\psi_{\alpha}(\lambda)\Big]}_{\textit{Gain between $0$ and $T_{0}$ for arm $a$}},
\end{align*}
where the formula is valid under the assumption $\forall a \in [d]$, $T-T_{0} \geq d_{\mathit{eff}}(N_{a}(T_{0})-\mathbf{E}[\mathbf{N}])$. Such assumption encompass the fact that the remaining budget $T-T_{0}$ should be sufficient to correct the deviation observed. Logically, we know that on expectation $\mathbb{E}[N_{a}(T_{0})-\mathbf{E}[\mathbf{N}]] = 0$, that is no systematic deviation is expected. For for all realization of randomness, the following deterministic crude upper bound hold:
\begin{align*}
    d_{\mathit{eff}}(N_{a}(T_{0})-\mathbf{E}[\mathbf{N}]) \leq d_{\mathit{eff}}(1-\frac{1}{d_{\mathit{eff}}p_{a}})\frac{T_{0}}{p_{a}d_{\mathit{eff}}}
\end{align*}
this in turn imposes:
\begin{align*}
    \frac{T_{0}}{T} \leq \min_{a\in[d]}\frac{1}{1 + \frac{d_{\mathit{eff}}-\frac{1}{p_{a}}}{p_{a}d_{\mathit{eff}}}}.
\end{align*}
For instance, in the uniform censorship model, this yields condition $\frac{T_{0}}{T} \leq \frac{dp}{dp + d-1}$. Nevertheless, this is overly conservative and we can get considerably stronger results by considering high-probability concentration results on $N_{a}(T_{0})$.  Indeed, thanks to Chernoff Bounds for Binomial distribution, we have for $\delta \equiv T_{0}^{-1/4}$, that with probability at least $1-2d\exp\{-\delta^{2}T_{0}/3d_{\mathit{eff}}\}$, for all $a$, $(1-\delta)T_{0}/d_{\mathit{eff}} \leq N_{a}(T_{0})\leq (1+\delta)T_{0}/d_{\mathit{eff}}$. In particular, this yields $(1-\delta)T_{0}/d_{\mathit{eff}}\leq \mathbf{E}[\mathbf{N}]\leq (1+\delta)T_{0}/d_{\mathit{eff}}$. Under this event, we have $d_{\mathit{eff}}(N_{a}(T_{0})-\mathbf{E}[\mathbf{N}])\leq 2\delta T_{0}$, which imposes $T\geq (1+2\delta)T_{0} = T_{0} + 2T_{0}^{3/4}$. In particular, for $T_{0}\equiv \beta T$, where $\beta \in ]0,1[$, such condition will always be verified for $T$ large enough. 

On the other hand, still using Lemma \ref{Optimization Lemma Finite}, we write the conditional expected gain of the offline policy on this same realization of $(N_{a}(T))_{a\in[d]}$ for \ref{OFF} as: 
\begin{align*}
    \sum_{a\in[d]}\underbrace{\frac{1}{p_{a}}\Big[ \psi_{\alpha}(\frac{T-T_{0}}{d_{\mathit{eff}}}+N_{a}(T_{0})) -\psi_{\alpha}(N_{a}(T_{0}))\Big]}_{\textit{Gain between $T_{0}$ and $T$ for arm $a$}} 
    + \sum_{a\in[d]}\underbrace{\frac{1}{p_{a}}\Big[ \psi_{\alpha}(N_{a}(T_{0})) -\psi_{\alpha}(\lambda)\Big]}_{\textit{Gain between $0$ and $T_{0}$ for arm $a$}}.
\end{align*}
where $\psi_{\alpha}(N_{a}(T_{0}))$ is artificially introduced. The difference between the two comes from the possibility for the monitoring policy to homogenize the realized $N_{a}(T_{0})$ into a uniform $\mathbf{E}[\mathbf{N}]$. The random difference $\mathcal{G}_{\mathit{single}}(T_{0})$, seen as a function of the realization of $(N_{a}(T_{0}))$ is then equal to:
\begin{align*}
    \mathcal{G}_{\mathit{single}}(T_{0}) &\triangleq \sum_{a\in[d]}\frac{1}{p_{a}}[\psi_{\alpha}(\frac{T-T_{0}}{d_{\mathit{eff}}}+\mathbf{E}[\mathbf{N}])-\psi_{\alpha}(\frac{T-T_{0}}{d_{\mathit{eff}}}+N_{a}(T_{0}))] \\
    &= d_{\mathit{eff}}\Big[\Bar{\psi}_{\alpha}(\mathbf{E}[\mathbf{N}])-\mathbf{E}[\Bar{\psi}_{\alpha}(\mathbf{N})]\Big],
\end{align*}
which is exactly the Jensen's gap of the concave function $\Bar{\psi}_{\alpha}: x \mapsto \psi_{a}(\frac{T-T_{0}}{d_{\mathit{eff}}}+x)$. The main insight is that this gap is then asymptotically equivalent to:
\begin{align*}
    \Bar{\psi}_{\alpha}(\mathbf{E}[\mathbf{N}])-\mathbf{E}[\Bar{\psi}_{\alpha}(\mathbf{N})]\sim -\frac{\Bar{\psi}^{(2)}_{\alpha}(\mathbf{E}[\mathbf{N}])}{2}\mathbf{V}[\mathbf{N}],
\end{align*}
where the RHS is positive, given that $\Bar{\psi}^{(2)}_{\alpha}(\mathbf{E}[\mathbf{N}])$ is negative. To show this, we use the original proof of Jensen's inequality and introduce the interval $I\triangleq [\min_{a} N_{a}(T_{0}),\max_{a}N_{a}(T_{0})]$ to leverage the mean value theorem. This then yields:
\begin{align*}
     \frac{\min_{y\in I}-\Bar{\psi}^{(2)}_{\alpha}(y)}{-\Bar{\psi}^{(2)}_{\alpha}(\mathbf{E}[\mathbf{N}])}\leq 2\frac{\Bar{\psi}_{\alpha}(\mathbf{E}[\mathbf{N}])-\mathbf{E}[\Bar{\psi}_{\alpha}(\mathbf{N})]}{-\Bar{\psi}^{(2)}_{\alpha}(\mathbf{E}[\mathbf{N}])\mathbf{V}[\mathbf{N}]} \leq \frac{\max_{y\in I}-\Bar{\psi}^{(2)}_{\alpha}(y)}{-\Bar{\psi}^{(2)}_{\alpha}(\mathbf{E}[\mathbf{N}])}
\end{align*}
Whenever $T_{0}$ is a constant independent of $T$, for $T\rightarrow+\infty$ by explicitly writing the definition of the upper and lower bounds, we have almost surely:
\begin{align*}
    \frac{\min_{y\in I}-\Bar{\psi}^{(2)}_{\alpha}(y)}{-\Bar{\psi}^{(2)}_{\alpha}(\mathbf{E}[\mathbf{N}])} \rightarrow 1 \quad \text{and} \quad 
    \frac{\max_{y\in I}-\Bar{\psi}^{(2)}_{\alpha}(y)}{-\Bar{\psi}^{(2)}_{\alpha}(\mathbf{E}[\mathbf{N}])} \rightarrow 1 
\end{align*}
Difficulties arises when $T_{0}$ is a function of $T$, as in the statement of the result where $T_{0}\equiv\beta T$. By considering the same concentration event as the one introduced above, we have:
\begin{align*}
    \frac{\min_{y\in I}-\psi^{(2)}_{\alpha}(y)}{-\Bar{\psi}^{(2)}_{\alpha}(\mathbf{E}[\mathbf{N}])} &= \min_{y\in I} \Big(\frac{T-T_{0}+d_{\mathit{eff}}\mathbf{E}[\mathbf{N}]}{T-T_{0}+d_{\mathit{eff}}y}\Big)^{^{1+\alpha}} \geq \Big(\frac{T-T_{0}+(1-\delta)T_{0}}{T-T_{0}+(1+\delta)T_{0}}\Big)^{^{1+\alpha}} \\&\leq \Big(\frac{T-\delta T_{0}}{T+\delta T_{0}}\Big)^{^{1+\alpha}} =  \Big(\frac{T-T_{0}^{3/4}}{T+T_{0}^{3/4}}\Big)^{^{1+\alpha}} \rightarrow 1.
\end{align*}
and similarly:
\begin{align*}
    \frac{\max_{y\in I}-\psi^{(2)}_{\alpha}(y)}{-\Bar{\psi}^{(2)}_{\alpha}(\mathbf{E}[\mathbf{N}])} &= \max_{y\in I} \Big(\frac{T-T_{0}+d_{\mathit{eff}}\mathbf{E}[\mathbf{N}]}{T-T_{0}+d_{\mathit{eff}}y}\Big)^{^{1+\alpha}} \leq \Big(\frac{T-T_{0}+(1+\delta)T_{0}}{T-T_{0}+(1-\delta)T_{0}}\Big)^{^{1+\alpha}} \\
    &\leq \Big(\frac{T+\delta T_{0}}{T-\delta T_{0}}\Big)^{^{1+\alpha}} =\Big(\frac{T+ T_{0}^{3/4}}{T- T_{0}^{3/4}}\Big)^{^{1+\alpha}} \rightarrow 1.
\end{align*}
Thus, thanks to the exponential concentration, we conclude that:
\begin{align*}
    \mathbb{E}[\mathcal{G}_{\mathit{single}}(T_{0})] = \frac{d_{\mathit{eff}}}{2} \mathbb{E}[\Bar{\psi}_{\alpha}(\mathbf{E}[\mathbf{N}])-\mathbf{E}[\Bar{\psi}_{\alpha}(\mathbf{N})]]\sim -\frac{d_{\mathit{eff}}}{2}\mathbb{E}[\Bar{\psi}^{(2)}_{\alpha}(\mathbf{E}[\mathbf{N}])\mathbf{V}[\mathbf{N}]].
\end{align*}
Next, we affirm that:
\begin{align*}
    \mathbb{E}[\Bar{\psi}^{(2)}_{\alpha}(\mathbf{E}[\mathbf{N}])\mathbf{V}[\mathbf{N}]] &\overset{a)}{\sim} \mathbb{E}\Big[\Bar{\psi}^{(2)}_{\alpha}(\mathbf{E}[\mathbf{N}])\Big]\mathbb{E}\Big[\mathbf{V}[\mathbf{N}]\Big] \overset{b)}{\sim} \Bar{\psi}^{(2)}_{\alpha}(\mathbb{E}[\mathbf{E}[\mathbf{N}]])\mathbb{E}\Big[\mathbf{V}[\mathbf{N}]\Big],
\end{align*}
where $a)$ leverages the previous bounds and where we use for $b)$ similar concentration results on inverse of Binomial as done for the proof of Lemma \ref{asympt_off}..
We then use the fact that $\mathbb{E}[\mathbf{E}[\mathbf{N}]]= T_{0}/d_{\mathit{eff}}$ to conclude:
\begin{align*}
    \mathbb{E}[\mathcal{G}_{\mathit{single}}(T_{0})] \sim -\frac{d_{\mathit{eff}}}{2}\psi^{(2)}_{\alpha}(\frac{T}{d_{\mathit{eff}}})\mathbb{E}\Big[\mathbf{V}[\mathbf{N}]\Big].\tag{$\mathcal{V}$}\label{OS_Gain}
\end{align*}
We consider this result to be one of the main insight for adaptivity, as it involves that at first order, the gain grows linearly in the expected value of the empirical variance of the arm allocation process. In opposition to the single monitoring policy, the adaptive policy continuously exploits such variance. Yet, in doing so, it creates a second order induced variance but we then show that this phenomena is negligible at first order. To reach a result with explicit dependency on the censorship probability $(p_{a})_{a\in[d]}$, we note that:
\begin{align*}
    \mathbb{E}[\mathbf{V}[\mathbf{N}]] = \frac{T_{0}}{d_{\mathit{eff}}^{3}}\sum_{a\in [d]}\frac{1}{p_{a}}\Big[\sum_{b\neq a}\frac{1-p_{b}}{p_{b}}\Big], 
\end{align*}
and therefore:
\begin{align*}
      \mathbb{E}\Big[\mathcal{G}_{\mathit{single}}(T_{0})\Big]  &\sim \frac{\alpha}{2d_{\mathit{eff}}^{2}}(\frac{d_{\mathit{eff}}}{T})^{1+\alpha} T_{0}\sum_{a\in [d]}\frac{1}{p_{a}}\Big[\sum_{b\neq a}\frac{1-p_{b}}{p_{b}}\Big] 
    = \gamma_{\alpha}(\mathbf{p})\frac{T_{0}}{T^{1+\alpha}}.
\end{align*}
In particular, for $T_{0}=\beta T$, this yields $\mathbb{E}\Big[\mathcal{G}_{\mathit{single}}(\beta T)\Big] = \gamma_{\alpha}(\mathbf{p})\frac{\beta}{T^{\alpha}} + o(\frac{1}{T^{\alpha}})$.

The second step closely mirrors the proof of Lemma \ref{asympt_off} and consists in justifying the use of the continuous approximation for the two optimization problems (\ref{OFF}) and (\ref{SM}). As in Lemma \ref{asympt_off}, we show that the difference between the continuous and discrete optimization results at most in a second order gain of $o(\frac{1}{T^{\alpha}})$, even when maximized as a decoupled quantity. By combining those two results, we finally deduce as announced:
\begin{align*}
    \max_{\pi \in \Pi_{\text{single}}(\beta T)} \mathbb{E}[\mathbb{V}_{\alpha}(T,\pi)] - \max_{\pi \in \Pi_{\text{off}}} \mathbb{E}[\mathbb{V}_{\alpha}(T,\pi)] \sim \mathbb{E}\Big[\mathcal{G}_{\mathit{single}}(\beta T)\Big] =   \gamma_{\alpha}(\mathbf{p})\frac{\beta}{T^{\alpha}} + o(\frac{1}{T^{\alpha}})
\end{align*}

\textbf{Complete Adaptivity}

We next tackle the proof of (\ref{Constant}), where the main idea is to show that a formula analogous to (\ref{OS_Gain}) holds for the variance of a suited random process. First, using the same proof technique as in Sec. 4 of \cite{decayingB} thanks to the decaying property of the reward in function of the number of realization, we show that the optimal adaptive policy is the greedy policy, that is the policy that picks at time $t$ the action:
\begin{align*}
    a_{t} \triangleq \operatorname{argmax}_{a\in\mathcal{A}_{t}}(N_{a}(t-1)+\lambda)^{-\alpha},
\end{align*}
with arbitrary but consistent tie-breaking. In particular, this ensures that for all actions $a,b$ and time $t$, we have $|N_{a}(t)-N_{b}(t)|\leq 1$. We then introduce the offline and adaptive allocations:
\begin{align*}
    \tau_{a}^{off}(T) &\triangleq \frac{T}{p_{a}d_{eff}} \quad \text{and} \quad
    \tau_{a}^{on}(T)  \triangleq \sum_{i=1}^{N_{a}(T)} \frac{1}{p_{a}}+\xi^{a}_{i}= \frac{N_{a}(T)}{p_{a}} + S^{a}(N_{a}(T)) 
\end{align*}
where $\frac{1}{p_{a}}+\xi^{a}_{i}$ is the total random number of allocation it takes for action $a$ to be realized in the $i^{th}$ selection, $\xi^{a}_{i}$ being equal the centered deviation with respect to the expected value $\frac{1}{p_{a}}$. Of key importance in our proof is $S^{a}(N_{a}(T))$, the cumulative deviation defined as $\sum_{i=1}^{N_{a}(T)}\xi^{a}_{i}$. Note that it is well approximated in large $T$ regime as a random sum of $N_{a}(T)$ i.i.d. geometric centered variable of parameter $p_{a}$. Given this and the total budget constraint, we have the simple relation $\tau_{a}^{on}(T)=\tau_{a}^{off}(T) + \frac{1}{d_{eff}}\sum_{b}[\frac{S^{a}(N_{a}(T))}{p_{b}}-\frac{S^{b}(N_{b}(T))}{p_{a}}]$.
A relevant quantity to introduce is the random allocation difference $\Delta \tau_{a,b}$ between actions $a$ and $b$ defined by:
\begin{align*}
    \Delta \tau_{a,b} \triangleq \frac{1}{d_{eff}} (\frac{S^{a}(N_{a}(T))}{p_{b}}-\frac{S^{b}(N_{b}(T))}{p_{b}})
\end{align*}
Using this notation, we simply have $\tau_{a}^{on} = \tau_{a}^{off} + \sum_{b} \Delta \tau_{a,b}$. We then introduce the random sets $I^{+}\triangleq\{a: \tau_{a}^{on}\geq \tau_{a}^{off}\}=\{a: \sum_{b} \Delta \tau_{a,b} \geq 0\}$ and $I^{-}\triangleq\{a: \tau_{a}^{on}< \tau_{a}^{off}\}=\{a: \sum_{b} \Delta \tau_{a,b} < 0\}$. On the one hand, $I^{+}$ represents the set of actions that are more sampled by the adaptive policy than by the offline policy i.e. that leads to a gain thanks to the greedy property. One the other hand, $I^{-}$ is the set of actions under-selected by the adaptive policy, leading to a loss although inferior in absolute value to the resulting gain of $I^{+}$. As for the proof of the single monitoring case, we condition on the realization $(N_{a}(T))$ and use a continuous approximation given this conditioning to study the difference of gain. Thus, we have the action gain  for $a \in I^{+}$,:
\begin{align*}
    g_{a} &\triangleq \frac{1}{p_{a}}[\psi_{\alpha}(N_{a}(T))-\psi_{\alpha}(N_{a}(T)-p_{a}\sum_{b}\Delta \tau_{a,b})] \\
    &\approx \frac{1}{p_{a}}\Big[p_{a}\Big(\sum_{b} \Delta \tau_{a,b}\Big)\psi_{\alpha}^{(1)}(N_{a}(T))-\frac{(p_{a}\sum_{b} \Delta \tau_{a,b})^{2}}{2}\psi_{\alpha}^{(2)}(N_{a}(T))\Big].
\end{align*}
On the other hand, for $a \in I^{-}$, we have the action loss still under the continuous approximation:
\begin{align*}
    l_{a}&\triangleq \frac{1}{p_{a}}[\psi_{\alpha}(N_{a}(T)+p_{a}\sum_{b}\Delta \tau_{a,b})-\psi_{\alpha}(N_{a}(T))] \\
    &\approx \frac{1}{p_{a}}\Big[p_{a}\Big(\sum_{b} \Delta \tau_{a,b}\Big)\psi_{\alpha}^{(1)}(N_{a}(T))+\frac{(p_{a}\sum_{b} \Delta \tau_{a,b})^{2}}{2}\psi_{\alpha}^{(2)}(N_{a}(T))\Big].
\end{align*}
By introducing $\mathcal{G}_{\text{adapt}} \triangleq \sum_{a\in I^{+}} g_{a}-\sum_{a\in I^{-}} l_{a}$, the adaptive equivalent of $\mathcal{G}_{\text{single}}$ and combining previous two results, we deduce:
\begin{align*}
    \mathcal{G}_{\text{adapt}} &= \sum_{a\in I^{+}}\frac{1}{p_{a}}\Big[p_{a}\Big(\sum_{b} \Delta \tau_{a,b}\Big)\psi_{\alpha}^{(1)}(N_{a}(T))-\frac{(p_{a}\sum_{b} \Delta \tau_{a,b})^{2}}{2}\psi_{\alpha}^{(2)}(N_{a}(T))\Big] \\&\quad \quad - \sum_{a\in I^{-}} \frac{1}{p_{a}}\Big[p_{a}\Big(\sum_{b} \Delta \tau_{a,b}\Big)\psi_{\alpha}^{(1)}(N_{a}(T))+\frac{(p_{a}\sum_{b} \Delta \tau_{a,b})^{2}}{2}\psi_{\alpha}^{(2)}(N_{a}(T))\Big]\\
    &= \sum_{a\in I^{+}} \Big(\sum_{b} \Delta \tau_{a,b}\Big)\psi_{\alpha}^{(1)}(N_{a}(T)) - \sum_{a\in I^{-}} \Big(\sum_{b} \Delta \tau_{a,b}\Big)\psi_{\alpha}^{(1)}(N_{a}(T)) \\
    &\quad \quad - \frac{1}{2}\sum_{a\in[d]}(p_{a}\sum_{b} \Delta \tau_{a,b})^{2}\psi_{\alpha}^{(2)}(N_{a}(T))
\end{align*}
We then leverage the Taylor expansion $\psi_{\alpha}^{(1)}(N_{a}(T)) = \psi_{\alpha}^{(1)}(\Bar{N}(T)) + \psi_{\alpha}^{(2)}(\bar{N}(T))(\bar{N}(T)-N_{a}(T))$, where $\Bar{N}(T)\triangleq \sum_{a\in [d]}N_{a}(t)/d$. We know that the second term is asymptotically negligible given that $\alpha>0$ and that the difference between $\bar{N}(T)$ and $N_{a}(T)$ is constant with exponential probability, thanks to the greedy policy property. We combine this result with the fact that by definition $\sum_{a\in I^{+}} \sum_{b}\Delta \tau_{a,b} - \sum_{a\in I^{-}} \sum_{b}\Delta \tau_{a,b}=0$ to deduce that at first order:
\begin{align*}
    \mathcal{G}_{\text{adapt}} = - \frac{1}{2}\psi_{\alpha}^{(2)}(\bar{N}(T))\sum_{a\in[d]}(p_{a}\sum_{b} \Delta \tau_{a,b})^{2} \tag{$\mathcal{L}$}\label{AD_Gain}
\end{align*}
We see formula (\ref{AD_Gain}) as the adaptive analogous of (\ref{OS_Gain}). Indeed, it involves the product of the second derivative $\frac{1}{2}\psi_{\alpha}^{(2)}(\bar{N}(T))$, evaluated on a quantity concentrating at $T/d_{\mathit{eff}}$ with a variance term associated to the adaptive action allocation process. We remark that for any $a\in[d]$:
\begin{align*}
    \mathbb{E}\Big[(\sum_{b} \Delta \tau_{a,b})^{2}\Big] &= \frac{1}{d_{\mathit{eff}}^{2}}\mathbb{E}\Big[([\sum_{ b\neq a}\frac{1}{p_{b}}]S^{a}(N_{a}(T))-\frac{1}{p_{a}}\sum_{ b\neq a}S^{b}(N_{b}(T)))^{2}\Big] \\
    &= \frac{1}{d_{\mathit{eff}}^{2}}\left[(d_{\mathit{eff}}-\frac{1}{p_{a}})^{2}\mathbb{V}[S^{a}(N_{a}(T)) ]+\sum_{ b\neq a }\frac{1}{p_{a}^{2}}\mathbb{V}[S^{b}(N_{b}(T)) ]\right]
\end{align*}
and therefore, by summing:
\begin{align*}
    \sum_{a\in[d]}p_{a}\mathbb{E}\Big[(\sum_{b} \Delta \tau_{a,b})^{2}\Big] &= \frac{1}{d_{\mathit{eff}^{2}}}\sum_{a\in[d]}p_{a}\Big[(\sum_{ b\neq a}\frac{1}{p_{b}})^{2}\mathbb{V}[S^{a}(N_{a}(T)) ]+\sum_{ b\neq a }\frac{1}{p_{a}^{2}}\mathbb{V}[S^{b}(N_{b}(T)) ]\Big]\\
    &= \frac{1}{d_{\mathit{eff}^{2}}}\sum_{a\in[d]} \Big[p_{a}(d_{\mathit{eff}}-\frac{1}{p_{a}})^{2}+d_{\mathit{eff}}-\frac{1}{p_{a}}\Big]\mathbb{V}[S^{a}(N_{a}(T)) ]\\
    &= \sum_{a\in[d]} p_{a}\frac{d_{\mathit{eff}}-\frac{1}{p_{a}}}{d_{\mathit{eff}}}\mathbb{V}[S^{a}(N_{a}(T))].
\end{align*}
To obtain the leading order of $\mathbb{V}[S^{a}(N_{a}(T))]$, we use Wald's second equation and the fact that $S^{a}(N_{a}(T))$ is approximated by a sum of geometric random variable of parameter $p_{a}$, modulo a asymptotically negligible summing constraint due to the fixed total budget $T$. This yields $\mathbb{V}[S^{a}(N_{a}(T))] \sim \frac{1-p_{a}}{p_{a}^{2}}\mathbb{E}[N_{a}(T)]\sim\frac{1-p_{a}}{p_{a}^{2}}\mathbb{E}[\bar{N}(T)]$, where the last results leverages again the fact that the difference between the two quantities is constant with exponential probability. We conclude with further algebraic calculation that:
\begin{align*}
    \mathbb{E}[\mathcal{G}_{\text{adapt}}] &\sim -\mathbb{E}[\frac{\psi_{\alpha}^{(2)}}{2}(\bar{N}(T))]\sum_{a\in[d]} p_{a}\frac{d_{\mathit{eff}}-\frac{1}{p_{a}}}{d_{\mathit{eff}}}\frac{1-p_{a}}{p_{a}^{2}}\mathbb{E}[\bar{N}(T)] \\
    &\sim \frac{\alpha}{2}(\frac{d_{\mathit{eff}}}{T})^{1+\alpha} \sum_{a\in [d]}\frac{1}{p_{a}}\Big[\sum_{b\neq a}\frac{1-p_{b}}{p_{b}}\Big]\frac{T}{d_{\mathit{eff}}^{2}} \\
    &\sim \gamma_{\alpha}(\mathbf{p})\frac{1}{T^{\alpha}}.
\end{align*}
By justifying again that the continuous gain approximation leads to terms of order $o(\frac{1}{T^{\alpha}})$, as done in the proof of Lemma \ref{asympt_off}, we conclude that:
\begin{align*}
    \max_{\pi \in \Pi_{\text{adapt}}} \mathbb{E}[\mathbb{V}_{\alpha}(T,\pi)] - \max_{\pi \in \Pi_{\text{off}}} \mathbb{E}[\mathbb{V}_{\alpha}(T,\pi)] = \gamma_{\alpha}(\mathbf{p})\frac{1}{T^{\alpha}} + o(\frac{1}{T^{\alpha}}).
\end{align*}
\end{proof}


\section{Proof of Sec. \ref{CB} - Contextual Bandits} \label{Proof LCB}

In this section, we prove Thm. \ref{THM Linear arms} of Sec.\ref{CB}, extending the results of MAB to LCB. To do so, we prove Lemmas \ref{Potential Reduction Linear}, \ref{Optimistic Lemma Linear} and Prop. \ref{Potential Control Linear}. Note that the proof of Thm. \ref{THM Linear Optim MTM} is differed to next section. We conclude the section by discussing the extension of our analysis to Generalized Linear Contextual Bandits. 


\subsection{Proof of Lemma \ref{Potential Reduction Linear}}

\PotentialReductionLinear*

\begin{proof}
We have under the event $\neg \mathcal{H}_{\text{UCB}}^{II}(\delta)$ introduced in Lemma \ref{Optimistic Lemma Linear} and thanks to Holder inequality:
\begin{align*}
    \Delta_{t}(a)\triangleq  \max_{\Tilde{a}\in\mathcal{A}_{t}}\langle\theta^{\star},\Tilde{a} \rangle - \langle \theta^{\star},a_{t} \rangle \leq 2\beta_{\delta}(t-1) \|a_{t}\|_{(\mathbb{W}^{C}(t-1))^{-1}} .
\end{align*}
Therefore, the conditional regret is upper-bounded by:
\begin{align*}
    R(T|\neg \mathcal{H}_{\text{UCB}}^{II}(\delta)) \leq \beta_{\delta}(T)\sum_{t=1}^{T}\|a_{t}\|_{(\mathbb{W}^{C}(t-1))^{-1}} = \beta_{\delta}(T)\Tilde{\mathbb{V}}_{\frac{1}{2}}(T,\pi),
\end{align*}
where we introduced $\Tilde{\mathbb{V}}_{\frac{1}{2}}(T,\pi) \triangleq \sum_{t=1}^{T}\|a_{t}\|_{\mathbb{W}^{C}(t-1)^{-1}}$.
Cauchy Schwartz inequality then allows to make the junction $\Tilde{\mathbb{V}}_{\frac{1}{2}}(T,\pi) \leq \sqrt{T}\sqrt{\mathbb{V}_{1}(T,\pi)}$. We then introduce $\Tilde{\beta}_{\delta}(T)$ a deterministic upper bound on $\beta_{\delta}(T)$:
\begin{align*}
    \beta_{\delta}(T) &= \sqrt{\sigma^{2} \log \left(\frac{\det(\mathbb{W}^{C}_{T})}{\det(\lambda \mathbb{I}_{d})}\right)+2\sigma^{2}\log(\frac{1}{\delta})}+\sqrt{\lambda} \|\theta^{\star}\|_{2} \\
    &\leq \underbrace{\sqrt{\sigma^{2}d \log (1+\frac{T}{d\lambda})+2\sigma^{2}\log(\frac{1}{\delta})}+\sqrt{\lambda} \|\theta^{\star}\|_{2}}_{ \triangleq \Tilde{\beta}_{\delta}(T)} \\
    &= \Theta(\sqrt{d\log(T)}).
\end{align*}
Using the concavity of square root and Jensen's inequality, we have $\mathbb{E}[\sqrt{\mathbb{V}_{1}(T,\pi)}] \leq \sqrt{\mathbb{E}[\mathbb{V}_{1}(T,\pi)]}$. Finally, thanks to Lemma \ref{Optimistic Lemma Linear}, we conclude that:
\begin{align*}
    \mathbb{E}[R(T,\pi_{\text{UCB}})] \leq 2\Tilde{\beta_{\delta}}(T) \sqrt{T\mathbb{E}[\mathbb{V}_{1}(T,\pi_{\textit{UCB}})]} + \delta T\Delta_{max}.
\end{align*}
\end{proof}

\subsection{Statement and Proof of Lemma \ref{Optimistic Lemma Linear}}
Analogous to Lemma \ref{Fail Optim Finite} for the MAB case, one key step in the proof is introduction of the failure of optimism event. Nevertheless, note the difference with the choice of norm.
\begin{lemma} \label{Optimistic Lemma Linear}
For any $\delta \in ]0,1]$, uniform regularization $\lambda>0$ and censored action generating process $(\mathbb{W}^{C}_{t})_{t\leq T}$, let's introduce the event: 
\begin{align*}
    \mathcal{H}_{\text{UCB}}^{II}(\delta) \triangleq \Big\{\exists t \geq 0, \|\hat{\theta}^{\lambda}_{t}-\theta^{\star}\|_{\mathbb{W}^{C}_{t}} >  \underbrace{\sqrt{\sigma^{2} \log \left(\frac{\det(\mathbb{W}^{C}_{t})}{\det(\lambda \mathbb{I}_{d})}\right)+2\sigma^{2}\log(\frac{1}{\delta})}+\sqrt{\lambda} \|\theta^{\star}\|_{2}}_{\triangleq\beta_{\delta}(t)}\Big\}.
\end{align*}
We then have $\mathbb{P}(\mathcal{H}_{\text{UCB}}^{II}(\delta))\leq \delta$.
\end{lemma}

\begin{proof}
The proof closely mirrors the self-normalized bound for vector-valued martingales of Thm.$1$ from \cite{NIPS2011_e1d5be1c}. The main subtlety is to apply the results to the censored measurable vectors $(x_{a_{t}}a_{t})$ instead of classically $(a_{t})$. This yields that with probability $1-\delta$, for all $t\geq 0$:
\begin{align*}
    \|\sum_{n=1}^{t}\epsilon_{n}x_{a_{n}}a_{n}\|^{2}_{\mathbb{W}_{t}^{C}} \leq \sigma^{2}\log\frac{\det(\mathbb{W}_{t}^{C})}{\det(\lambda \mathbb{I}_{d})} + 2\log(\frac{1}{\delta}).
\end{align*}
Thus, still on this event, for any $t\geq 0$ and action $a\in \mathbb{R}^{d}$, we have by definition of $\hat{\theta}^{\lambda}_{t}$ (Sec.\ref{UCB-algo}):
\begin{align*}
    \langle a,\hat{\theta}^{\lambda}_{t}\rangle - \langle a,\theta^{\star} \rangle = \langle a, (\mathbb{W}_{t}^{C})^{-1}\sum_{n=1}^{t}\epsilon_{n}x_{a_{n}}a_{n}\rangle - \lambda \langle a,  (\mathbb{W}_{t}^{C})^{-1}\theta^{\star}\rangle,
\end{align*}
and therefore, thanks to Cauchy-Schwartz inequality:
\begin{align*}
    |\langle a,\hat{\theta}^{\lambda}_{t}\rangle - \langle a,\theta^{\star} \rangle| \leq \|a\|_{(\mathbb{W}_{t}^{C})^{-1}}\Big(\|\sum_{n=1}^{t}\epsilon_{t}x_{a_{t}}a_{t}\|_{\mathbb{W}_{t}^{C}} + \lambda^{1/2}\|\theta^{\star}\|_{2}\Big)
\end{align*}
Using previous result, for all $a\in \mathbb{B}_{d}, t\geq 0$, with probability $1-\delta$, we have:
\begin{align*}
     |\langle a,\hat{\theta}^{\lambda}_{t}\rangle - \langle a,\theta^{\star} \rangle| \leq \sigma\sqrt{\log\Big(\frac{\det(\mathbb{W}_{t}^{C})}{\det(\lambda \mathbb{I}_{d})}\Big) + 2\log(\frac{1}{\delta})}+\lambda^{1/2}\|\theta^{\star}\|_{2}
\end{align*}
To conclude, we classically plug-in the value $a=\mathbb{W}_{t}^{C}(\hat{\theta}^{\lambda}_{t}-\theta^{\star})$ and divide both sides by $\|\hat{\theta}^{\lambda}_{t}-\theta^{\star}\|_{\mathbb{W}_{t}^{C}}$ to get that for all $t\geq 0$, with probability $1-\delta$, we have:
\begin{align*}
    \|\hat{\theta}^{\lambda}_{t}-\theta^{\star}\|_{\mathbb{W}^{C}_{t}} \leq  \sigma\sqrt{ \log \Big(\frac{\det(\mathbb{W}^{C}_{t})}{\det(\lambda \mathbb{I}_{d})}\Big)+2\log(\frac{1}{\delta})}+\lambda^{1/2} \|\theta^{\star}\|_{2}
\end{align*}
and therefore, by definition $\mathbb{P}(\mathcal{H}_{\text{UCB}}^{II}(\delta))\leq \delta$.

\end{proof}


\subsection{Proof of Prop. \ref{Potential Control Linear}}

\PotentialControlLinear*

\begin{proof}
First, we use Lemma \ref{CEN Lemma Linear} to deduce that under $\mathcal{H}_{\text{CEN}}^{II}(\delta)$:
\begin{align*}
    \mathbb{V}_{\alpha}(T,\pi|\mathcal{H}_{\text{CEN}}^{II}(\delta)) = \sum_{t=1}^{T}\operatorname{Tr}((\mathbb{W}^{C}_{t-1})^{-\alpha}a_{t}a_{t}^{\top}) \leq c_{\delta}^{\alpha}\sum_{t=1}^{T}\operatorname{Tr}(\mathbb{W}_{t-1}^{-\alpha}a_{t}a_{t}^{\top}).
\end{align*}    
For all $t\geq 1$, we then use the fact that $W_{t}\preceq (1+\frac{1}{\lambda})W_{t-1}$ to deduce $\operatorname{Tr}(\mathbb{W}_{t-1}^{-\alpha}a_{t}a_{t}^{\top}) \leq          (1+\frac{1}{\lambda})^{\alpha}\operatorname{Tr}(\mathbb{W}_{t}^{-\alpha}a_{t}a_{t}^{\top})$. The last and most important step is the integral comparison:
\begin{align*}
    \sum_{t=1}^{T}\operatorname{Tr}(\mathbb{W}_{t}^{-\alpha}a_{t}a_{t}^{\top}) \leq \int_{0}^{T} \operatorname{Tr}(\mathbb{W}(t)^{-\alpha}a(t)a(t)^{\top})\partial t = \operatorname{Tr}\Big(\int_{0}^{T}\mathbb{W}(t)^{-\alpha}a(t)a(t)^{\top}\partial t \Big).
\end{align*}
In the previous result, the continuous extension $(a(t),\mathbb{W}(t))_{t\leq T}$ of $(a_{t},\mathbb{W}_{t})_{t\in [T]}$ for a given policy $\pi$ is defined for any time $t\geq 1$ as:
\begin{align*}
    a(t)\triangleq a_{\lfloor t\rfloor} \quad \text{and} \quad \mathbb{W}(t)\triangleq \int_{u=1}^{t}p_{a(u)}a(u)a(u)^{\top}\partial u =  \mathbb{W}_{\lfloor t\rfloor} + (t-\lfloor t\rfloor)p(a_{\lceil t\rceil})a_{\lceil t\rceil}a_{\lceil t\rceil}^{\top}.
\end{align*}
This yields the result:
\begin{align*}
    \mathbb{V}_{\alpha}(T,\pi|\mathcal{H}_{\text{CEN}}^{II}(\delta)) \leq c_{\delta}^{\alpha}(1+\frac{1}{\lambda})^{\alpha}\operatorname{Tr}\Big(\int_{0}^{T}\mathbb{W}(t)^{-\alpha}a(t)a(t)^{\top}\partial t \Big).
\end{align*}
Finally, we conclude thanks to Lemma \ref{CEN Lemma Linear} that:
\begin{align*}
    \mathbb{E}[\mathbb{V}_{\alpha}(T,\pi)] \leq \frac{\delta}{\lambda^{\alpha}} + C(\delta)^{\alpha} \operatorname{Tr}\Big(\int_{0}^{T}\mathbb{W}(t)^{-\alpha}a(t)a(t)^{\top}\partial t\Big).
\end{align*}
\end{proof}

\begin{remark}\label{tour de force}
The main \textit{tour de force} of the continuous approximation we employ is to relax the maximization problem by considering the class of continuous deterministic integrable policies, which is considerably more tractable from an analysis perspective. On the one hand, it allows to get closed-form solution for the maximization problem whereas the discrete approach can only deal with approximations and upper bounds. On the other hand, it clearly reveals the underlying matrix function the discrete approach is approximating and henceforth allows to leverage powerful integration results. We leverage again this idea in the context of Sec.\ref{CB} to tackle impact of censorship.

To illustrate the abovementioned points, we remark that for the simpler case of classical uncensored environment, we obtain for $\alpha>0, \alpha \neq 1$:
\begin{align*}
    \sum_{t=1}^{T}\|a_{t}\|^{2}_{\mathbb{W}_{t-1}^{-\alpha}} 
    &\leq \Big(\frac{\lambda+1}{\lambda}\Big)^{\alpha}\frac{\operatorname{Tr}\Big(\int_{0}^{T} \partial\mathbb{W}(t)^{1-\alpha}\Big)}{1-\alpha} = \Big(\frac{\lambda+1}{\lambda}\Big)^{\alpha}\frac{\operatorname{Tr}(\mathbb{W}^{1-\alpha}_{T}-\mathbb{W}^{1-\alpha}_{0})}{1-\alpha}.
\end{align*}
For $\alpha < 1$, we then have thanks to Lemma \ref{Optimization Lemma Finite} the worst case bound $\operatorname{Tr}(\mathbb{W}_{T}^{1-\alpha}) \leq d^{\alpha}(d\lambda + T)^{1-\alpha}$ and henceforth:
\begin{align*}
    \sum_{t=1}^{T}\|a_{t}\|^{2}_{\mathbb{W}_{t-1}^{-\alpha}} \leq \Big(\frac{\lambda+1}{\lambda}\Big)^{\alpha} \frac{d^{\alpha}(d\lambda + T)^{1-\alpha}-d\lambda^{1-\alpha}}{1-\alpha}
\end{align*}
On the other hand, for $\alpha > 1$, we deduce:
\begin{align*}
    \sum_{t=1}^{T}\|a_{t}\|^{2}_{\mathbb{W}_{t-1}^{-\alpha}} \leq \Big(\frac{\lambda+1}{\lambda}\Big)^{\alpha}\frac{d\lambda^{1-\alpha}}{\alpha-1}.
\end{align*}
Finally, for $\alpha =1$, we use the formula $\operatorname{Tr}(\log(A))=\log(\det A)$ to deduce:
\begin{align*}
     \sum_{t=1}^{T}\|a_{t}\|^{2}_{\mathbb{W}_{t-1}^{-1}} 
     &\leq \frac{\lambda+1}{\lambda}\int_{0}^{T}\frac{\partial\log\det(\mathbb{W}(t))}{\partial t}\partial t = \frac{\lambda+1}{\lambda}\operatorname{Tr}(\log\mathbb{W}_{T}-\log\mathbb{W}_{0}) = \frac{\lambda+1}{\lambda}\log\frac{\det\mathbb{W}_{T}}{\det\mathbb{W}_{0}}\\
     &\leq \frac{\lambda+1}{\lambda}\log(1+\frac{T}{\lambda d}),
\end{align*}
where we used again Lemma \ref{Optimization Lemma Finite} to obtain the last (worst-case) upper bound. In doing so, we recover and extend the recent results of \cite{carpentier2020elliptical} in a more natural way.\footnote{Yet, we conjecture that the preliminary use of Cauchy Schwartz inequality in the case $\alpha > 1$ to affirm $\sum_{t=1}^{T}\|a_{t}\|_{\mathbb{W}_{t-1}^{-\alpha}}\leq \sqrt{T\sum_{t=1}^{T}\|a_{t}\|^{2}_{\mathbb{W}_{t-1}^{-\alpha}}}$ is suboptimal in this case as it imposes a $\mathcal{O}(\sqrt{T})$ scaling.} Note that the rank $1$ assumption is not needed in the continuous relaxation and therefore our results still hold whenever $a(t)a(t)^{T}$ is replaced by any positive semi-definite matrix $H(t)$.
\end{remark}


\subsection{Statement of Lemma \ref{CEN Lemma Linear}}

In order to prove previous property on $\mathbb{V}_{\alpha}$, a key step mirroring the MAB case is the use of high confidence lower bound on the censorship process, proven using anytime matrix martingale inequalities:
\begin{lemma}(\cite{adpt_cofond_Russo}) \label{CEN Lemma Linear}
For any $\delta \in ]0,1]$, $\lambda>0$ and policy $\pi$, let's introduce the event:
\begin{align*}
    \mathcal{H}_{\text{CEN}}^{II}(\delta) \triangleq \Big\{\exists t\geq 0, \mathbb{W}^{C}_{t} \prec \frac{1}{c_{\delta}} \mathbb{W}_{t}\Big\},
\end{align*}
where $c_{\delta}\triangleq 8\max(\frac{\log(d/\delta))}{\lambda},1)$.
We then have $\mathbb{P}(\mathcal{H}_{\text{CEN}}^{II}(\delta))\leq \delta$.
\end{lemma}

Note that picking as in the MAB case $\delta \sim d/T^{2}$ would lead to a constant $c_{\delta}=\Theta(\log(T))$, that is a worsening confidence interval, except if we manage to control the initialization. One interesting technical question for future work would be to allow an initialization condition as in Lemma \ref{Derando Censo Linear} ensuring $\mathbb{W}(T_{0})$ counterbalance $\log(d/\delta)$.


\subsection{Proof of Thm. \ref{THM Linear arms}}

\THMLineararms*

\begin{proof}
Analogous to the MAB case, we use Lemma \ref{Potential Reduction Linear} to deduce:
\begin{align*}
    \mathbb{E}[R(T,\pi_{\text{UCB}})] \leq 2\Tilde{\beta_{\delta}}(T) \sqrt{T\mathbb{E}[\mathbb{V}_{1}(T,\pi_{\textit{UCB}})]} + \delta T\Delta_{max},
\end{align*}
where we have:
\begin{align*}
    \Tilde{\beta}_{\delta}(T) = \sqrt{\sigma^{2}d \log (1+\frac{T}{d\lambda})+2\sigma^{2}\log(\frac{1}{\delta})}+\sqrt{\lambda} \|\theta^{\star}\|_{2}.
\end{align*}
We then pick $\delta = \frac{d}{T^{2}}$, which yields:
\begin{align*}
     \mathbb{E}[R(T,\pi_{\text{UCB}})] \leq 2\Big(\sqrt{\sigma^{2}d \log (1+\frac{T}{d\lambda})+2\sigma^{2}\log(\frac{T^{2}}{d})}+\sqrt{\lambda} \|\theta^{\star}\|_{2}\Big) \sqrt{T\mathbb{E}[\mathbb{V}_{1}(T,\pi_{\textit{UCB}})]} + \frac{d \Delta_{max}}{T}.
\end{align*}
We then apply Lemma \ref{Potential Control Linear} with $\alpha=1$ and $\delta = \frac{d}{T^{2}}$ to deduce:
\begin{align*}
     \mathbb{E}[\mathbb{V}_{\alpha}(T,\pi)] &\leq \frac{d}{\lambda T^{2}} + 8\frac{\lambda+1}{\lambda}\max(\frac{2\log(T)}{\lambda},1) \operatorname{Tr}\Big(\int_{0}^{T}\mathbb{W}(t)^{-1}a(t)a(t)^{\top}\partial t\Big) \\
     &\leq \frac{d}{\lambda T^{2}} + 8\frac{\lambda+1}{\lambda}\max(\frac{2\log(T)}{\lambda},1) \max_{\pi \in \Pi}\operatorname{Tr}\Big(\int_{0}^{T}\mathbb{W}(t)^{-1}a(t)a(t)^{\top}\partial t\Big).
\end{align*}
By applying Thm. \ref{THM Linear Optim MTM}, we deduce the two possibilities:
\begin{itemize}
    \item \textbf{Case 1: Single region $i_{l}$.} The effective dimension corresponding to this dynamics is $d/p_{i_{l}}$, with the following equality for $T\geq t_{l-1}$:
    \begin{align*}
        \max_{\pi \in \Pi}\operatorname{Tr}\Big(\int_{0}^{T}\mathbb{W}(t)^{-1}a(t)a(t)^{\top}\partial t\Big) = \frac{1}{p_{i_{l}}}\log\det(\mathbb{W}(T))+ \sum_{n=1}^{l-1} (\frac{1}{p_{i_{n}}}-\frac{1}{p_{i_{n+1}}})\log\det\mathbb{W}(t_{n}),
    \end{align*}
    where we have for $T\geq t_{l-1}$ $\mathbb{W}(T) =  p_{i_{l}}(T-t_{l-1})\mathbb{W}_{i_{l}} + \mathbb{W}(t_{l-1})$. Explicit formula of $(t_{n},\mathbb{W}(t_{n}))$ are given for all $n\leq l$ in Cor. \ref{Path Formula}. We then note that:
    \begin{align*}
        \frac{1}{p_{i_{l}}}\log\det(\mathbb{W}(T)) &= \frac{1}{p_{i_{l}}}\log\det(p_{i_{l}}(T-t_{l-1})\mathbb{W}_{i_{l}} + \mathbb{W}(t_{l-1})) \\
        &= d_{\mathit{eff}}\log(T) + \frac{1}{p_{i_{l}}}\log\det(p_{i_{l}}(1-\frac{t_{l-1}}{T})\mathbb{W}_{i_{l}} + \frac{1}{T} \mathbb{W}(t_{l-1})).
    \end{align*}
    For $T\geq t_{l-1}$, we then write this in the form:
    \begin{align*}
        \max_{\pi \in \Pi}\operatorname{Tr}\Big(\int_{0}^{T}\mathbb{W}(t)^{-1}a(t)a(t)^{\top}\partial t\Big) = d_{\mathit{eff}}\log(T) + f(T),
    \end{align*}
    where $f(T)=o(\log(T))$.
    \item \textbf{Case 2: Bi-region $(i_{l+1},i_{l})$.} Similarly, for $T\geq t_{l}$, we have:
    \begin{align*}
        \max_{\pi \in \Pi}\operatorname{Tr}\Big(\int_{0}^{T}\mathbb{W}(t)^{-1}a(t)a(t)^{\top}\partial t\Big) &= d_{\mathit{eff}}\log(1+\frac{T-t_{l}}{t_{l}+\lambda^{\star}}) + \sum_{n=1}^{l} (\frac{1}{p_{i_{n}}}-\frac{1}{p_{i_{n+1}}})\log\det\mathbb{W}(t_{n}) \\
        &=  d_{\mathit{eff}}\log(T)+ d_{\mathit{eff}}\log(\frac{1}{T}+\frac{1-\frac{t_{l}}{T}}{t_{l}+\lambda^{\star}}) \\
        &\quad \quad + \sum_{n=1}^{l} (\frac{1}{p_{i_{n}}}-\frac{1}{p_{i_{n+1}}})\log\det\mathbb{W}(t_{n}) \\
        &= d_{\mathit{eff}}\log(T)+ f(T),
    \end{align*} 
    where $f(T)=o(\log(T))$.
\end{itemize}
Therefore, for given $d_{\mathit{eff}}$, $f$ and $t_{0}$, we know that the following holds for all $T\geq t_{0}$:
\begin{align*}
    \mathbb{E}[\mathbb{V}_{\alpha}(T,\pi)] &\leq \frac{d}{\lambda T^{2}} + 8\frac{\lambda+1}{\lambda}\max(\frac{2\log(T)}{\lambda},1) \operatorname{Tr}\Big(\int_{0}^{T}\mathbb{W}(t)^{-1}a(t)a(t)^{\top}\partial t\Big) \\
     &\leq \frac{d}{\lambda T^{2}} + 8\frac{\lambda+1}{\lambda}\max(\frac{2\log(T)}{\lambda},1) (d_{\mathit{eff}}\log(T)+ f(T)).
\end{align*}
Putting the pieces together yields for $T\geq t_{0}$:
\begin{align*}
    \mathbb{E}[R(T,\pi_{\text{UCB}})] &\leq 2\Big(\sqrt{\sigma^{2}d \log (1+\frac{T}{d\lambda})+2\sigma^{2}\log(\frac{T^{2}}{d})}+\sqrt{\lambda} \|\theta^{\star}\|_{2}\Big) \sqrt{T}\Big(\frac{d}{\lambda T^{2}} \\&+ 8\frac{\lambda+1}{\lambda}\max(\frac{2\log(T)}{\lambda},1) (d_{\mathit{eff}}\log(T)+ f(T))\Big)^{1/2} + \frac{d \Delta_{max}}{T}.
\end{align*}
By imposing regularization of order $\lambda = o(\log(T))$ only considering the leading order, this yields:
\begin{align*}
    \mathbb{E}[R(T,\pi_{\text{UCB}})] \leq \Tilde{\mathcal{O}}(\sqrt{(d+4)\sigma^{2}}\sqrt{d_{\mathit{eff}}}\sqrt{T}).
\end{align*}
Finally, by working in large $d$ regime, we finally conclude that:
\begin{align*}
    \mathbb{E}[R(T,\pi_{\text{UCB}})] \leq \Tilde{\mathcal{O}}(\sigma\sqrt{d\cdot d_{\mathit{eff}}}\sqrt{T}).
\end{align*}
Again, we note that our proof easily allows to get high-probability bounds on regret instead of bounds on its expected value.

\end{proof}

\subsection{Extension to Generalized Linear Contextual Bandits} \label{gen linear}

On what follows, we provide a sketch of the extension our results to Generalized Linear Contextual Bandits (GLCB) but differ the complete treatment to future work. In this model, the reward of a given action $a$ is assumed to be of the form:
\begin{align*}
    r(a) = \mu(\langle a,\theta^{\star}\rangle)
\end{align*}
for a given function $\mu$ strictly increasing, continuously differentiable and real-valued. Notable instances of such a problem include the Logistic bandit and the Poisson bandit. Of particular importance in the dimensionality study of the problem are the constants:
\begin{align*}
    L_{\mu}=\sup _{a \in \cup\mathcal{A}_{t}} \mu^{(1)}(\langle a, \theta^{\star}\rangle) \quad \text{and} \quad \kappa=\inf _{a \in \cup\mathcal{A}_{t}} \mu^{(1)}(\langle a, \theta^{\star}\rangle).
\end{align*}
An important requirement of GLCB is the assumption $\kappa>0$ needed to ensure identifiability of $\theta^{\star}$ and asymptotic normality. 
Given this, the suited definition of pseudo-regret considered is: 
\begin{align*}
    R(T,\pi) \triangleq \sum_{t=1}^{T} \max_{a\in\mathcal{A}_{t}}\mu(\langle a, \theta^{\star}\rangle) - \mu(\langle a_{t}, \theta^{\star}\rangle) 
\end{align*}
Note that this regret can be easily mapped to the one studied above thanks to the fact that $L_{\mu}$ is a Lipschitz constant for $\mu$: for all $a,\Tilde{a} \in \cup\mathcal{A}_{t}$, $|\mu(\langle a,\theta^{\star}\rangle)-\mu(\langle \Tilde{a},\theta^{\star}\rangle)|\leq L_{\mu}|\langle a,\theta^{\star}\rangle-\langle \Tilde{a},\theta^{\star}\rangle|$. Mirroring the proof of \cite{li2017provably}, we use a Maximum Likelihood Estimator (MLE) instead of a Least-Square Estimator for $\theta^{\star}$. More precisely, we define $\hat{\theta}^{\mathit{MLE}}_{t}$ as the solution of the equation:
\begin{align*}
    \sum_{n=1}^{t}\langle a_{n},\epsilon_{t}+\mu(\langle a_{n},\theta^{\star}\rangle)-\mu(\langle a_{n},\theta\rangle)\rangle = 0
\end{align*}
A minor difference between the approach of \cite{li2017provably} and what precedes is the use of a period of initial random sampling (e.g. \textit{exploration}) instead of the regularization to ensure inversibility of the design matrix $\mathbb{W}^{C}_{t}$. More precisely, the initial sampling ensures that with high-probability, $\lambda_{\min}(\mathbb{W}^{C}_{t})>0$ in a finite time $T_{\text{init}}$. To be possible, this requires the assumption that there exists $\sigma_{0}^{2}>0$ such that for all $t\geq 1$, we have $\lambda_{min}\left(\mathbf{E}_{a \in \mathcal{A}_{t}}\left[a a^{\top}\right]\right) \geq \sigma_{0}^{2}$, where the expectation $\mathbf{E}$ is associated with an uniform sampling of actions. Under the same assumption, the impact of censorship on this initialization step is at worst an increase of the sampling time to $\Tilde{T_{\text{init}}}\triangleq T_{\text{init}}/p_{\min}$, which is still constant. Following Lemma $9$ of \cite{li2017provably}, we then consider the censored high-probability confidence set for any $\delta \in [\frac{1}{T},1]$:
\begin{align*}
    \mathcal{H}_{\text{UCB}}^{III}(\delta) \triangleq \Big\{\exists t \geq 0, \|\hat{\theta}^{\mathit{MLE}}_{t}-\theta^{\star}\|_{\mathbb{W}^{C}_{t}} >  \frac{\sigma}{\kappa} \sqrt{\frac{d}{2} \log (1+2 \frac{t}{d})+\log (1 / \delta)} \quad \text{and} \quad \lambda_{\min}(\mathbb{W}^{C}_{t})>1 \Big\}.
\end{align*}
and a direct extension of their results allows us to conclude $\mathbb{P}(\mathcal{H}_{\text{UCB}}^{III}(\delta))\leq \delta$. Note that the constant $\kappa$ appears when upper bounding in the Loewner order the Fischer Information Matrix of the MLE by the matrix $\mathbb{W}^{C}_{t}$. Post-initialization, the conditional regret is then upper bounded by:
\begin{align*}
    R(T,\pi_{\text{UCB}}|\neg \mathcal{H}_{\text{UCB}}^{III}(\delta)) &\leq \Tilde{T_{\text{init}}}\Delta_{max} + \sum^{T}_{t=T_{\text{init}}}L_{\mu} \frac{\sigma}{\kappa} \sqrt{\frac{d}{2} \log (1+2 \frac{t}{d})+\log (1 / \delta)}\|a_{t}\|_{(\mathbb{W}^{C}_{t})^{-1}} \\
    &\leq \Tilde{T_{\text{init}}}\Delta_{max} + L_{\mu} \frac{\sigma}{\kappa} \sqrt{\frac{d}{2} \log (1+2 T / d)+\log (1 / \delta)}\sqrt{T\mathbb{V}_{1}(\pi_{\text{UCB}},T)},
\end{align*}
Combining these elements and taking $\delta = \frac{1}{T}$, we conclude that:
\begin{align*}
    \mathbb{E}[R(T,\pi_{\text{UCB}})]&\leq \Tilde{\mathcal{O}}\Big(\frac{L_{\mu}}{\kappa}\sqrt{d}\sqrt{T\mathbb{E}[\mathbb{V}_{1}(\pi_{\text{UCB}},T)]}\Big) \leq \Tilde{\mathcal{O}}\Big(L_{\mu}\frac{\sqrt{d\cdot d_{\mathit{eff}}}}{\kappa}\sqrt{T}\Big),
\end{align*}
where we used Thm. \ref{THM Linear Optim MTM} to control $\mathbb{E}[\mathbb{V}_{1}(\pi_{\text{UCB}},T)]$ as done in the proof of Th,\ref{THM Linear arms}.


\section{Effective Dimension and Temporal Dynamics for Multi-Threshold Models} \label{Proof Multi}

In this section, we prove Thm. \ref{THM Linear Optim MTM} and discuss its implications. In doing so, we introduce and prove Lemmas \ref{One-step Transient Analysis}, \ref{Dual Reachability Analysis}, \ref{Bi-Region Effective Dimension} and Cor. \ref{Path Formula}. We conclude the section by illustrating results for the single-threshold model, through Cor. \ref{single_thres}.


\subsection{Supplementary Notations}\label{sup_not}

Without loss of generality (i.e. up to an orthogonal transformation), we can consider that $u\equiv e_{d}$, the $d^{th}$ basis vector. Given this, for two regions $i<j$, we introduce the notations:
\begin{align*}
    l(i,j) \triangleq \frac{\sin^{2}(\rho_{i})}{\sin^{2}(\rho_{j})} \quad &\textit{and} \quad u(i,j)  \triangleq \frac{\cos^{2}(\rho_{i})}{\cos^{2}(\rho_{j})} \\
    r^{\star}(i,j) \triangleq \frac{(d-1) u(i,j)+l(i,j)}{d} \quad &\textit{and} \quad r^{\dagger}(i,j) \triangleq \frac{1}{r^{\star}(j,i)} = \frac{dl(i,j)u(i,j)}{u(i,j)+(d-1) l(i,j)} \\
    \mathbb{W}_{i} \triangleq \begin{pmatrix}
\frac{\cos^{2}(\rho_{i})}{d-1}\mathbb{I}_{d-1} & (0) \\
(0) & \sin^{2}(\rho_{i}) 
\end{pmatrix}\quad &\textit{and} \quad\mathbb{W}(i,j) = \begin{pmatrix}
\cos^{2}(\rho_{j})(u(i,j)-\frac{p_{i}}{p_{j}})\mathbb{I}_{d-1} & (0) \\
(0) & \sin^{2}(\rho_{j})(\frac{p_{i}}{p_{j}}-l(i,j)) 
\end{pmatrix}.
\end{align*}
Whenever $i$ and $j$ are clear from context, we use in $u$ (resp. $l$) as abbreviation for $u(i,j)$ (resp. $l(i,j)$).

\subsection{Proof of Thm. \ref{THM Linear Optim MTM}}

\THMLinearOptimMTM*

\begin{algorithm}
    \SetAlgoLined
    \KwInit{Set current region $S\gets k$}
    \While(\tcc*[f]{Lemma \ref{One-step Transient Analysis},Fig.\ref{reach_stat}}){a region is reachable from region $S$}{
        \textbf{play} region $S$ optimal policy \textbf{until} first reachable region $i^{\star}$ is reached\;
        \eIf(\tcc*[f]{Lemma \ref{Dual Reachability Analysis}, Fig.\ref{reach}}){region $i^{\star}$ is dual reachable from region $S$}{
            Bi-region $(i^{\star},S)$ effective dimension (case 2)\tcc*{Lemma \ref{Bi-Region Effective Dimension}}
            \textbf{play} Bi-region $(i^{\star},S)$ optimal policy\;
            \textbf{End}\;}
            {Update current region $S\gets i^{\star}$\tcc*{Lemma \ref{Dual Reachability Analysis}, Fig.\ref{switch}}
            }
    }
    Single region $S$ effective dimension (case 1)\tcc*{Lemma \ref{One-step Transient Analysis}}
    \textbf{play} region $S$ optimal policy\;
\caption{Algorithmic description of the dynamics of $\mathbb{W}(t)$}
\label{Dyn_MT}
\end{algorithm}

\begin{proof}
We first summarize the dynamics of the optimal policy of (\ref{optim_prob}) through an algorithmic description in Alg. \ref{Dyn_MT}. Two key notions of our analysis are the concepts of reachability and dual reachability of a region $i$ from a base region $j$, as described in Lemmas \ref{One-step Transient Analysis} and \ref{Dual Reachability Analysis} and schematized in Fig.\ref{reach_stat}, \ref{reach} and \ref{switch}. Formally, they can be written as two independent necessary constraints on the ratio $p_{i}/p_{j}$: $p_{i}/p_{j}<r^{\star}(i,j)$ for reachability and $p_{i}/p_{j}>r^{\dagger}(i,j)$ for dual reachability.

The categorization result provided in the statement of Thm. \ref{THM Linear Optim MTM} follows from the two possible termination condition of the algorithm. We use as algorithmic invariant to ensure the termination the fact that the set of reachable regions is strictly decreasing for inclusion and finite. Hence, the while loop will terminate either because a dual reachable region is reached or because no more regions are reachable. In order to not overload the presentation, time aspect is not present in the algorithmic description but is extensively covered in Lemmas \ref{One-step Transient Analysis}, \ref{Dual Reachability Analysis}, \ref{Bi-Region Effective Dimension} and Cor. \ref{Path Formula}, as well as in what follows. One of our main finding is that the dynamics of the optimal policy of (\ref{optim_prob}) are described through $\mathbb{W}(t)$ by two qualitatively different regimes. We emphasize that our continuous approach to analyzing cumulative censored potential is key to obtaining these results.
\paragraph{Transient Regime:}
From the \textbf{while} loop in the algorithmic description results a so-called transient regime. More precisely, there exists a decreasing sequence of censorship regions $\{i_{1}=k,\dots,i_{l}\}$ of length $l \in [k+1]$ and associated time sequence $\{t_{0}\triangleq 0,t_{1},\dots,t_{l}\}$ such that whenever $t_{j}\leq t \leq t_{j+1}$ for a given index $j\leq l-1$, the evolution of $\mathbb{W}(t)$ is given by:
    \begin{align*}
        \mathbb{W}(t) &=  p_{i_{j+1}}(t-t_{j})\mathbb{W}_{i_{j+1}} + \mathbb{W}(t_{j}) = p_{i_{j+1}}(t-t_{j})\mathbb{W}_{i_{j+1}} + \sum_{n=1}^{j} p_{i_{n}}(t_{n}-t_{n-1})\mathbb{W}_{i_{n}} + \lambda \mathbb{I}_{d}.
    \end{align*}
    This result follows from a simple induction with repeated use of Lemma \ref{One-step Transient Analysis}, giving the exact sequence of censorship regions, Moreover, closed-formed formula for the time sequence is provided in Cor. \ref{Path Formula}. We interpret this transient step as an adversarial self-correction of the initial misspecification of censorship at an extra cost. This characterization of transient regime highlights an important consequence of using classical algorithms in censored environments.
\paragraph{Steady State Regime:} Post-transient regime, the dynamics of $\mathbb{W}(t)$ enter a steady state regime, where one of the two cases necessarily arise:  
\begin{itemize}
    \item \textbf{Case 1: Single region $i_{l}$.} This case arises when the \textbf{while} loop ends because no other regions are reachable. It is equivalent to have the last element of the time sequence $t_{l}$ is equal to $+\infty$ and we have the single region evolution for all $t\geq t_{l-1}$ thanks to Lemma \ref{One-step Transient Analysis}:
    \begin{align*}
        \mathbb{W}(t) &=  p_{i_{l}}(t-t_{l-1})\mathbb{W}_{i_{l}} + \mathbb{W}(t_{l-1}) =  p_{i_{l}}(t-t_{l-1})\mathbb{W}_{i_{l}} + \sum_{n=1}^{l-1} p_{i_{n}}(t_{n}-t_{n-1})\mathbb{W}_{i_{n}}  + \lambda \mathbb{I}_{d}.
    \end{align*}
    The effective dimension corresponding to this dynamic is $d/p_{i_{l}}$, with the following equality for $T\geq t_{l-1}$:
    \begin{align*}
        \int_{0}^{T}\frac{1}{p(a(t))}\frac{\partial\log\det(\mathbb{W}(t))}{\partial t}\partial t = \frac{1}{p_{i_{l}}}\log\det(\mathbb{W}(T))+ \sum_{n=1}^{l-1} (\frac{1}{p_{i_{n}}}-\frac{1}{p_{i_{n+1}}})\log\det\mathbb{W}(t_{n}),
    \end{align*}
    where the closed-form formula for $\mathbb{W}(t_{n})$ is provided in Cor. \ref{Path Formula} for all $n\leq l-1$.
    \item \textbf{Case 2: Bi-region $(i_{l+1},i_{l})$.} This case arises when the \textbf{while} loop ends because dual reachable region $i_{l+1}$ is reached from region $i_{l}$, with $i_{l+1}<i_{l}$. For all $t\geq t_{l}$, Lemma \ref{Dual Reachability Analysis} yields the evolution:
    \begin{align*}
         \mathbb{W}(t) &\propto p_{i_{l+1}}(t+\lambda^{\star})\begin{pmatrix}
\cos^{2}(\phi_{i_{l}})(u(i_{l+1},i_{l})-\frac{p_{i_{l+1}}}{p_{i_{l}}})\mathbb{I}_{d-1} & (0) \\
(0)  &\sin^{2}(\phi_{i_{l}})(\frac{p_{i_{l+1}}}{p_{j}}-l(i_{l+1},i_{l})) 
\end{pmatrix}.
    \end{align*}
    where $\lambda^{\star}$ and the proportionality factor are specified in the proof. The corresponding effective dimension is given by (\ref{bi_reg}) and the following equality holds for all $T\geq t_{l}$ thanks to Lemma \ref{Bi-Region Effective Dimension}:
    \begin{align*}
        \int_{0}^{T}\frac{1}{p(a(t))}\frac{\partial\log\det(\mathbb{W}(t))}{\partial t}\partial t = d_{\mathit{eff}}\log(1+\frac{T-t_{l}}{t_{l}+\lambda^{\star}}) + \sum_{n=1}^{l} (\frac{1}{p_{i_{n}}}-\frac{1}{p_{i_{n+1}}})\log\det\mathbb{W}(t_{n}),
    \end{align*}
    where the closed-form formula for $\mathbb{W}(t_{n})$ is provided in Cor. \ref{Path Formula} for all $n\leq l$.
\end{itemize}
\end{proof}
\begin{remark}
Fig.\ref{reach_stat} and \ref{switch} provide further insights on formula (\ref{bi_reg}) for $d_{\mathit{eff}}$. Throughout the proof and as illustrated on Fig.\ref{reach_stat}, we see that for (\ref{bi_reg}) to arise, $\frac{p_{i}}{p_{j}}$ must belong to a certain interval $J\triangleq ]\max(1,r^{\dagger}(i,j)),r^{\star}(i,j)[$. As $r^{\star}(i,j)< u(i,j)$ and $r^{\dagger}(i,j)> l(i,j)$, we see (\ref{bi_reg}) as a weighted average of the relative distance of $\frac{p_{i}}{p_{j}}$ to $u(i,j)$ and $l(i,j)$. Fig.\ref{deff} provides a sketch of the variations of $d_{\mathit{eff}}$ as $\frac{p_{i}}{p_{j}}$ evolves in this interval.
\end{remark}

\begin{figure}[h]
    \centering
    \includegraphics[width=0.8 \textwidth]{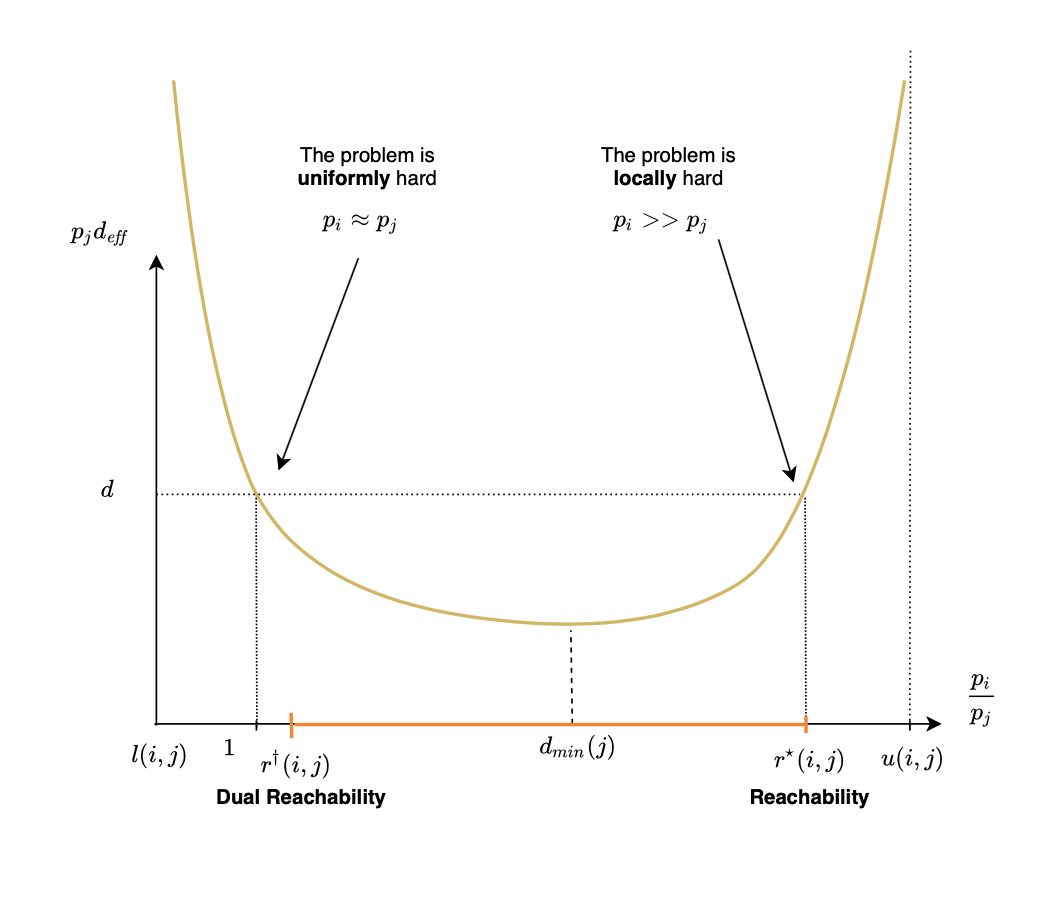}
    \caption{Sketch plot of normalized effective dimension $p_{j}d_{\mathit{eff}}$ with respect to $\frac{p_{i}}{p_{j}}$. We recover the uniform and local hardness conditions mentioned in the discussion of Thm. \ref{THM Linear Optim MTM}, as well as the existence of a \textit{minimum effective dimension} for a certain value of $\frac{p_{i}}{p_{j}}$. The necessary conditions of reachability and dual reachability (Lemma \ref{Dual Reachability Analysis} and \ref{One-step Transient Analysis}) verified by $\frac{p_{i}}{p_{j}}$ impose that it belongs to the orange segment.}
    \label{deff}
\end{figure}


\subsection{Statement and Proof of Lemma \ref{One-step Transient Analysis}}

\begin{lemma}[Reachability Analysis]\label{One-step Transient Analysis}
Let's assume we start at a given time $t_{1}$ in transient censored region $j$, with a matrix
\begin{align*}
    \mathbb{W}(t_{1}) = \begin{pmatrix}
\lambda_{a}\mathbb{I}_{d-1} & (0) \\
(0) & \lambda_{b}
\end{pmatrix},
\end{align*}
where $\lambda_{a}\geq \lambda_{b}$. We introduce $I_{j}\triangleq \{i; i <j \quad \textit{and} \quad \frac{p_{i}}{p_{j}} < r^{\star}(i,j)\}$, the set of reachable regions from region $j$ and affirm that we have the two possible cases:
\begin{itemize}
    \item If $I_{j} = \varnothing$, i.e. no region is reachable from region $j$, we switch to a steady state regime with single region $j$ effective dimension $d_{\mathit{eff}}=d/p_{j}$.
    \item Otherwise, next region added to the transient sequence is $i^{\star}\triangleq \operatorname{argmin}_{i\in I_{j}}\mu^{\star}(i,j,\lambda_{a},\lambda_{b})$, at time $t_{2}\triangleq t_{1}+\frac{1}{p_{j}}\mu^{\star}(i^{\star},j,\lambda_{a},\lambda_{b})$ and we have:
    \begin{align*}
    \mathbb{W}(t_{2}) = \frac{(d-1)\sin^{2}(\phi_{j})\lambda_{a}-\cos^{2}(\phi_{j})\lambda_{b}}{d\cos^{2}(\phi_{j})\sin^{2}(\phi_{j})(r^{\star}(i^{\star},j)-\frac{p_{i}}{p_{j}})}\mathbb{W}(i^{\star},j).
\end{align*}
\end{itemize}
\end{lemma}

\begin{figure}[h]
    \centering
    \includegraphics[width=0.9 \textwidth]{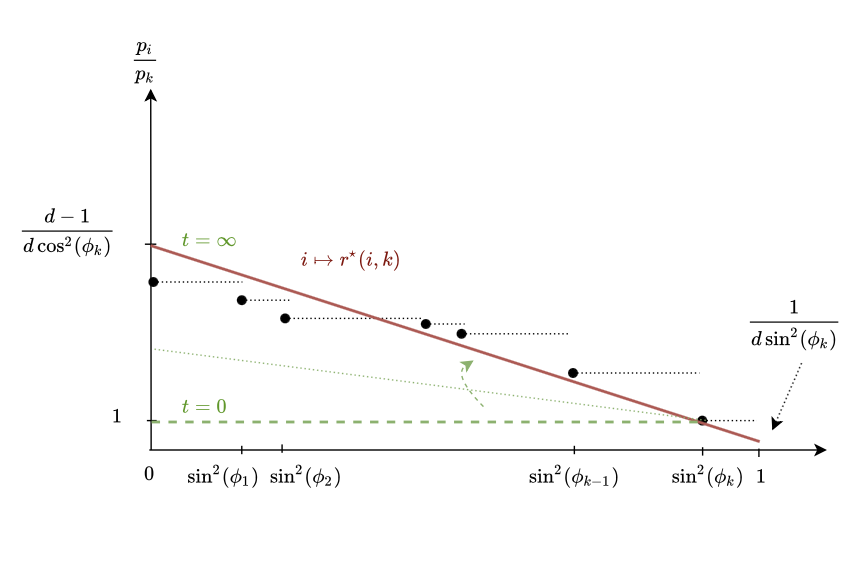}
    \caption{Illustration of the set of reachable regions from a base region $k$, as a function of $\frac{p_{i}}{p_{k}}$. Black dots and lines correspond to censorship regions defined by \ref{MT_model}. In this figure, we see that a region is reachable if and only if the black dot is below the red reachability line. As time increases, the green line rotates with region $k$ as pivot and asymptotically approaches to the red line. Hence, the first reachable region is the one first \textit{reached} by the green line.}
    \label{reach_stat}
\end{figure}

\begin{proof}
First, we note that the initial starting point is recovered for $t_{1}=0$, base censored state $k$  and $\lambda_{a}=\lambda_{b}=\lambda$ but this Lemma allows to go beyond the first step in the study of the behavior of the system. We know the temporal evolution for normalized budget $\mu \triangleq p_{1}(t-t_{1})$ is of the form:
\begin{align*}
    \mathbb{W}(t) =  \begin{pmatrix}
(\mu\frac{\cos^{2}(\phi_{j})}{d-1}+\lambda_{a})\mathbb{I}_{d} & (0) \\
(0) & \mu\sin^{2}(\phi_{j})+\lambda_{b}
\end{pmatrix} = \mu \mathbb{W}_{j} + \mathbb{W}(t_{1}).
\end{align*}
We recall that the set of actions associated with region $j$ is $\{a\in \mathbb{B}_{d}, \sin(\phi_{j}) \leq \langle a,e_{d}\rangle <\sin(\phi_{j+1})\}$. Therefore, the use of Kiefer-Wolfowitz theorem \cite{lattimore2020bandit} combined with the fact $\lambda_{a}\geq \lambda_{b}$ yields that the optimal policy while evolving in region $j$ only plays unit action vector $v_{j}\equiv (\cos(\phi_{j})/(d-1)^{1/2},\dots, \cos(\phi_{j})/(d-1)^{1/2},\sin(\phi_{j}))$. By noting that $v_{j}v_{j}^{\top}=\mathbb{W}_{j}$, we obtain the formula announced. 
Reachability of a given state $i<j$ from state $j$ after time $t_{1}$ is then defined as:
\begin{align*}
    \exists t \geq t_{t}, \quad \frac{1}{p_{i}}\operatorname{Tr}(\mathbb{W}(t)^{-1}\mathbb{W}_{i}) &= \frac{1}{p_{j}}\operatorname{Tr}(\mathbb{W}(t)^{-1}\mathbb{W}_{j}).
\end{align*}
We interpret this as a classical a first-order optimally condition for convex maximization problems, where the matrix $\mathbb{W}_{j}$ is weighted by the censorship probability representing the speed of increase in region $j$. We then rewrite this condition as:
\begin{align*}
    \exists \mu \geq 0, \quad  \frac{1+f(\mu)\cos^{2}(\phi_{i})}{1+f(\mu)\cos^{2}(\phi_{j})} =\frac{p_{i}}{p_{j}} \quad \textit{where} \quad f(\mu) \triangleq \frac{\mu\sin^{2}(\phi_{j})+\lambda_{b}}{\mu\frac{\cos^{2}(\phi_{j})}{d-1}+\lambda_{a}}-1.
\end{align*}
We know that $f$ is increasing in $\mu$ and the LHS of the equation above is decreasing in $f(\mu)$ as $i< j$. Hence, the reachability condition than be stated by looking at the limit of $f$ in $+\infty$. By using the fact that $\lim_{\mu \rightarrow +\infty} f(\mu) = \frac{d\sin^{2}(\phi_{j})-1}{\cos^{2}(\phi_{j})}$, we deduce that the reachability condition is equivalent to looking at the position of $\frac{p_{i}}{p_{j}}$ with respect to: 
\begin{align*}
    r^{\star}(i,j) \triangleq \frac{1+ud[\sin^{2}(\phi_{j})-\frac{1}{d}]}{d\sin^{2}(\phi_{j})} = \frac{(d-1)u+l}{d} = \frac{1}{d}\operatorname{Tr}(\mathbb{W}_{j}^{-1}\mathbb{W}_{i}).
\end{align*}
On the one hand, if $\frac{p_{i}}{p_{j}} \geq r^{\star}(i,j)$, the state in never reachable in a finite time. On the other hand, whenever $\frac{p_{i}}{p_{j}}< r^{\star}(i,j)$, the state is reachable by investing a budget $\mu^{\star}(i,j,\lambda_{a},\lambda_{b})$ such that:
\begin{align*}
    f(\mu^{\star}(i,j,\lambda_{a},\lambda_{b})) = \frac{1}{\cos^{2}(\phi_{j})}\frac{\frac{p_{i}}{p_{j}}-1}{u-\frac{p_{i}}{p_{j}}},
\end{align*}
which in turn involves:
\begin{align*}
    \mu^{\star}(i,j,\lambda_{a},\lambda_{b}) 
    &= \frac{d-1}{d\sin^{2}(\phi_{j})\cos^{2}(\phi_{j})}\frac{(\sin^{2}(\phi_{j})\lambda_{a}+\cos^{2}(\phi_{j})\lambda_{b})\frac{p_{i}}{p_{j}} - (\sin^{2}(\phi_{i})\lambda_{a}+\cos^{2}(\phi_{i})\lambda_{b})}{r^{\star}(i,j)-\frac{p_{i}}{p_{j}}}.
\end{align*}
In particular, at $t_{1}=0$ whenever $\lambda_{b}=\lambda_{a}=\lambda$ and $j=k$, this gives:
\begin{align*}
    \mu^{\star}(i,k,\lambda,\lambda) = \frac{(d-1)\lambda}{d\sin^{2}(\phi_{k})\cos^{2}(\phi_{k})} \frac{\frac{p_{i}}{p_{k}}-1}{r^{\star}(i,k)-\frac{p_{i}}{p_{k}}}.
\end{align*}
The first reachable region from region $j$ is then defined as $i^{\star}\triangleq \operatorname{argmin}_{i\in I}\mu^{\star}(i,j,\lambda_{a},\lambda_{b})$, where $I\triangleq \{i; i <j \quad \textit{and} \quad \frac{p_{i}}{p_{j}} < r^{\star}(i,j)\}$. 
Note that at the moment $t_{2}\triangleq t_{1}+\frac{1}{p_{j}}\mu^{\star}(i^{\star},j,\lambda_{a},\lambda_{b})$ when this region is reached, we have:
\begin{align*}
    \mathbb{W}(t_{2}) = \frac{(d-1)\sin^{2}(\phi_{j})\lambda_{a}-\cos^{2}(\phi_{j})\lambda_{b}}{d\cos^{2}(\phi_{j})\sin^{2}(\phi_{j})(r^{\star}(i^{\star},j)-\frac{p_{i}}{p_{j}})}\mathbb{W}(i,j).
\end{align*}
On the other hand, whenever the set $I$ is empty, by definition, the process reaches case $1$ steady-state regime and only plays optimal policy of region $j$ for remaining budget. 
To be fully general, we note that two or more regions can be reached simultaneously. In this case, the optimal policy tie-breaks by taking the region with maximal index i.e. higher censorship, as further described in Lemma \ref{Dual Reachability Analysis}. 
\end{proof}


\subsection{Statement and Proof of Cor. \ref{Path Formula}}

More generally, this allows us to deduce the next technical corollary:
\begin{corollary}\label{Path Formula}
For a sequence of censored regions $\{i_{1}=k,\dots,i_{l},i_{l+1},\dots\}$, we have for the $l^{th}$ region of the transient sequence, with starting time $t_{l-1}$ and ending time $t_{l}$:
\begin{align*}
    \mathbb{W}(t_{l}) &= \lambda \mathbb{I}_{d}+\sum_{n=1}^{l}\mu^{\star}(i_{n+1},i_{n},\lambda^{\mathbb{W}(t_{n-1})}_{a},\lambda_{b}^{\mathbb{W}(t_{n-1})})\mathbb{W}_{i_{n}}\\
    &=\frac{\lambda\frac{(d-1)\sin^{2}(\phi_{k})-\cos^{2}(\phi_{k})}{\cos^{2}(\phi_{i_{l}})\sin^{2}(\phi_{i_{l}})}\displaystyle\prod_{n=1}^{l-1} \Big(r^{\dagger}(i_{n+1},i_{n})-\frac{p_{i_{n+1}}}{p_{i_{n}}}\Big)}{d^{l}\displaystyle\prod_{n=1}^{l}\Big( r^{\star}(i_{n+1},i_{n})-\frac{p_{i_{n+1}}}{p_{i_{n}}}\Big)\displaystyle\prod_{n=1}^{l-1} \Big(u(i_{n+1},i_{n})+dl(i_{n+1},i_{n})\Big)}\mathbb{W}(i_{l+1},i_{l}),
\end{align*}
where $t_{l}$ is characterized by:
\begin{align*}
    t_{l} = \sum_{n=1}^{l}\frac{1}{p_{i_{n}}}\mu^{\star}(i_{n+1},i_{n},\lambda^{\mathbb{W}(t_{n-1})}_{a},\lambda_{b}^{\mathbb{W}(t_{n-1})}),
\end{align*}
and where $\lambda^{\mathbb{W}(t_{n})}_{a}$ and $\lambda_{b}^{\mathbb{W}(t_{n})}$ refer respectively to the upper and lower coefficient of the diagonal matrix $\mathbb{W}(t_{n})$.
\end{corollary}
\begin{proof}
We leverage a simple induction reasoning using for $l \geq 1$ the formula given within the proof of lemma \ref{One-step Transient Analysis}:
\begin{align*}
    t_{l} &= t_{l-1} + \frac{1}{p_{i_{l}}}\mu^{\star}(i_{l+1},i_{l},\lambda^{\mathbb{W}(t_{l-1})}_{a},\lambda_{b}^{\mathbb{W}(t_{l-1})}) \\
    \mathbb{W}(t_{l}) &= \frac{(d-1)\sin^{2}(\phi_{i_{l}})\lambda^{\mathbb{W}(t_{l-1})}_{a}-\cos^{2}(\phi_{i_{l}})\lambda^{\mathbb{W}(t_{l-1})}_{b}}{d\cos^{2}(\phi_{i_{l}})\sin^{2}(\phi_{i_{l}})(r^{\star}(i_{l+1},i_{l})-\frac{p_{i_{l+1}}}{p_{i_{l}}})}\mathbb{W}(i_{l+1},i_{l}),
\end{align*}
and the initialization conditions $t_{0}=0$ and $\mathbb{W}(0)=\lambda \mathbb{I}_{d}$.
\end{proof}


\subsection{Statement and Proof of Lemma \ref{Dual Reachability Analysis}}

\begin{lemma}[Dual Reachability Analysis]\label{Dual Reachability Analysis}
Let's assume we are currently playing transient region $j$ and we reach the region $i$ at time $t_{l}$. We then have the following two possible cases:
\begin{itemize}
    \item If $\frac{p_{i}}{p_{j}} > r^{\dagger}(i,j)$, we say that regions $i$ is dual reachable from region $j$, leading to a steady state regime with bi-region $(i,j)$ effective dimension. In such case, for $t\geq t_{l}$, the potential increase is of the form:
    \begin{align*}
            \mathbb{W}(t) = \frac{1}{p_{i}/p_{j} + \frac{d}{u+(d-1)l}\frac{r^{\star}(i,j)-p_{i}/p_{j}}{p_{i}/p_{j} - r^{\dagger}(i,j)}}\frac{u-l}{u+(d-1)l}\frac{1}{p_{i}/p_{j}-r^{\dagger}(i,j)}p_{i}(t+\lambda^{\star})\mathbb{W}(i,j).
    \end{align*}
    \item Otherwise, we switch from base region $j$ to base region $i$ and continue in the transient regime.
\end{itemize}
\end{lemma}

\begin{figure}[h]
    \centering
    \includegraphics[width=0.8\textwidth]{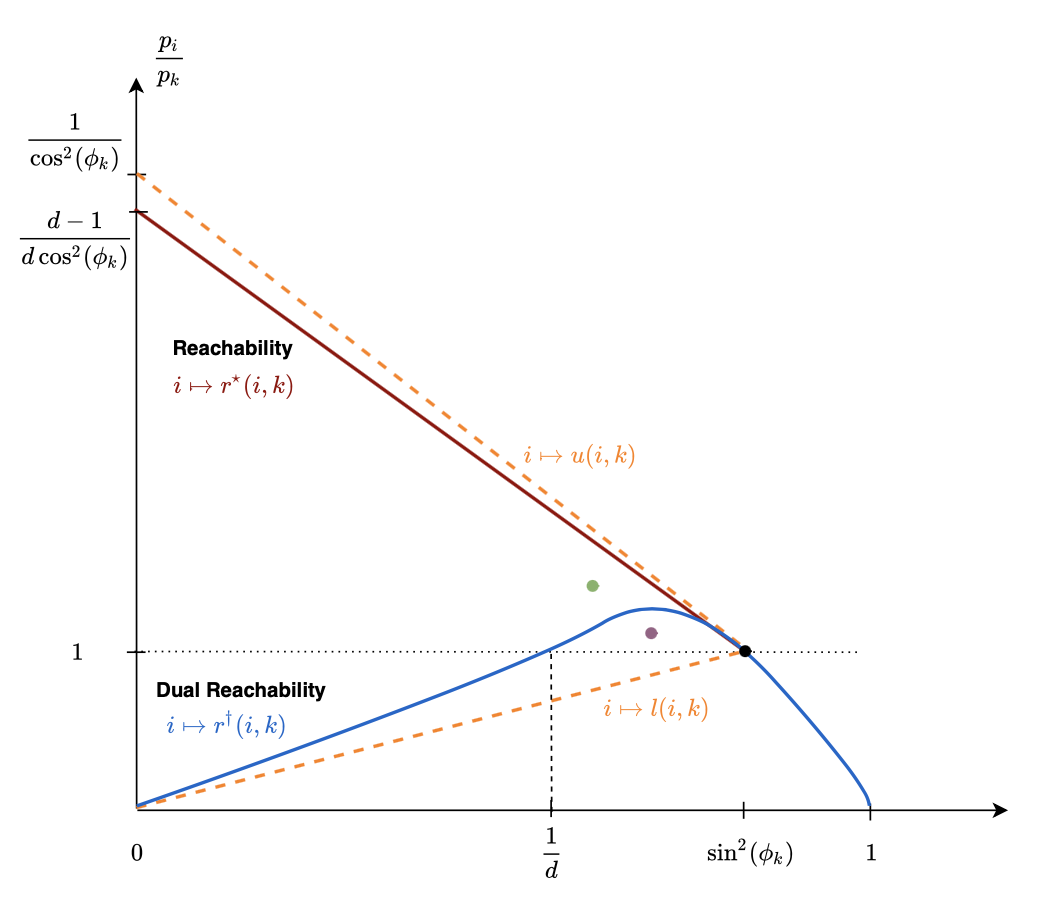}
    \caption{Sketch plot of reachability and dual reachability conditions from base region $k$ associated with the black dot (Lemma \ref{Dual Reachability Analysis} and \ref{One-step Transient Analysis}) as a function of $\frac{p_{i}}{p_{j}}$. For a region $i$ to be reachable, $\frac{p_{i}}{p_{j}}$ has to be below the red line. For a region $i$ to be dual reachable, $\frac{p_{i}}{p_{j}}$ has to be above the blue line. Henceforth, the red dot here is a censorship region that is both reachable and dual reachable whereas the purple dot is a reachable but not dual reachable region. Orange lines represent the functions $u(i,k)$ and $l(i,k)$ introduced above in Sec.\ref{sup_not}.}
    \label{reach}
\end{figure}

\begin{figure}[h]
    \centering
    \includegraphics[width=0.9 \textwidth]{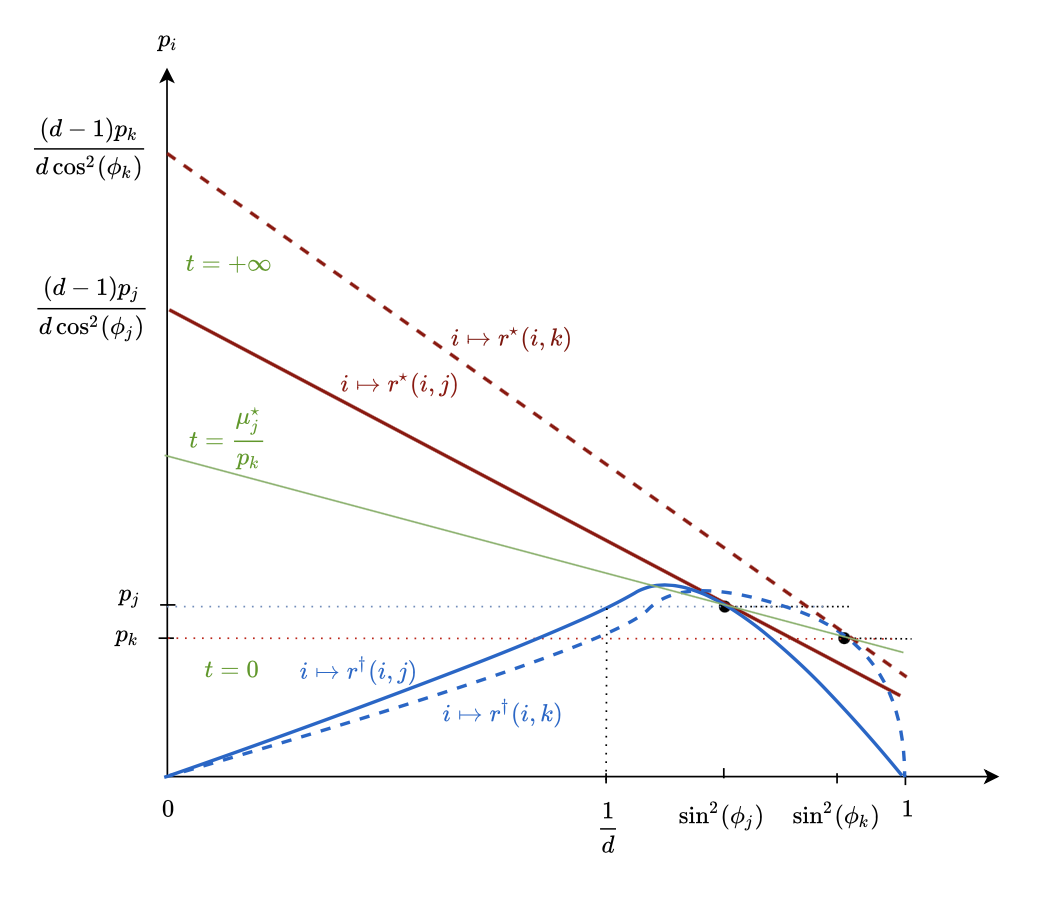}
    \caption{Sketch plot of the evolution of reachability and dual reachability conditions after a region $j$ is reached from region $k$ but is not dual reachable (Else condition in Alg. \ref{Dyn_MT}). Doted red (resp. blue) line is reachability (resp. dual reachability) condition for previous region $k$ and full red (resp. blue) lines is reachability (resp. dual reachability) condition for new region $j$. Instead of starting from horizontal line at $t=0$ to find new reachable state, rotation with region $j$ as pivot is initialized at the green line associated with $t=\frac{\mu^{\star}_{j}}{p_{k}}$. Note that the $y$-axis is not normalized here.}
    \label{switch}
\end{figure}

\begin{proof}
Using previous section, we know that $\mathbb{W}(t_{l}) \propto \mathbb{W}(i,j)$
where we recall that the matrix $\mathbb{W}(i,j)$ has the strong property that the gains in regions $i$ and $j$ are equal i.e.:
\begin{align*}
    \frac{1}{p_{i}}\operatorname{Tr}(\mathbb{W}(i,j)^{-1}\mathbb{W}_{i}) &= \frac{1}{p_{j}}\operatorname{Tr}(\mathbb{W}(i,j)^{-1}\mathbb{W}_{j}).
\end{align*} One of the main result we show in the multi-threshold censorship model is that for $t\geq t_{l}$, we have:
\begin{align*}
    \mathbb{W}(t) - \mathbb{W}(t_{l}) \propto (t-t_{l}) \mathbb{W}(i,j),
\end{align*}
which involves in particular that for $t\geq t_{l}, \mathbb{W}(t)\propto \mathbb{W}(i,j)$. This is possible thanks to the fact that the optimal policy produces a combination of $p_{i}\mathbb{W}_{i}$ and $p_{j}\mathbb{W}_{j}$ proportional to $\mathbb{W}(i,j)$ so that optimally of both regions $i$ and $j$ is maintained while maximal first-order gain is simultaneously ensured. The proportionality condition is then written as the existence of $\mu_{i},\mu_{j}>0$ such that $p_{i}\mu_{i}\mathbb{W}_{i}+p_{j}\mu_{j}\mathbb{W}_{j}\propto  \mathbb{W}(i,j)$ or equivalently as:
\begin{align*}
    \exists \mu_{i},\mu_{j}>0, \quad  \frac{\frac{1}{d-1}[p_{i}\mu_{i}\cos^{2}(\phi_{i})+p_{j}\mu_{j}\cos^{2}(\phi_{j})]}{p_{i}\mu_{i}\sin^{2}(\phi_{i})+p_{j}\mu_{j}\sin^{2}(\phi_{j})} = \frac{\cos^{2}(\phi_{j})(u(i,j)-\frac{p_{i}}{p_{j}})}{\sin^{2}(\phi_{j})(\frac{p_{i}}{p_{j}}-l(i,j))} \triangleq R,
\end{align*}
where $\mu_{i}$ and $\mu_{j}$ are the infinitesimal time increase in regions $i$ and $j$. It leads in turn to the ratio equality:
\begin{align*}
    \frac{p_{i}\mu_{i}}{p_{j}\mu_{j}} = \frac{\sin^{2}(\phi_{j})(d-1)R - \cos^{2}(\phi_{j})}{\cos^{2}(\phi_{i})-\sin^{2}(\phi_{i})(d-1)R} = \frac{(d-1)u+l - d\frac{p_{i}}{p_{j}}}{(u+(d-1)l)\frac{p_{i}}{p_{j}} - dlu} = \frac{d}{u+(d-1)l}\frac{r^{\star}(i,j)-\frac{p_{i}}{p_{j}}}{\frac{p_{i}}{p_{j}} - r^{\dagger}(i,j)}.
\end{align*}
Thus, we see that bi-region stationarity is possible if and only if $\frac{p_{i}}{p_{j}} > r^{\dagger}(i,j)$ where we introduced the dual reachability condition:
\begin{align*}
    r^{\dagger}(i,j) \triangleq \frac{dl(i,j)u(i,j)}{u(i,j)+(d-1) l(i,j)} = \Big(\frac{\frac{d-1}{u(i,j)}+\frac{1}{l(i,j)}}{d}\Big)^{-1} = \Big(\frac{1}{d}\operatorname{Tr}(\mathbb{W}_{i}^{-1}\mathbb{W}_{j})\Big)^{-1} = \frac{1}{r^{\star}(j,i)}.
\end{align*}
Hence, the use of the term dual reachability comes from the fact that region $i$ is dual reachable from region $j$ if and only if region $j$ is reachable from region $j$.
In such case, further algebraic calculation then lead to the instantaneous potential increase $\partial W$ for infinitesimal time $\partial t\triangleq \mu_{j}+\mu_{j}$:
\begin{align*}
    \partial W(\partial t) &\triangleq p_{j}\mu_{j}\mathbb{W}_{j} + p_{i}\mu_{i}\mathbb{W}_{i} = \frac{u-l}{u+(d-1)l}\frac{1}{\frac{p_{i}}{p_{j}}-r^{\dagger}(i,j)}p_{j}\mu_{j}\mathbb{W}(i,j).
\end{align*}
We then note that:
\begin{align*}
    \frac{\mu_{i}+\mu_{j}}{\mu_{j}} = 1 + \frac{1}{\frac{p_{i}}{p_{j}}}\frac{d}{u+(d-1)l}\frac{r^{\star}(i,j)-\frac{p_{i}}{p_{j}}}{\frac{p_{i}}{p_{j}} - r^{\dagger}(i,j)}.
\end{align*}
Therefore, we conclude that:
\begin{align*}
    \partial W(\partial t) &= \frac{1}{p_{i}/p_{j} + \frac{d}{u+(d-1)l}\frac{r^{\star}(i,j)-p_{i}/p_{j}}{p_{i}/p_{j} - r^{\dagger}(i,j)}}\frac{u-l}{u+(d-1)l}\frac{1}{p_{i}/p_{j}-r^{\dagger}(i,j)}p_{i}(\mu_{j}+\mu_{i})\mathbb{W}(i,j) \\
    &= \frac{1}{p_{i}/p_{j} + \frac{d}{u+(d-1)l}\frac{r^{\star}(i,j)-p_{i}/p_{j}}{p_{i}/p_{j} - r^{\dagger}(i,j)}}\frac{u-l}{u+(d-1)l}\frac{1}{p_{i}/p_{j}-r^{\dagger}(i,j)}p_{i}\partial t\mathbb{W}(i,j).
\end{align*}

We then introduce $\lambda^{\star}$ defined such that:
\begin{align*}
    (t_{l}+\lambda^{\star})\mathbb{W}(i,j) \triangleq \frac{1}{p_{i}}\frac{(u+(d-1)l)(p_{i}/p_{j}-r^{\dagger}(i,j))}{u-l}\Big(p_{i}/p_{j} + \frac{d}{u+(d-1)l}\frac{r^{\star}(i,j)-p_{i}/p_{j}}{p_{i}/p_{j} - r^{\dagger}(i,j)}\Big)\mathbb{W}(t_{l}).
\end{align*}
Given the previous two results, we conclude that for all $t\geq t_{l}$:
\begin{align*}
    \mathbb{W}(t) = \frac{1}{p_{i}/p_{j} + \frac{d}{u+(d-1)l}\frac{r^{\star}(i,j)-p_{i}/p_{j}}{p_{i}/p_{j} - r^{\dagger}(i,j)}}\frac{u-l}{u+(d-1)l}\frac{1}{p_{i}/p_{j}-r^{\dagger}(i,j)}p_{i}(t+\lambda^{\star})\mathbb{W}(i,j).
\end{align*}
Note that entering the bi-region stationary regime impedes new regions to be reachable. Indeed, going back to the initial definition of reachability, region $n$ is said to be reachable from region $j$ after time $t_{l}$ if and only if: 
\begin{align*}
    \exists t\geq t_{l}, \quad \frac{1}{p_{n}}\operatorname{Tr}(\mathbb{W}(t)^{-1}\mathbb{W}_{n}) &= \frac{1}{p_{j}}\operatorname{Tr}(\mathbb{W}(t)^{-1}\mathbb{W}_{j}).
\end{align*}
Yet, using previous result on the evolution of $\mathbb{W}(t)$, we know that the ratio of those two quantities remain equal for any $t\geq t_{l}$ i.e. no new regions can be reached.

Moreover, using the optimality criterion of Lemma \ref{One-step Transient Analysis}, when several regions are reached simultaneously, the tie-breaking is performed by considering the most censored region, i.e. the one with the highest $i$ index. If the chosen region is not dual reachable, then the next one is considered. In the case where none of them is dual reachable, the base region becomes the maximally censored region and we immediately reiterate the procedure described in Lemma \ref{Dual Reachability Analysis}.


\end{proof}


\subsection{Statement and Proof of Lemma \ref{Bi-Region Effective Dimension}}

\begin{lemma}[Bi-Region Effective Dimension]\label{Bi-Region Effective Dimension}
Let's assume we reach a bi-region $(i,j)$ steady state regime at time $t_{l}\leq T$. Then, we have:
\begin{align*}
    \int_{t_{l}}^{T}\frac{1}{p(a(t))}\frac{\partial\log\det(\mathbb{W}(t))}{\partial t}\partial t = d_{\mathit{eff}}\log(1+\frac{T-t_{l}}{t_{l}+\lambda^{\star}}) \sim d_{\mathit{eff}}\log(T),
\end{align*}
where $d_{\mathit{eff}}=\frac{1}{p_{j}}\left[(d-1) \frac{1-l(i,j)}{p_{i}/p_{j}-l(i,j)}+\frac{u(i,j)-1}{u(i,j)-p_{i}/p_{j}}\right]$ and $\lambda^{\star}$ is given in the proof of Lemma \ref{Dual Reachability Analysis}. Moreover, we have the cumulative transient potential:
\begin{align*}
     \int_{0}^{t_{l}}\frac{1}{p(a(t))}\frac{\partial\log\det(\mathbb{W}(t))}{\partial t}\partial t &= \sum_{n=1}^{l} \frac{1}{p_{i_{n}}}\int_{t_{n-1}}^{t_{n}}\partial\log\det(\mathbb{W}(t)) = \sum_{n=1}^{l} \frac{1}{p_{i_{n}}} \log\frac{\det(\mathbb{W}(t_{n}))}{\det(\mathbb{W}(t_{n-1}))} \\
    &= \sum_{n=1}^{l} (\frac{1}{p_{i_{n}}}-\frac{1}{p_{i_{n+1}}})\log\det\mathbb{W}(t_{n}).
\end{align*}
\end{lemma}

\begin{proof}

For $t\geq t_{l}$, we have the infinitesimal two-step increase $\partial G$ during the infinitesimal time $\partial t\triangleq \mu_{i}+\mu_{j}$:
\begin{align*}
    \partial G(\partial t) &\triangleq \mu_{i}\operatorname{Tr}(\mathbb{W}(t)^{-1}\mathbb{W}_{i}) + \mu_{j}\operatorname{Tr}((\mathbb{W}(t)+\mu_{i}p_{i}\mathbb{W}_{i})^{-1}\mathbb{W}_{j})\\
    &= \mu_{i}\operatorname{Tr}(\mathbb{W}(t)^{-1}\mathbb{W}_{i}) + \mu_{j}\operatorname{Tr}(\mathbb{W}(t)^{-1}\mathbb{W}_{j}) +o(\partial t)\\
    &= \frac{p_{i}\mu_{i}+p_{j}\mu_{j}}{p_{j}}\operatorname{Tr}(\mathbb{W}(t)^{-1}\mathbb{W}_{j}) +o(\partial t),
\end{align*}
where we used the property of $\mathbb{W}(i,j)$. Invoking lemma $\ref{Dual Reachability Analysis}$, we know the evolution of $\mathbb{W}(t)$ for $t\geq t_{l}$:
\begin{align*}
    \mathbb{W}(t)=\frac{1}{1 + \frac{1}{p_{i}/p_{j}}\frac{d}{u+(d-1)l}\frac{r^{\star}(i,j)-p_{i}/p_{j}}{p_{i}/p_{j} - r^{\dagger}(i,j)}}\frac{u-l}{u+(d-1)l}\frac{1}{p_{i}/p_{j}-r^{\dagger}(i,j)}p_{j}(t+\lambda^{\star})\mathbb{W}(i,j),
\end{align*}
as well as the relations between $\mu_{i}$ and $\mu_{j}$:
\begin{align*}
    \begin{cases}
\frac{p_{i}\mu_{i}+p_{j}\mu_{j}}{p_{j}} &= \mu_{j}(1+\frac{d}{u+(d-1)l}\frac{r^{\star}(i,j)-p_{i}/p_{j}}{p_{i}/p_{j} - r^{\dagger}(i,j)})\\
\frac{\mu_{i}+\mu_{j}}{\mu_{j}} &= 1 + \frac{1}{p_{i}/p_{j}}\frac{d}{u+(d-1)l}\frac{r^{\star}(i,j)-p_{i}/p_{j}}{\frac{p_{i}}{p_{j}} - r^{\dagger}(i,j)}.
\end{cases}
\end{align*}
We invoke the fact that $\operatorname{Tr}(\mathbb{W}(i,j)^{-1}\mathbb{W}_{j})=\frac{1}{u-p_{i}/p_{j}}+\frac{1}{p_{i}/p_{j}-l}$ to conclude that:
\begin{align*}
    \partial G(\partial t) &=\frac{1}{p_{j}}\frac{[(d-1)l+u-d]\frac{p_{i}}{p_{j}}-[d l u-((d-1)u+l)]}{(u-\frac{p_{i}}{p_{j}})(\frac{p_{i}}{p_{j}}-l)}\frac{(1 + \frac{1}{p_{i}/p_{j}}\frac{d}{u+(d-1)l}\frac{r^{\star}(i,j)-p_{i}/p_{j}}{p_{i}/p_{j} - r^{\dagger}(i,j)})\mu_{j}}{t+\lambda^{\star}} \\
    &= \frac{1}{p_{j}}\left[(d-1) \frac{1-l}{\frac{p_{i}}{p_{j}}-l}+\frac{u-1}{u-\frac{p_{i}}{p_{j}}}\right] \frac{\partial t}{t+\lambda^{\star}} \\
    &= d_{\mathit{eff}}\frac{\partial t}{t+\lambda^{\star}}.
\end{align*}
Given that $\partial t$ is an infinitesimal time increase, we have in the steady state regime:
\begin{align*}
    \int_{t_{l}}^{T} \partial G = d_{\mathit{eff}}\int_{t_{l}}^{T}\frac{\partial t}{t+\lambda^{\star}} = d_{\mathit{eff}}\log(\frac{T+\lambda^{\star}}{t_{l}+\lambda^{\star}}) =  d_{\mathit{eff}}\log(1+\frac{T-t_{l}}{t_{l}+\lambda^{\star}}).
\end{align*}
We finally note that the cumulative potential coming from the transient period is equal to:
\begin{align*}
    \int_{0}^{t_{l}} \partial G &= \sum_{n=1}^{l} \frac{1}{p_{i_{n}}}\int_{t_{n-1}}^{t_{n}}\partial\log\det(\mathbb{W}(t)) = \sum_{n=1}^{l} \frac{1}{p_{i_{n}}} \log\frac{\det(\mathbb{W}(t_{n}))}{\det(\mathbb{W}(t_{n-1}))} \\
    &= \sum_{n=1}^{l} (\frac{1}{p_{i_{n}}}-\frac{1}{p_{i_{n+1}}})\log\det\mathbb{W}(t_{n}),
\end{align*}
where the closed-form expression of $\mathbb{W}(t_{n})$ is given in Corollary \ref{Path Formula}.
\end{proof}


\subsection{Special case: Single-threshold model}

\begin{corollary}\label{single_thres}
For the single threshold model with two regions $0$ and $1$ and associated censorship probabilities $p_{0}< p_{1}$, our main theorem yields:
\begin{itemize}
    \item If $\frac{p_{0}}{p_{1}} < \frac{d-1}{d\cos^{2}(\phi_{1})}$, then we reach bi-region steady state regime and have the effective dimension:
    \begin{align*}
        d_{\mathit{eff}} = \frac{d-1}{p_{0}}+\frac{1}{p_{0}}\frac{\sin^{2}(\phi_{1})}{\frac{p_{1}}{p_{0}}-\cos^{2}(\phi_{1})} \in [\frac{d}{p_{0}},\frac{d}{p_{1}}].
    \end{align*}
    \item Otherwise, we are from $t=0$ in single-region steady state regime and have the effective dimension $d_{\mathit{eff}} = d/p_{1}$.
\end{itemize}
\end{corollary}
\begin{proof}
Using Lemma \ref{Dual Reachability Analysis} in the case of the single threshold model, we note that if region $0$ is reachable, it is necessarily dual reachable given that $r^{\dagger}(0,1)=0$ and henceforth, we always have $p_{0}/p_{1}> r^{\dagger}(0,1)$. Thanks to the results of Lemma \ref{One-step Transient Analysis}, we also note that $r^{\star}(0,1)=\frac{p_{0}}{p_{1}} < \frac{d-1}{d\cos^{2}(\phi_{1})}$ and that if region $0$ is reachable, it is done in a time:
\begin{align*}
    t_{1} = \frac{1}{p_{1}}\frac{(d-1)\lambda}{d\sin^{2}(\phi_{1})\cos^{2}(\phi_{1})}\frac{\frac{p_{0}}{p_{1}}-1}{\frac{d-1}{d\cos^{2}(\phi_{1})}-\frac{p_{0}}{p_{1}}}
\end{align*}
\end{proof}

\end{document}